\documentclass[nohyperref]{article}

\usepackage{microtype}
\usepackage{graphicx}
\usepackage{wrapfig}
\usepackage{subfigure}
\usepackage{booktabs} % for professional tables
\usepackage[hang,flushmargin]{footmisc}
\usepackage{algorithm,algorithmic}
\usepackage{textcomp}

\usepackage[accepted,nohyperref]{icml2023}
\makeatletter
\def\addcontentsline#1#2#3{%
  \addtocontents{#1}{\protect\contentsline{#2}{#3}{\thepage}}}
\long\def\addtocontents#1#2{%
  \protected@write\@auxout
    {\let\label\@gobble \let\index\@gobble \let\glossary\@gobble}%
    {\string\@writefile{#1}{#2}}}
\makeatother
\usepackage{hyperref}
\definecolor{mydarkblue}{rgb}{0,0.08,0.45}
\hypersetup{ %
pdftitle={},
pdfauthor={},
pdfsubject={Proceedings of the International Conference on Machine Learning 2023},
pdfkeywords={},
pdfborder=0 0 0,
pdfpagemode=UseNone,
colorlinks=true,
linkcolor=mydarkblue,
citecolor=mydarkblue,
filecolor=mydarkblue,
urlcolor=mydarkblue,
pdfauthor={Anonymous Submission},
}
\usepackage{amsmath,amssymb}
\usepackage{mathtools}
\usepackage{amsthm}
\usepackage{multirow}
\usepackage{bbm}

\usepackage[capitalize,noabbrev]{cleveref}
\usepackage{caption}
\theoremstyle{plain}
\newtheorem{theorem}{Theorem}[section]
\newtheorem{proposition}[theorem]{Proposition}
\newtheorem{lemma}[theorem]{Lemma}
\newtheorem{corollary}[theorem]{Corollary}
\theoremstyle{definition}
\newtheorem{definition}[theorem]{Definition}
\newtheorem{assumption}[theorem]{Assumption}
\theoremstyle{remark}
\newtheorem{remark}[theorem]{Remark}
\newtheorem{example}[theorem]{\textbf{\emph{Example}}}
\newtheorem{claim}[theorem]{\textbf{\emph{Claim}}}
\newenvironment{counterexample}{%
  \proof}{\endproof}
\usepackage{footnote}
\makesavenoteenv{algorithm}
\usepackage[textsize=tiny]{todonotes}
\usepackage[textsize=tiny]{todonotes}
\usepackage[section]{placeins}
\usepackage{array}
\usepackage{enumitem}
\usepackage{etoolbox}  % patch def of algorithmic environment
\makeatletter
\patchcmd{\algorithmic}{\addtolength{\ALC@tlm}{\leftmargin} }{\addtolength{\ALC@tlm}{\leftmargin}}{}{}
\makeatother
\newcommand{\alglinelabel}{%
  \addtocounter{ALC@line}{-1}% Reduce line counter by 1
  \refstepcounter{ALC@line}% Increment line counter with reference capability
  \label% Regular \label
}
\newcommand{\ind}{\perp\!\!\!\!\perp} 
\usepackage[para,online,flushleft]{threeparttable}
\usepackage[misc]{ifsym}
\icmltitlerunning{Which Invariance Should We Transfer? A Causal Minimax Learning Approach}

\begin{document}

\twocolumn[
\icmltitle{Which Invariance Should We Transfer? A Causal Minimax Learning Approach}




\begin{icmlauthorlist}
\icmlauthor{Mingzhou Liu}{pkucs,pkucfcs}
\icmlauthor{Xiangyu Zheng}{pkugh}
\icmlauthor{Xinwei Sun \Letter}{fdu}
\icmlauthor{Fang Fang}{pkuphy}
\icmlauthor{Yizhou Wang}{pkucs,pkucfcs,pkuai}
%\icmlauthor{Firstname6 Lastname6}{sch,yyy,comp}
%\icmlauthor{Firstname7 Lastname7}{comp}
%\icmlauthor{Firstname8 Lastname8}{sch}
%\icmlauthor{Firstname8 Lastname8}{yyy,comp}
\end{icmlauthorlist}

\icmlaffiliation{pkucs}{Sch. of Computer Science, Peking University}
\icmlaffiliation{pkucfcs}{Center on Frontiers of Computing Studies, Peking University}
\icmlaffiliation{pkugh}{Dep. of Statistics, Guanghua Sch. of Management, Peking University}
\icmlaffiliation{fdu}{Sch. of Data Science, Fudan University}
\icmlaffiliation{pkuphy}{Sch. of Psychological and Cognitive Sciences, Peking University}
\icmlaffiliation{pkuai}{Inst. for Artificial Intelligence,
Peking University}

\icmlcorrespondingauthor{Xinwei Sun}{sunxinwei@fudan.edu.cn}

\icmlkeywords{robustness, minimax, subset selection, structural causal model, $\sim_G$-equivalence}

\vskip 0.3in
]



\printAffiliationsAndNotice{}  % leave blank if no need to mention equal contribution

\addtocontents{toc}{\protect\setcounter{tocdepth}{0}}

\begin{abstract}
  A major barrier to deploying current machine learning models lies in their non-reliability to dataset shifts. To resolve this problem, most existing studies attempted to transfer stable information to unseen environments. Particularly, \emph{independent causal mechanisms}-based methods proposed to remove mutable causal mechanisms via the \emph{do}-operator. Compared to previous methods, the obtained stable predictors are more effective in identifying stable information. However, a key question remains: \emph{which subset of this whole stable information should the model transfer, in order to achieve optimal generalization ability?} To answer this question, we present a comprehensive minimax analysis from a causal perspective. Specifically, we first provide a graphical condition for the whole stable set to be optimal. When this condition fails, we surprisingly find with an example that this whole stable set, although can fully exploit stable information, is not the optimal one to transfer. To identify the optimal subset under this case, we propose to estimate the worst-case risk with a novel \emph{optimization} scheme over the intervention functions on mutable causal mechanisms. We then propose an efficient algorithm to search for the subset with minimal worst-case risk, based on a newly defined equivalence relation between stable subsets. Compared to the exponential cost of exhaustively searching over all subsets, our searching strategy enjoys a polynomial complexity. The effectiveness and efficiency of our methods are demonstrated on synthetic data and the diagnosis of Alzheimer's disease. 
\end{abstract}
\section{Introduction}

Current machine learning systems, which are commonly deployed based on their in-distribution performance, often encounter \emph{dataset shifts} \cite{quinonero2008dataset} such as covariate shift, label shift, \emph{ect.}, due to changes in the data generating process. When such shifts exist in deployment environments, the model may give unreliable prediction results, which can cause severe consequences in safe-critical tasks such as healthcare \citep{hendrycks2021many}. At the heart of this unreliability issue are \emph{stability} and \emph{robustness} aspects, which respectively denote the insensitivity of prediction behavior and generalization errors to dataset shifts. 

For example, consider the system deployed to predict the Functional Activities Questionnaire (FAQ) score that is commonly adopted \cite{mayo2016use} to measure the severity of Alzheimer's disease (AD). During the prediction, the system can only access biomarkers and volumes of brain regions as covariates, with demographic information anonymous for privacy consideration. However, the changes in such demographics can cause shifts in covariates. To achieve reliability for the deployed model, its prediction is desired to be stable against demographic changes, and meanwhile to be constantly accurate across all populations. For this purpose, this paper aims to find the most robust (\emph{i.e.}, minimax optimal) predictor, among the set of stable predictors across all deployed environments.

\begin{figure*}[t!]
\centering
\includegraphics[width=0.9\textwidth]{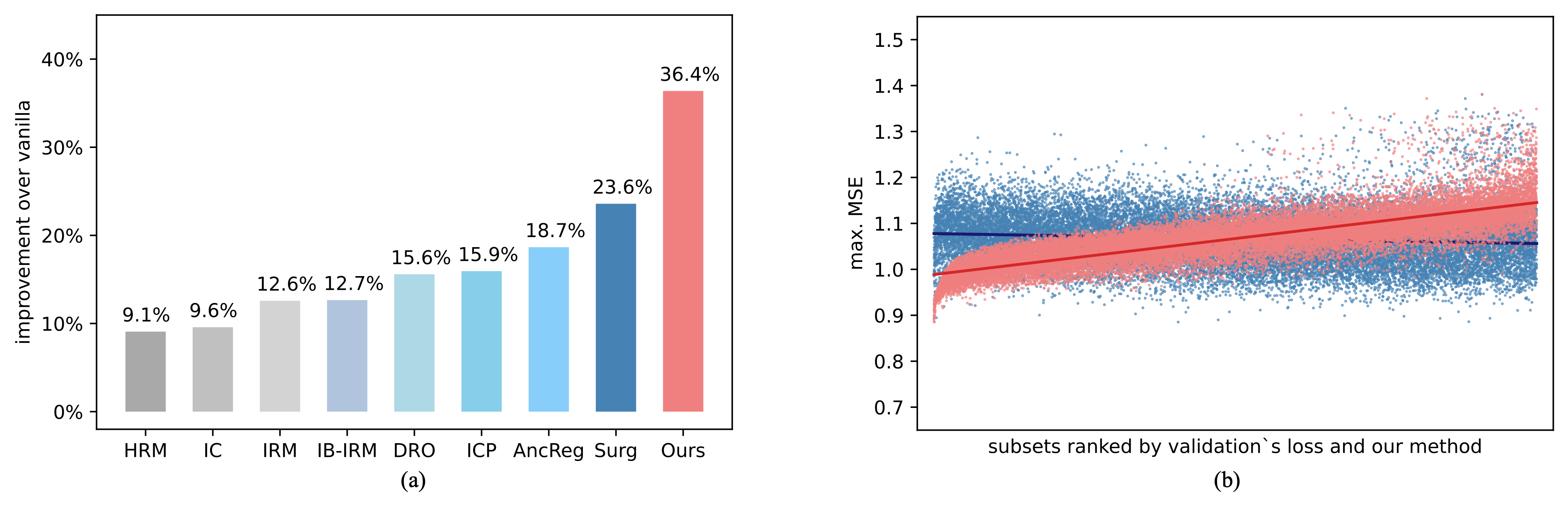}
\caption{FAQ prediction in Alzheimer's disease. (a) Comparison of maximal mean square error (max. MSE) over deployed environments. (b) Max. MSE of subsets that are ranked in ascending order from left to right, respectively according to the estimated worst-case risk of our method (marked by red) and the validation‘s loss adopted by \cite{subbaswamy2019preventing} (marked by blue). }
\label{fig:adni-intro}
\end{figure*}

To achieve this goal, many studies attempted to learn invariance to transfer to unseen data. Examples include ICP \cite{peters2016causal} and  \citep{rojas2018invariant, liu2021heterogeneous, ausset2022empirical} that assumed the prediction mechanism given causal features or representations to be invariant; or \citep{subbaswamy2019preventing, rothenhausler2021anchor} that explicitly attributed the variation to a prior selection diagram or an exogenous variable. Particularly, the recent \emph{independent causal mechanisms} (ICM)-based methods \cite{subbaswamy2019preventing,scholkopf2021toward} causally factorized the joint distribution into the mutable ($M$) set and the stable ($S$) set, which contained variables with changed and unchanged causal mechanisms, respectively. By intervening on $M$, they obtained a set of stable predictors, with each containing a stable subset of $S$ to transfer. Compared to ICP-related methods \cite{peters2016causal, buhlmann2020invariance}, these stable predictors exploited more types of invariance and thus potentially had better prediction power. However, an important question on robustness has not been studied: \emph{which subset of $S$ should the model transfer, in order to achieve optimal generalization ability?}

In this paper, we give a comprehensive answer from the perspective of the structural causal model. Specifically, we first provide a graphical condition that is sufficient for the whole stable set to be optimal. This condition can be easily tested via causal discovery. When this condition fails, we construct an example that counter-intuitively shows that this whole stable set, although keeps all the stable information, is \textbf{NOT} the optimal one to transfer. Under this case, we propose an \emph{optimization} scheme over the intervention functions on $M$, which is provable to identify the worst-case risk for each stable subset. Our key observation is that the source of dataset shifts is governed by $M$; therefore, the intervention on $M$, if set appropriately, can well mimic the worst-case deployed environment. Back to the FAQ prediction example, Fig.~\ref{fig:adni-intro} (b) shows that our method can consistently reflect the maximal mean squared error (max. MSE) of stable subsets; as a contrast, the validation's loss adopted by \cite{subbaswamy2019preventing} fails to do so. This explains our advantage in predicting FAQ across patient groups shown in Fig.~\ref{fig:adni-intro} (a). 

To efficiently search for the optimal subset, we define an equivalence relation between stable subsets via $d$-separation such that two equivalent subsets share the same worst-case risk. We theoretically show that compared to exhaustively searching over all subsets, searching over only equivalence classes can reduce the exponential complexity to a polynomial one. The effectiveness and efficiency of our methods are demonstrated by the improved robustness, stability, and computation efficiency on a synthetic dataset and the diagnosis of Alzheimer's disease. 

\textbf{Contributions.} To summarize, our contributions are: 
\begin{enumerate}[noitemsep,topsep=0.1pt]
    \item We propose to select the optimal subset of invariance to transfer, guided by a comprehensive minimax analysis from the causal perspective. To the best of our knowledge, this is the \emph{first} work to study the problem of \emph{which invariance should we transfer}, in the literature of robust learning. 
    \item We define an equivalence relation between stable subsets, and accordingly propose to search over only equivalence classes. This new search algorithm can be efficiently solved in polynomial time. 
    \item We achieve better robustness and stability than others on synthetic data and Alzheimer's disease diagnosis. 
\end{enumerate}
\section{Related work}

\textbf{Causality-based domain generalization.} There are emerging works that considered domain generalization from the causal perspective. One line of works \cite{arjovsky2019invariant,liu2021heterogeneous,ahuja2021invariance,ausset2022empirical} promoted invariance as a key surrogate feature of causation where the causal graph was more of a motivation. Another line of works \cite{peters2016causal,rojas2018invariant,martinet2022variance} was based on invariance assumptions regarding the causal mechanisms. The works most relevant to us are \citep{subbaswamy2019preventing, scholkopf2021toward}, which followed the principle of independent causal mechanisms \cite{scholkopf2012causal} to identify invariance by removing the mutable causal mechanisms. However, they \textbf{did not} study how to select the optimal subset in terms of robustness on out-of-distribution generalization. 

\textbf{Optimization-based domain generalization.} Some recent works, \emph{e.g.}, DRO \cite{sinha2018certifiable} and \citep{sagawa2019distributionally,wu2022generalization} formulated domain generalization as a minimax optimization problem and optimized the predictor for robustness. For optimization convenience, they usually constrained the dataset shifts to a limited extent, which limited their application in the real world. \textbf{In contrast}, we adopt optimization to estimate the worst-case risks of predictors, then select the best one via comparison. Our method can generalize well in a broader distribution family, where the extent of dataset shifts can be unbounded.

\textbf{Heterogeneous causal discovery.} Our work benefits from the recent progress in heterogeneous causal discovery \cite{ghassami2018multi,huang2020causal,perry2022causal}, a field that seeks to learn the causal graph with data from multiple environments. However, unlike causal discovery that recovers causal relationships, we focus on minimax analysis and robust subset selection.

\section{Preliminary} 
\label{sec.background}

We consider the supervised regression scenario, where the system includes a target variable $Y \!\in\! \mathcal{Y}$, covariates $\mathbf{X}\!:=\! [X_1,...,X_d] \!\in\! \mathcal{X}$, and data collected from heterogeneous environments. In practice, different ``environments" can refer to different groups of subjects or different experimental settings. We denote the set of training environments as $\mathcal{E}_{\mathrm{tr}}$, and the broader set of environments for deployment as $\mathcal{E}$. We denote $E$ as the environmental indicator variable with support $\mathcal{E}$. We use $\{\mathcal{D}_{e}\}_{e\in\mathcal{E}_{\text{tr}}}$ to denote our training data, with ${\mathcal{D}_{e}\!:=\!\{(\boldsymbol{x}_{k}^e, {y}_{k}^e)\}_{k=1}^{n_e} \sim_{i.i.d} P^e(\mathbf{X},Y)}$ being data collected from environment $e$. In a directed acyclic graph (DAG) $G$, we denote the parents, children, neighbors, and descendants of the vertex $V_i$ as $\mathbf{Pa}(V_i)$, $\mathbf{Ch}(V_i)$, $\mathbf{Neig}(V_i)$, and $\mathbf{De}(V_i)$, respectively. We use $\ind_G$ to denote \emph{d}-separation in $G$. We denote $G_{\overline{V_i}}$ as the graph attained via deleting all arrows pointing into $V_i$. 

Our goal is to find the most robust predictor $f^*$ among stable predictors with data from $\mathcal{E}_{\mathrm{tr}}$. Here, we say a predictor $f \!:\! \mathcal{X} \to \mathcal{Y}$ is stable if it is independent of $E$. We denote the set of stable predictors as $\mathcal{F}_S$. For robustness, a commonly adopted measurement \cite{peters2016causal,ahuja2021invariance} is to investigate a predictor's worst-case risk, which provides a safeguard for deployment in unseen environments. That is, we want $f^*$ to have the following minimax property:
\begin{align}
    f^{*}(\boldsymbol{x})=\mathop{\mathrm{argmin}}\limits_{f \in \mathcal{F}_S}        \max_{e\in\mathcal{E}}
    \mathbb{E}_{P^e}[(Y-f(\boldsymbol{x}))^2].
\label{eq.robust}
\end{align}
Next, we introduce some basic assumptions, which are commonly made in causal inference and learning \cite{spirtes2000causation,pearl2009causality,arjovsky2019invariant}.
\begin{assumption}[Structural causal model]
\label{eq:assum-dag}
We assume that $P^e(\mathbf{X},Y)$ is entailed by an \emph{unknown} DAG $G$ over $\mathbf{V}$ for all $e \in \mathcal{E}$, where $\mathbf{V}:= \mathbf{X} \cup Y$. Each variable $V_i \in \mathbf{V}$ is generated by a structural equation $V_i=g^e_{i}(\mathbf{Pa}({V_i}), U_i)$, where $U_i$ denotes an exogenous variable. We assume  
each $g^e_{i}$ is continuous and bounded. Each edge $V_i \to V_j$ in $G$ means $V_i$ is a direct cause of $V_j$. Besides, we assume the model is Markovian which states that $\mathbf{A} \ind_G \mathbf{B} | \mathbf{Z} \Rightarrow \mathbf{A} \ind \mathbf{B} | \mathbf{Z}$ for disjoint vertex sets $\mathbf{A},\mathbf{B},\mathbf{Z} \subseteq \mathbf{V}$. 
\end{assumption}

According to the Causal Markov Condition theorem \cite{pearl2009causality}, the joint distribution can be causally factorized into $P^e(\mathbf{V})=\prod_i P^e(V_i|\mathbf{Pa}(V_i)))$, where $P^e(V_i|\mathbf{Pa}(V_i)))$ is the causal factor of $V_i$. Based on the principle of independent causal mechanisms \cite{scholkopf2012causal}, these causal factors are autonomous of each other. On the basis of this, the interventional distribution is defined as $P(\mathbf{V}|do(V_i\!=\!v_i)):=\prod_{j \neq i} P^e(V_j|\mathbf{Pa}(V_j)))\mathbbm{1}_{V_i=v_i}$. Here $do(V_i\!=\!v_i)$ means lifting $V_i$ from its original causal mechanism $g^e_{i}(\mathbf{Pa}({V_i}), U_i)$ and setting it to a constant value $v_i$. 

In addition to the Markovian assumption, we also assume the causal faithfulness, which enables us to infer the graph structure from probability properties:

\begin{assumption}[Causal faithfulness]
\label{asm:faithfulness}
    For disjoint vertex sets $\mathbf{A},\mathbf{B},\mathbf{Z} \subseteq \mathbf{V}$, $\mathbf{A} \ind \mathbf{B} | \mathbf{Z} \Rightarrow \mathbf{A} \ind_G \mathbf{B} | \mathbf{Z}$. 
\end{assumption}

\textbf{Sparse mechanism shift hypothesis across $\mathcal{E}$.} To build the connection between seen and unseen environments for transfer, we adopt the \emph{sparse mechanism shift hypothesis} \cite{scholkopf2021toward}, \emph{i.e.}, distributional shifts in $P^e(\mathbf{X},Y)$ are the results of changes in only a subset of causal factors. Formally, 
\begin{align}
\label{eq: causal factorization}
    & P^e(\mathbf{X},Y) = P(Y|\mathbf{Pa}(Y)) \prod_{i \in S} P\left(X_i|\mathbf{Pa}(X_i)\right) \nonumber \\
    & \prod_{i \in M} P^e(X_i|\mathbf{Pa}(X_i)), d_S:=|S|, d_M:=|M|,
\end{align}
where $S, M$ respectively denote stable and mutable sets such that each $X_i \in \mathbf{X}_S$ has an invariant causal factor $P(X_i|\mathbf{Pa}(X_i))$; while the factor of each $X_i \in \mathbf{X}_M$ varies across $\mathcal{E}$. Correspondingly, we call $\mathbf{X}_S$ as stable variables and $\mathbf{X}_M$ as mutable variables. In addition to $\mathbf{X}_S$, we also assume the causal factor of $Y$ keeps invariant across $\mathcal{E}$, as widely adopted by the existing literature \cite{arjovsky2019invariant,sun2021recovering,mitrovic2021representation}.

To recover the $\mathbf{X}_M$ from the training distribution, it is also necessary to assume that $\mathcal{E}_{\mathrm{tr}}$ can reflect the mutation of $\mathbf{X}_M$ across $\mathcal{E}$. Formally,
\begin{assumption}[Consistent heterogeneity]
\label{assum:e-train}
For each $X_i \in \mathbf{X}_M$, there exists two different environments $e, e^\prime \in \mathcal{E}_{\mathrm{tr}}$ such that $P^{e}(X_i|\mathbf{Pa}(X_i)) \neq P^{e^\prime}(X_i|\mathbf{Pa}(X_i))$.
\end{assumption}

\textbf{Stable predictor set $\mathcal{F}_S$ via $do(\boldsymbol{x}_M)$.} Based on Eq.~\eqref{eq: causal factorization}, the \citep{subbaswamy2019preventing} obtained a \emph{stable predictor} set $\mathcal{F}_S:=\{f_{S^\prime}(\boldsymbol{x})|S^\prime\!\subseteq \!{S}\}, f_{S^\prime}(\boldsymbol{x})\!:=\!\mathbb{E}[Y|\boldsymbol{x}_{S^\prime}, do(\boldsymbol{x}_M)]$ by intervening on $\mathbf{X}_M$. Compared to \cite{peters2016causal, rojas2018invariant} that only used invariance from stable causal features, these stable predictors in $\mathcal{F}_S$ could additionally exploit invariance from mutable features, thus potentially having better transfer ability.

However, regarding robustness, it remains unknown which predictor in $\mathcal{F}_S$ is optimal. As identifying $f^* \in \mathcal{F}_S$ is equivalent to selecting the optimal stable subset $S^* \subseteq S$ such that $f_{S^*}=f^*$, it turns to the following question: \emph{which subset of $S$ is the most robust one to transfer?} 

\section{Minimax analysis for the optimal subset}
\label{sec.identify}

In this section, we provide a comprehensive minimax analysis to answer the above question. At a first glance, one may take $S$ as optimal since it keeps all stable information. We shall show that this is not necessarily the case. To this end, we first provide a graphical condition for the whole stable set to be optimal, \emph{i.e.}, $S^* = S$. This graphical condition can be easily tested via causal discovery. Second, when this condition is not met, we offer a counter-example in which $S$ is not optimal. Then, to identify $S^*$ in this case, we propose an optimization scheme that is provable to identify the worst-case risk for each subset, equipped with which we can pick up the $S^*$ as the one with minimal worst-case risk. 

\begin{figure}[t]
    \centering
    \includegraphics[width=0.75\columnwidth]{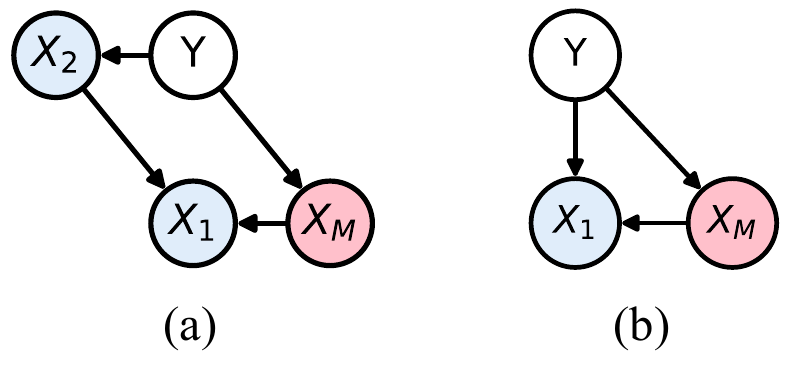}
    \caption{Illustration of the graphical condition in Thm.~\ref{thm:graph degenerate}. Stable and mutable variables are respectively marked blue and red. In both (a) and (b), we have $\mathbf{X}_M^0=\{X_M\}, \mathbf{W}=\{X_1\}$.}
    \label{fig: example graphical condition}
\end{figure}

Next, we first introduce a graphical condition and show that the whole stable set $S$ is optimal under this condition. 

\begin{theorem}[Graphical condition for $S^*\!=\!S$]
\label{thm:graph degenerate}
Suppose Asm.~\ref{eq:assum-dag} holds. Denote $\mathbf{X}_{M}^{0}\!:=\!\mathbf{X}_{M}\cap \mathbf{Ch}(Y)$ as mutable variables in $Y$'s children, and $\mathbf{W}:=\mathbf{De}(\mathbf{X}_{M}^{0}) \backslash \mathbf{X}_M^0$ as descendants of $\mathbf{X}_{M}^{0}$. Then, we have $S^*=S$ if $Y$ does not point to any vertex in $\mathbf{W}$. 
\end{theorem}

{To understand the graphical condition, note that $Y \not \to \mathbf{W}$ enables applying the inference rules \cite{pearl2009causality} to remove the ``do'' in $P(Y|\mathbf{X}_S,do(\boldsymbol{x}_M))$ and degenerate it to a conditional distribution $P(Y|\mathbf{X}^\prime)$, for some $\mathbf{X}^\prime \subseteq \mathbf{X}$. This degeneration allows us to construct a $P^e$ where any other predictor has a larger quadratic loss than $f_S$ \cite{rojas2018invariant}, thus proving the optimality of $S$. Formally, we have the following equivalence result:}
\begin{proposition}
    Under Asm.~\ref{eq:assum-dag}, the graphical condition holds if and only if $P(Y|\mathbf{X}_S,do(\boldsymbol{x}_M))$ can degenerate to a conditional distribution without the ``do''. 
\end{proposition}
\begin{example}
    {To understand this equivalence, consider the DAG shown in Fig.~\ref{fig: example graphical condition} (a), where $Y \not \to \mathbf{W}$. We then have {\small $Y\ind_{G_{\overline{X_M}}}\! X_1,X_M|X_2$} and hence {\small $P(Y|X_1,X_2,do(x_M))=P(Y|X_2)$}. As a contrast, for the DAG shown in Fig.~\ref{fig: example graphical condition} (b), the collider $X_1$ causes {\small $Y \not \ind_{G_{\overline{X_M}}} X_M | X_1$} and prevents the removing of the ``do'' in {\small $P(Y|X_1,do(x_M))$}.}
\end{example}

{The graphical condition can be effectively tested via causal discovery, as guaranteed by the following proposition:}
\begin{proposition}[Testability of Thm.~\ref{thm:graph degenerate}]
\label{prop.thm3.2 justifiability}
{Under Asm.~\ref{eq:assum-dag}-\ref{assum:e-train}, we have that \textbf{i)} the $\mathbf{W}$ is identifiable; and \textbf{ii)} the condition $Y \not \to \mathbf{W}$ is testable from $\{\mathcal{D}_e\}_{e \in \mathcal{E}_{\text{tr}}}$.}
\end{proposition}
\begin{remark}
    {To test $Y \not \to \mathbf{W}$, we first learn the skeleton of $G$, followed by detecting $\mathbf{X}_M^0$ and $\mathbf{W}$ with the heterogeneous causal discovery algorithm CD-NOD \cite{huang2020causal}. Then, we have $Y \not \to \mathbf{W}$ if and only if $Y$ is not adjacent to $\mathbf{W}$ because $\mathbf{W} \subseteq \mathbf{De}(Y)$ by definition. More details are left to Appx.~\ref{sec: appendix discovery}.}
\end{remark}

Thm.~\ref{thm:graph degenerate} only provides a partial characterization for $S$ to be optimal; it is still unclear whether the whole stable set is optimal in all cases. In the following, we give a negative answer with a counter-example, whose DAG of Fig.~\ref{fig: example graphical condition} (b) does not satisfy the graphical condition and $Y, X_M, X_{1}$ are binary variables. We have the following result:
\begin{claim}
\label{prop:counter}
There exists $P(Y)$ and $P(X_{1}|X_M, Y)$, such that $f_{S}(\boldsymbol{x}):=\mathbb{E}[Y|x_{1},do(x_M)]$ has a larger worst-case risk than $f_{\emptyset}(\boldsymbol{x}) := \mathbb{E}[Y|do(x_M)]$: 
\begin{align*}
    \max_{e\in \mathcal{E}} \mathbb{E}_{P^e}[(Y-f_S(\boldsymbol{x}))^2] > \max_{e \in \mathcal{E}} \mathbb{E}_{P^e}[(Y-f_\emptyset(\boldsymbol{x}))^2].
\end{align*}
\end{claim}

\begin{remark}
This result seems surprising as intuitively the whole stable set should be optimal since it fully exploits the stable information, according to existing minimax results in \citep{peters2016causal, rojas2018invariant}. To explain, one should note that these results are built on conditional distributions, {where one can construct a $P^e$ to make any other subset have a larger quadratic loss than $S$.} However, when the interventional distribution can not degenerate, such construction is generally not feasible. Please refer to Appx.~\ref{sec: counter example appendix new} for details. 
\end{remark}

Under general cases where the whole stable set may not be optimal, it remains unknown that \emph{which subset of $S$ is the optimal one to transfer}. To answer this question, we propose to estimate the worst-case risk $\mathcal{R}_{S^\prime}\!:=\!\max_{e \in \mathcal{E}} \mathbb{E}_{P^e}[\left(Y \!-\! f_{S^\prime}(\boldsymbol{x})\right)^2]$ for each subset $S^\prime \subseteq S$; then the $S^*$ corresponds to the subset with minimal $\mathcal{R}$.

For this purpose, we consider a distribution family $\{P_h\}_{h}$, where $h$ maps from $\mathcal{P}a(\mathcal{X}_M)$ to $\mathcal{X}_M$ and $P_h\! :=\! P(Y,\mathbf{X}_S| do(\mathbf{X}_M\!=\!h(\boldsymbol{pa}(\boldsymbol{x}_M)))$. This distribution set keeps the invariant mechanisms of $Y$ and $\mathbf{X}_S$ unchanged while allowing the $\mathbf{X}_M$ given their parents to vary arbitrarily, which can well mimic the distributional shifts among deployed environments in $\mathcal{E}$. Particularly, we show that the worst-case risk $\mathcal{R}_{S^\prime}$ can be attained at some $P_h$, where $h$ is a Borel measurable function. Formally, denote the Borel function set as $\mathcal{B}$, we have: 

\begin{theorem}[Worst-case risk identification]
\label{thm:min-max}
Let $\mathcal{L}_{S^\prime}:=\max_{h \in  \mathcal{B}}\mathbb{E}_{P_h}[(Y \!-\! f_{S^\prime}(\boldsymbol{x}))^2]$ be the maximal population loss over $\{P_h\}_{h \in \mathcal{B}}$ for subset $S^\prime$. Then, we have $\mathcal{L}_{S^\prime}=\mathcal{R}_{S^\prime}$ for each $S^\prime \subseteq S$. Therefore, we have $S^*= \mathop{\mathrm{argmin}}_{S^\prime \subseteq S} \mathcal{L}_{S^\prime}$.
\end{theorem}

This result inspires the following optimization scheme over functions $h \in  \mathcal{B}$ to estimate $\mathcal{R}_{S^\prime}$: 
\begin{align*}
    \max_{h \in  \mathcal{B}} \mathcal{L}_{S^\prime}(h) := \mathbb{E}_{P_h}[(Y \!-\! f_{S^\prime}(\boldsymbol{x}))^2],
\end{align*}
as the optimality of which is assured to attain $\mathcal{R}_{S^\prime}$. To implement, we parameterize $h$ with a multilayer perceptron (MLP) $h_\theta$ and optimize over $\theta$, due to the ability of MLP to approximate any Borel function \cite{hornik1989multilayer}. To show the tractability of this optimization, we have the following identifiability result for $\mathcal{L}_{S^\prime}(h)$:

\begin{proposition}
	\label{prop.thm3.3 justifiability}
	Under Asm.~\ref{eq:assum-dag}-\ref{assum:e-train}, the  $P_h$, $f_{S^\prime}$, and hence $\mathcal{L}_{S^\prime}(h)$ are identifiable. 
\end{proposition}

\section{Searching $S^*$ among equivalence classes}
\label{sec.learn}

In this section, we provide Alg.~\ref{alg:identify-f-star} to identify $S^*$, which combines Thm.~\ref{thm:graph degenerate} and Thm.~\ref{thm:min-max}. Specifically, Alg.~\ref{alg:identify-f-star} returns $S$ as $S^*$ (line 3), if the graphical condition $Y \not\to \mathbf{W}$ is tested true. Otherwise, it searches over subsets to identify $S^*$ in terms of the estimated worst-case risk $\mathcal{L}$. For this purpose, a simple search method that is commonly adopted in the literature \cite{peters2016causal,rojas2018invariant,magliacane2018domain,subbaswamy2019preventing} is to \emph{exhaustively} search over all subsets of $S$. 

\begin{algorithm}[b!]
   \caption{Optimal subset $S^*$ selection.}
\label{alg:identify-f-star}

\begin{flushleft}
    \textbf{Input:} The training data $\{\mathcal{D}_{e}\}_{e\in\mathcal{E}_{\text{tr}}}$.
\end{flushleft}

\begin{algorithmic}[1]
   \STATE Learn the skeleton of $G$; detect $\mathbf{X}_M^0$, $\mathbf{W}$.
   \STATE \textbf{if} {$Y \not \to \mathbf{W}$} \textbf{then}
        \STATE \quad $S^* \leftarrow S$. \hfill{\# Thm.~\ref{thm:graph degenerate}} 
   \STATE \textbf{else}
    \STATE \quad  Recover $\mathrm{Pow}(S)/\!\sim_G$ with Alg.~\ref{alg: recover g-equivalence}.
    \STATE \quad  $\mathcal{L}_{\min}\leftarrow \infty$.
        \STATE \quad \textbf{for} {$[S^\prime]$ \textbf{in} $\mathrm{Pow}(S)/ \!\sim_G$} \textbf{do}
        \STATE \quad \quad \textbf{if} {$\mathcal{L}_{S^\prime} <\mathcal{L}_{\min}$} \textbf{then}
            \STATE \quad \quad \quad $\mathcal{L}_{\min} \leftarrow \mathcal{L}_{S^\prime}$, $S^* \leftarrow S^\prime$. \hfill{\# Thm.~\ref{thm:min-max}} 
            \STATE \quad \quad \textbf{end if}
        \STATE \quad \textbf{end for}
    \STATE \textbf{end if} 
    \STATE \textbf{return} $S^*$.
\end{algorithmic}
\end{algorithm}

In the following, we provide a new search strategy with better efficiency, by noticing that the exhaustive search can be redundant for subsets that have the same worst-case risk. Formally, we introduce the equivalence relation as follows:

\begin{definition}[Equivalence relation]
\label{def:equi}
     Consider a general graph $G$ over the target $Y$ and covariates $\mathbf{X}$. Let $\sim_G$ be an equivalence relation on all subsets of $\{1,...,\mathrm{dim}(\mathbf{X})\}$. We say $S^{\prime} \sim_G S^{\prime \prime}$ if $\exists \, S_{\cap} \subseteq S^\prime \cap S^{\prime \prime}$ such that:
    \begin{align}
    \label{eq:sim-G-def2}
    Y \ind_{G}  \mathbf{X}_{S_\cap^c} | \mathbf{X}_{S_\cap}, \text{ where } S_\cap^c:=(S^\prime\cup S^{\prime \prime}) \backslash S_\cap.
    \end{align}
    We call elements of the quotient space $\mathrm{Pow}(S)/\!\sim_G$ as equivalence classes. We use $[S^\prime] :=\{S^{\prime \prime}|S^{\prime \prime} \sim_G S^\prime\}$ to denote the equivalence class of $S^\prime$ and $N_G:=|\mathrm{Pow}(S)/\!\sim_G\!|$ to denote the number of equivalence classes. 
\end{definition}

\begin{remark}
The causal graph $G$ in Def.~\ref{def:equi} can be a Maximal Ancestral Graph (MAG) \cite{spirtes2000causation}, where bi-directed edges ($\leftrightarrow$) and undirected edges ($-$) exist due to unobserved confounders and selection variables, respectively. Correspondingly, ``$\ind_{G}$" in Eq.~\eqref{eq:sim-G-def2} refers to $m$-separation. 
\end{remark}

In our scenario, we are interested in the $\sim_G$ relation between stable subsets in the subgraph $G_S$ over $\mathbf{X}_S \cup Y$, which corresponds to conditioning on  ``$do(\boldsymbol{x}_M)$" in $G$. According to Def.~\ref{def:equi}, two stable subsets $S^\prime$ and $S^{\prime \prime}$ are equivalent if they share an intersection set $S_\cap$ that can \emph{d}-separate $S' \backslash S_\cap$ and $S^{\prime \prime} \backslash S_\cap$ from $Y$. As a result, we have {$P(Y|\mathbf{X}_{S^\prime},do(\boldsymbol{x}_M)) \!=\! P(Y|\mathbf{X}_{S_\cap},do(\boldsymbol{x}_M)) = P(Y|\mathbf{X}_{S^{\prime \prime}},do(\boldsymbol{x}_M))$} and hence $\mathcal{R}_{S^\prime}\!=\!\mathcal{R}_{S^{\prime \prime}}$. For example, in Fig.~\ref{fig: example graphical condition} (a), we have {\small {$\{X_2\} \sim_G \{X_1,X_2\}$}} as {\small{$\mathbf{X}_{S_\cap}=\{X_2\}$}} \emph{d}-separates {\small $\mathbf{X}_{S_\cap^c}=\{X_1,X_2\}\backslash \{X_2\}=\{X_1\}$} from $Y$ in $G_S$.

\begin{algorithm}[t!]
   \caption{Equivalence classes recovery.}
   \label{alg: recover g-equivalence}
\begin{algorithmic}[1]
       \FUNCTION{$\text{recover}(G)$}
        \IF{$\mathbf{Neig}(Y)=\emptyset$}
        \STATE \textbf{return} $\{\mathrm{Pow}(S)\}$.
        \ELSE
        \STATE $\mathrm{Pow}(S)/\!\sim_G \leftarrow \emptyset$.
        \FOR{$S^\prime \subseteq \mathbf{Neig}(Y)$}
        \STATE {\small Construct a $\mathrm{MAG}$ $M_G$ over $S \backslash \mathbf{Neig}(Y)$, with $S^\prime$ as the selection set, $\mathbf{Neig}(Y)\backslash S^\prime$ as the latent set.}
        \vspace{+0.05cm}
        \STATE {\small $\mathrm{Pow}(S\backslash \mathbf{Neig}(Y))/\!\sim_{M_G} \leftarrow \textbf{\textbf{recover}}(M_G)$.}
        \vspace{+0.05cm}
        \STATE {\small Add $S^\prime$ to each subset in $\mathrm{Pow}(S\backslash \mathbf{Neig}(Y))/\!\sim_{M_G}$.}
        \STATE {\small $\mathrm{Pow}(S)/\!\sim_G\text{.append}(\mathrm{Pow}(S\backslash \mathbf{Neig}(Y))/\!\sim_{M_G})$.}
        \ENDFOR
        \STATE\textbf{return} $\mathrm{Pow}(S)/\!\sim_G$.
        \ENDIF
       \ENDFUNCTION
\end{algorithmic}
\begin{flushleft}
\textbf{Input:} The causal graph $G$.
\end{flushleft}
\begin{algorithmic}[1]
    \STATE Let $G_S$ the subgraph of $G$ over $\mathbf{X}_S\cup Y$.
    \STATE \textbf{return} \textbf{recover}$(G_S)$.
\end{algorithmic}
\end{algorithm}

With this $\sim_G$ equivalence, we only need to search equivalence classes, rather than all subsets. To enable this search, we provide Alg.~\ref{alg: recover g-equivalence} to recover the $\mathrm{Pow}(S)/\!\sim_G$ in a recursive manner. Specifically, given the input graph $G$, we first obtain the subgraph $G_S$ by removing $\mathbf{X}_M$ in $G$. Then we find $Y$'s neighbors. Since any two vertices in $\mathbf{Neig}(Y)$ cannot $d$-separate each other from $Y$, we go over each subset $S^\prime \subseteq \mathbf{Neig}(Y)$ to construct a MAG over vertices other than $\mathbf{Neig}(Y)$, with $S^\prime$ as the selection set and $\mathbf{Neig}(Y) \backslash S^\prime$ as the latent set. Then it is left to recover equivalence classes in each MAG, {and include them to $\mathrm{Pow}(S)/\!\sim_G$} after appending the selection set $S^\prime$ (line 9,10). We recursively repeat the above procedure until $\mathbf{Neig}(Y)$ is empty, which indicates all subsets are equivalent since all of them are $d$-separated from $Y$. To illustrate, consider the following Exam.~\ref{exam:g-recover}.

\begin{figure}[t!]
    \centering
    \includegraphics[width=\columnwidth]{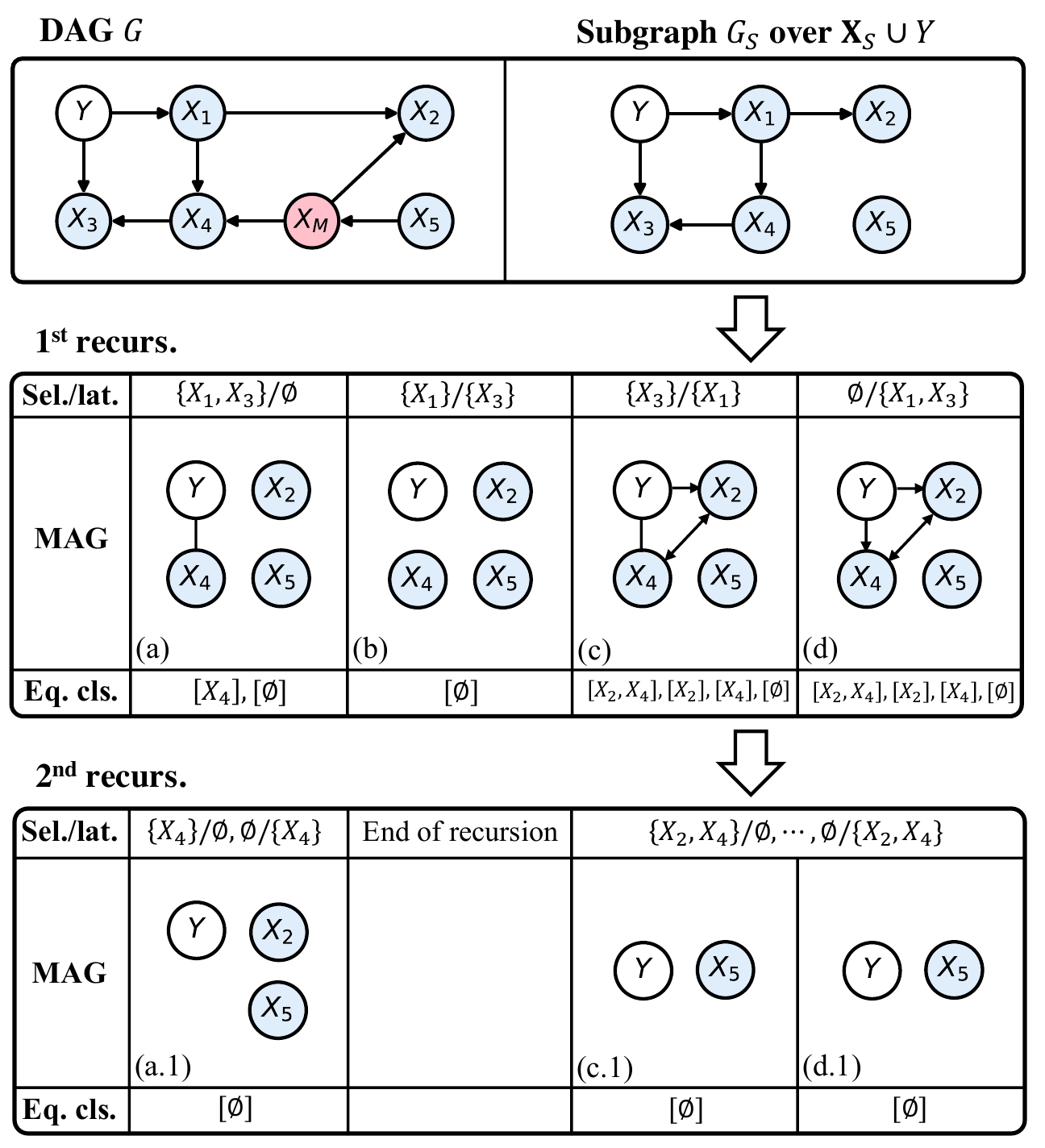}
    \caption{An example to illustrate Alg.\ref{alg: recover g-equivalence}. Stable and mutable variables are respectively marked blue and red.}
    \label{fig:alg_recover_example}
\end{figure}

\vspace{+0.1cm}
\begin{example}
\label{exam:g-recover}
    Consider the causal graph $G$ shown in Fig.~\ref{fig:alg_recover_example}. We first obtain the $G_S$ over $\mathbf{X}_S \cup Y$, where {\small $\mathbf{Neig}(Y)=\{X_1,X_3\}$}. We then take each subset {\small $S^\prime \subseteq \{X_1,X_3\}$} as the selection set and {\small $\{X_1,X_3\} \backslash S^\prime$} as the latent set to respectively construct MAGs (a-d) in the first recursion. For (a) with {\small $\mathbf{Neig}(Y)=\{X_4\}$}, we both obtain the MAG in (a.1) when taking {\small $\{X_4\}$} (resp. $\emptyset$) and $\emptyset$ (resp. {\small $\{X_4\}$}) as the selection set (resp. latent set). Since {\small $\mathbf{Neig}(Y)=\emptyset$} in (a.1), there is only one equivalence class {\small $[\emptyset]:=\mathrm{Pow}(\{X_2,X_5\})$}. Following line 9 in Alg.~\ref{alg: recover g-equivalence}, we append {\small $X_4$} and $\emptyset$ to each subset in equivalence classes of (a.1) to obtain the equivalence classes of (a): {\small $[X_4]$} and $[\emptyset]$. Similarly, after appending the selection set {\small $S'=\{X_1,X_3\}$}, we include {\small $[X_1,X_3,X_4]$} and {\small $[X_1,X_3]$} to {\small $\mathrm{Pow}(S)/\!\sim_G$}. We similarly apply this procedure to (b),(c),(d), which respectively contribute equivalence classes {\small $\{[X_1]\}$, $\{[X_3],[X_2,X_3],[X_3,X_4],[X_2,X_3,X_4]\}$}, and {\small $\{[\emptyset],[X_2],[X_4],[X_2,X_4]\}$} to {\small $\mathrm{Pow}(S)/\!\sim_G$}. 
\end{example}

In practice, we cannot access the true causal graph $G$ but can only recover the graph that is Markovian equivalent to $G$. The following proposition shows that Alg.~\ref{alg: recover g-equivalence} can still recover $\mathrm{Pow}(S)/\! \sim_G$ in this case. 

\begin{proposition}
\label{prop:recover}
    Under Asm.~\ref{eq:assum-dag}, \ref{asm:faithfulness}, for each input graph that is Markov equivalent to the ground-truth $G$, Alg.~\ref{alg: recover g-equivalence} can correctly recover the $\mathrm{Pow}(S)/\! \sim_G$.
\end{proposition}

{Besides, we in Appx.~\ref{appx:complexity alg recover} show that the complexity of Alg.~\ref{alg: recover g-equivalence} is $O(N_G)$, \emph{i.e.}, same as the complexity of searching $N_G$ equivalence classes, which is discussed as follows.}

\noindent\textbf{Searching complexity.} We show that compared to the exponential cost $O(2^{d_S})$ of exhaustive search, our search strategy enjoys a polynomial cost $\mathrm{P}(d_S)$ when $G_S$ is mainly composited of chain vertices. Here, a chain vertex is a vertex of degree $\leq 2$, and a chain is a sequence of connected chain vertices. Specifically, we have the following result:
\begin{proposition}[Complexity (informal)]
    \label{prop:complexity(informal)}
    Let $d_{\leq 2}$ and $d_{>2}\!:=\!d_S\!-\!d_{\leq 2}$ respectively denote the number of chain vertices and non-chain vertices. When the chain vertices are ``distributed intensively", $N_G=\mathrm{P}(d_S)$ if and only if $d_{>2}=O(\log(d_S))$.
\end{proposition}

Here, ``distributed intensively" means that chain vertices compose only a few chains. Roughly speaking, this is because when the graph is composed of multiple chains that do not intersect each other, $N_G$ is determined by the product of multiple chains' lengths. As a result, the $N_G$ tends to be smaller when the number of chains is small. Formal and more general results are left to Appx.~\ref{sec: appendix complexity}.

\section{Experiment}

\begin{figure*}[t!]
    \centering
    \includegraphics[width=.9\textwidth]{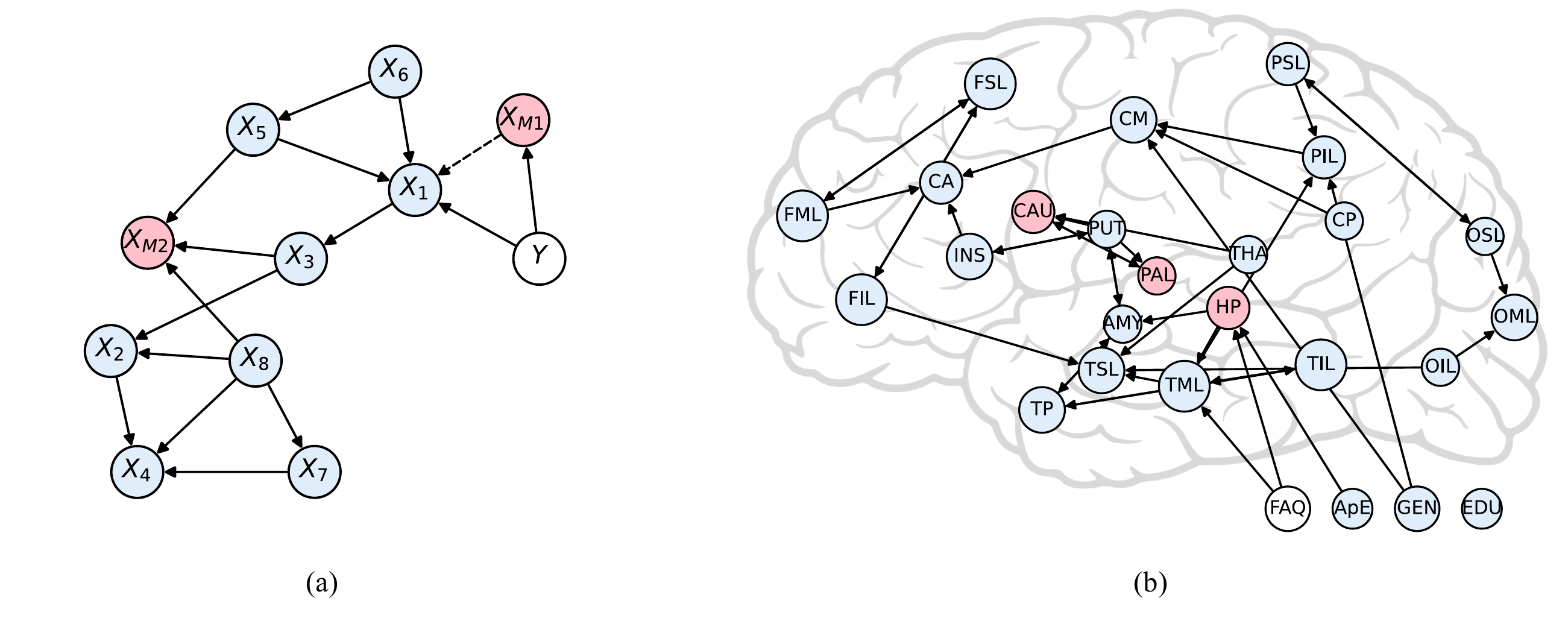}
    \caption{(a) The causal graph for synthetic data generation. Stable and mutable variables are respectively marked blue and red. The dashed edge $X_{M_1} \dashrightarrow X_1$ does not exist (resp. exist) in setting-1 (resp. setting-2). (b) The learned causal graph on ADNI. The target (FAQ) and biomarkers (ApE, GEN, EDU) are placed in the bottom right. Brain regions are placed at their positions in the brain.}
    \label{fig:simg3_ad_causal_graphs}
\end{figure*}

\begin{table*}[t!]
\caption{Evaluation on synthetic and ADNI datasets. The first column notes the methods we compare. The second and third columns respectively represent the maximal MSE and standard deviation of MSE over deployment environments. The best results are \textbf{boldfaced}.}
    \label{tab: numerical value, sythetic and ad}
    \centering
    \scalebox{0.9}{
    \begin{tabular}{c|c|c|c|c|c|c}
    \hline
    \multirow{2}{*}{Method} &  \multicolumn{3}{c|}{max. MSE ($\downarrow$)} & \multicolumn{3}{c}{std. MSE ($\downarrow$)}\\
    \cline{2-7} 
     & Syn1 & Syn2 & ADNI & Syn1 & Syn2 & ADNI \\
    \hline
    Vanilla & $1.336_{\pm 0.4}$ & $1.861_{\pm 0.4}$ & $1.399_{\pm 0.1}$ & $0.240_{\pm 0.2}$ & $0.481_{\pm 0.1}$ & $0.299_{\pm 0.0}$ \\ 
    ICP \cite{peters2016causal} & $1.855_{\pm 0.7}$ & $2.331_{\pm 0.2}$ & $1.176_{\pm 0.0}$ & $0.130_{\pm 0.1}$ & $0.230_{\pm 0.0}$ & $0.155_{\pm 0.0}$  \\ 
    IC \cite{rojas2018invariant}  & $1.211_{\pm 0.4}$ & $1.254_{\pm 0.1}$ & $1.165_{\pm 0.2}$ &  $0.176_{\pm 0.2}$ & $0.194_{\pm 0.1}$ & $0.198_{\pm 0.1}$\\ 
    DRO \cite{sinha2018certifiable}  & $1.364_{\pm 0.5}$ & $1.495_{\pm 0.1}$ & $1.181_{\pm 0.0}$ & $0.250_{\pm 0.2}$ & $0.326_{\pm 0.0}$ & $0.145_{\pm 0.0}$  \\ 
    Surgery \cite{subbaswamy2019preventing}  & ${0.926_{\pm 0.0}}$ & $1.101_{\pm 0.1}$ & $1.069_{\pm 0.1}$ & ${0.028_{\pm 0.0}}$ & $0.057_{\pm 0.0}$ & $0.129_{\pm 0.0}$  \\ 
    \hline
    IRM \cite{arjovsky2019invariant}  & $1.106_{\pm 0.2}$ & $1.246_{\pm 0.1}$ & $1.223_{\pm 0.0}$ & $0.127_{\pm 0.1}$ & $0.164_{\pm 0.1}$  & $0.177_{\pm 0.0}$ \\ 
    HRM \cite{liu2021heterogeneous}  & $0.975_{\pm 0.0}$ & $1.494_{\pm 0.1}$ & $1.272_{\pm 0.1}$ & $0.046_{\pm 0.0}$ & $0.312_{\pm 0.1}$ & $0.194_{\pm 0.1}$  \\ 
    IB-IRM \cite{ahuja2021invariance}  & $1.076_{\pm 0.0}$ & ${1.079_{\pm 0.0}}$ & $1.222_{\pm 0.2}$ & $0.056_{\pm 0.0}$ & $0.040_{\pm 0.0}$ & $0.113_{\pm 0.1}$  \\ 
    AncReg \cite{rothenhausler2021anchor}  & $0.938_{\pm 0.0}$ & $1.377_{\pm 0.2}$ & $1.138_{\pm 0.1}$ & $0.033_{\pm 0.0}$ & $0.257_{\pm 0.1}$ & $0.159_{\pm 0.0}$ \\ 
    \hline
    Ours (Alg.~\ref{alg:identify-f-star})  & $\mathbf{0.926_{\pm 0.0}}$ & $\mathbf{1.079_{\pm 0.0}}$ & $\mathbf{0.890_{\pm 0.1}}$ & $\mathbf{0.028_{\pm 0.0}}$ & $\mathbf{0.034_{\pm 0.0}}$ & $\mathbf{0.038_{\pm 0.0}}$ \\ 
    \hline
    \end{tabular}}
\end{table*}

We evaluate our method on synthetic data and a real-world application, \emph{i.e.}, diagnosis of Alzheimer's disease\footnote{Code is available at \url{https://github.com/lmz123321/which_invariance}.}.

\noindent \textbf{Compared baselines.} \textbf{i)} \textbf{Vanilla} that uses $\mathbb{E}[Y|x]$ to predict $Y$; \textbf{ii)} \textbf{ICP} \citep{peters2016causal} that assumed and used the invariance of parental features $P(Y|\mathbf{Pa}(Y))$; \textbf{iii)} \textbf{IC} \citep{rojas2018invariant} that extended ICP to features beyond $\mathbf{Pa}(Y)$; \textbf{iv)} \textbf{DRO} \cite{sinha2018certifiable} that constrained the distance between training and deployed distributions and conducted optimization for robustness; \textbf{v)} \textbf{Surgery estimator} \cite{subbaswamy2019preventing} that used validation's loss to identify the optimal subset; \textbf{vi)} \textbf{IRM} \citep{arjovsky2019invariant} that learned an invariant representation to transfer; \textbf{vii)} \textbf{HRM} \citep{liu2021heterogeneous} that extended IRM to cases with unknown environmental indices, by exploring the heterogeneity in data via clustering; \textbf{viii)} \textbf{IB-IRM} \citep{ahuja2021invariance} that leveraged the information bottleneck to supplement the invariance principle in IRM; and \textbf{ix)} \textbf{Anchor regression} \citep{rothenhausler2021anchor} that interpolated between ordinary least square (LS) and causal minimax LS. 

\noindent\textbf{Evaluation metrics.} We use the maximal mean square error (max. MSE) and the standard deviation of MSE (std. MSE) over deployed environments to evaluate the robustness and stability of predictors, respectively.

\noindent \textbf{Implementation details.} We use two-layer nonlinear MLPs to implement the $f_{S^\prime}$ and $h_\theta$. Hyperparameter settings of our method and baselines are left in Appx.~\ref{sec. appendix implementation}.

\begin{table}[t!]
\caption{Comparison of computational cost on ADNI.}
    \label{tab:ad search cost}
    \centering
    \scalebox{0.9}{
    \begin{tabular}{c c c}
    \hline
    Method & Searching cost & Time \\
    \hline
    Exhaustive ($\mathrm{Pow}(S)$) & $2^{25}$ & about 6.4y\\
    Ours ($\mathrm{Pow}(S)/\!\sim_G$) & $25307$ & 42h\\
    \hline
    \end{tabular}}
\end{table}

\subsection{Synthetic data}

\noindent\textbf{Data generation.} We use the DAG in Fig.~\ref{fig:simg3_ad_causal_graphs} (a) and the structural equation {\small $V_i = \alpha_i^e g_i\left(\sum_{V_j \in \mathbf{Pa}(V_i)} \beta_{i,j} V_j\right) + \varepsilon_i$} to generate data, where $\alpha_i^e$ keeps constant, \emph{i.e.}, $\alpha_i^e \equiv \alpha_i$ for all $e$ if $V_i$ is a stable variable; or varies with $e$ if $V_i$ is a mutable variable. For each $i$, the function $g_i$ is randomly chosen from \{\emph{identity}, \emph{tanh}, \emph{sinc}, \emph{sigmoid}\}. Each linear parameter $\beta_{i,j}$ is randomly drawn from a uniformed distribution $\mathcal{U}([-2,-0.5]\cup [0.5,2])$ and the noise item $\varepsilon_i \sim \mathcal{N}(0,0.1)$. We generate 20 environments and $n_e=100$ samples in each environment. To remove the effect of randomness, we repeat 5 times: each time we randomly pick 10 environments for training and the others for deployment.

We consider two different settings, with the graphical condition holds (resp. not hold) in setting-1 (resp. setting-2). Specifically, according to the definition, $\mathbf{X}_M^0=\{X_{M_1}\}$. In setting-1, the dashed edge $X_{M_1} \dashrightarrow X_1$ does not exist, hence $\mathbf{W}$ is empty and the graphical condition $Y \not\to \mathbf{W}$ in Thm.~\ref{thm:graph degenerate} holds. In this regard, the whole stable set is expected to be optimal. In setting-2, the edge $X_{M_1} \to X_1$ exists, hence $\mathbf{W}:=\{X_1,X_2,X_3,X_4,X_{M_2}\}$ and $Y \to \mathbf{W}$ that violates the graphical condition. In this regard, the $S^\prime$ with minimal $\mathcal{L}_{S^\prime}$ is expected to be optimal.

\begin{table}[t!]
\caption{Std. over equivalent subsets on synthetic data.}
    \label{tab:verify equivalence}
    \centering
    \scalebox{0.9}{
    \begin{tabular}{c c}
    \hline
    Metric \quad \quad \quad \quad  & Value \\
    \hline
     Inter-class std. \quad \quad \quad \quad  & $1.000$ \\
     Intra-class std. \quad \quad \quad \quad  & $0.008$ \\
    \hline
    \end{tabular}}
\end{table}

\begin{figure*}[t!]
    \centering
    \includegraphics[width=\textwidth]{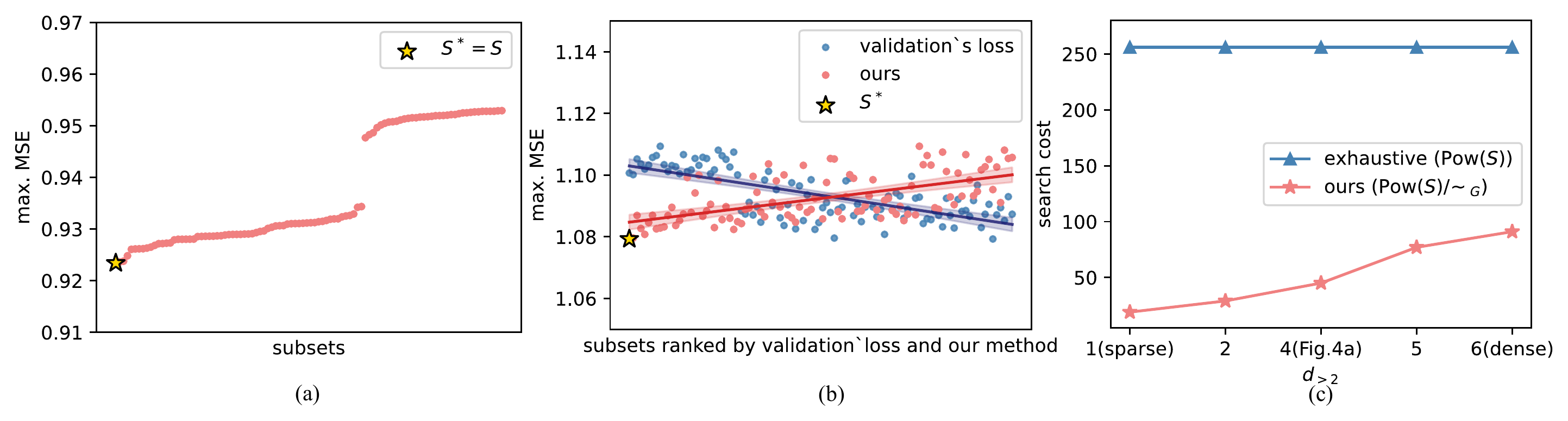}
    \caption{Results on synthetic data. (a) Setting-1: max. MSE of different subsets, where the whole stable set $S$ is optimal. (b) Setting-2: max. MSE of subsets ranked in the ascending order from left to right, respectively according to the estimated $\mathcal{L}$ of our method and the validation's loss adopted by \citep{subbaswamy2019preventing}. (c) Comparison of searching cost when $d_{> 2}$ increases.}
    \label{fig:simg3_scatters_search_cost}
\end{figure*}

\noindent\textbf{Results.} We report the max. MSE and std. MSE over deployment environments in Tab.~\ref{tab: numerical value, sythetic and ad}. As shown, our method outperforms the others in all settings, indicating better robustness (max. MSE) and stability (std. MSE). Besides, we report the max. MSE of different subsets in Fig.~\ref{fig:simg3_scatters_search_cost}. As shown, in setting-1, the whole stable set $S$ has the minimal max. MSE as expected; in setting-2, the subset with minimal $\mathcal{L}$ also has the minimal max. MSE over deployed environments. Besides, we can observe that the max. MSE shows an approximate increasing trend in subsets ranked by our method; as a contrast, the trend is decreasing in those ranked by the validation's loss adopted by \citep{subbaswamy2019preventing}. This result suggests that our method can consistently reflect the worst-case risk.

\textbf{Analysis of $\sim_G$ equivalence.} To show the effectiveness of Alg.~\ref{alg: recover g-equivalence} in recovering equivalence classes, we compute the \emph{intra-class standard deviation}, and compare it with \emph{inter-class std.}, in terms of max. MSE. For intra-class std., we first compute the standard deviation of max. MSE over all subsets in each equivalence class, then take the average over all equivalence classes. For inter-class std, we first compute the average max. MSE over all subsets in each equivalence class; then we compute the std. of the average max. MSE over equivalence classes. In Tab.~\ref{tab:verify equivalence}, we observe that the intra-class std. is much smaller than the inter-class std. This result suggests that our Alg.~\ref{alg: recover g-equivalence} to identify equivalent subsets is effective enough to guarantee the validity of searching over only equivalence classes rather than all subsets. 

\textbf{Searching complexity.} We first generate a sequence of causal graphs (Fig.~\ref{fig:simu1245_causal_graph}) with $d_{> 2}$ growing, by deleting/adding edges in the graph shown in Fig.~\ref{fig:simg3_ad_causal_graphs} (a) and then compute the searching cost for these graphs. We can see in Fig.~\ref{fig:simg3_scatters_search_cost} (c) that \textbf{i}) compared with the exhaustive search, our method can significantly save the searching cost in both sparse and dense graphs; \textbf{ii}) the searching cost over equivalence classes decreases when $d_{>2}$ decreases. 

\subsection{Alzheimer's disease diagnosis}

\noindent\textbf{Dataset \& preprocessing.} We consider the Alzheimer's Disease Neuroimaging Initiative \cite{petersen2010alzheimer} (ADNI) dataset, in which the imaging data is acquired from structural Magnetic Resonance Imaging (sMRI) scans. We apply the Dartel VBM \cite{ashburner2007fast} for preprocessing and the Statistical Parametric Mapping (SPM) for segmenting brain regions. Then, we implement the Automatic Anatomical Labeling (AAL) atlas \cite{tzourio2002automated} and region indices provided by \citep{young2018uncovering} to partition the whole brain into 22 regions (Tab.~\ref{tab: aal indices}). In addition to brain region volumes, we also include demographics (age, gender (GEN)) and genetic information (the number of ApoE-4 alleles (ApE)). With these covariates, we predict the Functional Activities Questionnaire (FAQ) score \cite{mayo2016use} for each patient. We split the dataset into seven environments according to age ($<$60, 60-65, 65-70, 70-75, 75-80, 80-85, $>$85), which respectively contain $n_e=$27,59,90,240,182,117,42 samples. We repeat 3 times, with each time randomly taking four environments for training and the rest for deployment. 

\noindent\textbf{Causal discovery.} The learned causal graph is shown in Fig.~\ref{fig:simg3_ad_causal_graphs} (b). As we can see, the affection of AD (measured by FAQ score) first shows in the hippocampus (HP) and medial temporal lobe (TML), then propagates to other brain regions, which echos existing studies that the HP and TML are early degenerated regions \cite{barnes2009meta,duara2008medial}, Besides, we observe that the caudate (CAU), pallidum (PAL), and hippocampus (HP) are mutable regions, which agrees with the heterogeneity found in different age groups \cite{cavedo2014medial,fiford2018patterns}.

\noindent\textbf{Equivalence and searching complexity.} As shown in Fig.~\ref{fig:simg3_ad_causal_graphs} (b), we have $\text{FAQ} \to \text{TML}$, which violates the graphical condition ($\text{TML} \in \mathbf{W}$) in Thm.~\ref{thm:graph degenerate}. We thus search over equivalence classes to find $S^*$. As shown in Tab.~\ref{tab:ad search cost}, there are only 25307 equivalence classes out of the $2^{25}$ subsets. Correspondingly, the training time can be saved from about 55,687 hours $\approx$ 6.4 years to only 42 hours.

\noindent\textbf{Results.} Fig.~\ref{fig:adni-intro} (a) shows the max. MSE of our method and baselines. As we can see, our method significantly outperforms the others, which demonstrates the effectiveness of Thm.~\ref{thm:min-max} in robust subset selection. Further, Fig.~\ref{fig:adni-intro} (b) shows that the max. MSE of subsets ranked by our method appears a positive correlation with the true worst-case risk; as a contrast, the correlation is negative for the max. MSE of subsets ranked by the validation`s loss. Particularly, the top subset selected by our method {\small \{FSL,TSL,TIL,PSL,OML,CM\}} reaches a max. MSE of 0.890; while the one selected by the validation`s loss {\small \{FSL,FML,TSL,TIL,PSL,CA,THA,GEN\}} only has a max. MSE of 1.069. These results demonstrate the effectiveness of our method in estimating the worst-case risk. The improvements over ICP, IRM, and their extensions can be attributed to the property use of invariance beyond stable causal features/representations. The advantage over DRO may lie in the robustness of our method beyond bounded distributional shifts; while the advantage over Anchor regression can be contributed to the relaxation of the linearity assumption.

\section{Conclusion}
\label{sect:conclusion}

In this paper, we propose a causal minimax learning approach to identify the optimal subset of invariance to transfer, in order to achieve robustness against dataset shifts. We first provide a graphical condition that is sufficient for the whole stable set to be optimal. When this condition fails, we propose an optimization-based approach that is provable to attain the worst-case risk for each subset. Further, we propose a new search strategy via $d$-separation, which enjoys better efficiency. The subset selected by our method outperforms the others in terms of robustness on Alzheimer's disease diagnosis. In the future, we are interested to extend our results to cases where the DAG is allowed to vary, which may happen when there are many deployed environments.

\section*{Acknowledgements}
This work was supported by MOST-2018AAA0102004. 


\bibliography{reference}

\begin{thebibliography}{43}
\providecommand{\natexlab}[1]{#1}
\providecommand{\url}[1]{\texttt{#1}}
\expandafter\ifx\csname urlstyle\endcsname\relax
  \providecommand{\doi}[1]{doi: #1}\else
  \providecommand{\doi}{doi: \begingroup \urlstyle{rm}\Url}\fi

\bibitem[Ahuja et~al.(2021)Ahuja, Caballero, Zhang, Gagnon-Audet, Bengio,
  Mitliagkas, and Rish]{ahuja2021invariance}
Ahuja, K., Caballero, E., Zhang, D., Gagnon-Audet, J.-C., Bengio, Y.,
  Mitliagkas, I., and Rish, I.
\newblock Invariance principle meets information bottleneck for
  out-of-distribution generalization.
\newblock \emph{Advances in Neural Information Processing Systems}, 34, 2021.

\bibitem[Arjovsky et~al.(2019)Arjovsky, Bottou, Gulrajani, and
  Lopez-Paz]{arjovsky2019invariant}
Arjovsky, M., Bottou, L., Gulrajani, I., and Lopez-Paz, D.
\newblock Invariant risk minimization.
\newblock \emph{arXiv preprint arXiv:1907.02893}, 2019.

\bibitem[Ashburner(2007)]{ashburner2007fast}
Ashburner, J.
\newblock A fast diffeomorphic image registration algorithm.
\newblock \emph{Neuroimage}, 38\penalty0 (1):\penalty0 95--113, 2007.

\bibitem[Ausset et~al.(2022)Ausset, Cl{\'e}men{\c{c}}on, and
  Portier]{ausset2022empirical}
Ausset, G., Cl{\'e}men{\c{c}}on, S., and Portier, F.
\newblock Empirical risk minimization under random censorship.
\newblock \emph{Journal of Machine Learning Research}, 2022.

\bibitem[Barnes et~al.(2009)Barnes, Bartlett, van~de Pol, Loy, Scahill, Frost,
  Thompson, and Fox]{barnes2009meta}
Barnes, J., Bartlett, J.~W., van~de Pol, L.~A., Loy, C.~T., Scahill, R.~I.,
  Frost, C., Thompson, P., and Fox, N.~C.
\newblock A meta-analysis of hippocampal atrophy rates in alzheimer's disease.
\newblock \emph{Neurobiology of aging}, 30\penalty0 (11):\penalty0 1711--1723,
  2009.

\bibitem[Berge(1963)]{berge1963topologies}
Berge, C.
\newblock \emph{Topological Spaces}.
\newblock Oliver and Boyd, London., 1963.

\bibitem[B{\"u}hlmann(2020)]{buhlmann2020invariance}
B{\"u}hlmann, P.
\newblock Invariance, causality and robustness.
\newblock \emph{Statistical Science}, 35\penalty0 (3):\penalty0 404--426, 2020.

\bibitem[Carr et~al.(1994)Carr, Roth, Luther, Rose, and
  Springer]{carr1994monocyte}
Carr, M.~W., Roth, S.~J., Luther, E., Rose, S.~S., and Springer, T.~A.
\newblock Monocyte chemoattractant protein 1 acts as a t-lymphocyte
  chemoattractant.
\newblock \emph{Proceedings of the National Academy of Sciences}, 91\penalty0
  (9):\penalty0 3652--3656, 1994.

\bibitem[Cavedo et~al.(2014)Cavedo, Pievani, Boccardi, Galluzzi, Bocchetta,
  Bonetti, Thompson, and Frisoni]{cavedo2014medial}
Cavedo, E., Pievani, M., Boccardi, M., Galluzzi, S., Bocchetta, M., Bonetti,
  M., Thompson, P.~M., and Frisoni, G.~B.
\newblock Medial temporal atrophy in early and late-onset alzheimer's disease.
\newblock \emph{Neurobiology of aging}, 35\penalty0 (9):\penalty0 2004--2012,
  2014.

\bibitem[Duara et~al.(2008)Duara, Loewenstein, Potter, Appel, Greig, Urs, Shen,
  Raj, Small, Barker, et~al.]{duara2008medial}
Duara, R., Loewenstein, D., Potter, E., Appel, J., Greig, M., Urs, R., Shen,
  Q., Raj, A., Small, B., Barker, W., et~al.
\newblock Medial temporal lobe atrophy on mri scans and the diagnosis of
  alzheimer disease.
\newblock \emph{Neurology}, 71\penalty0 (24):\penalty0 1986--1992, 2008.

\bibitem[Fiford et~al.(2018)Fiford, Ridgway, Cash, Modat, Nicholas, Manning,
  Malone, Biessels, Ourselin, Carmichael, et~al.]{fiford2018patterns}
Fiford, C.~M., Ridgway, G.~R., Cash, D.~M., Modat, M., Nicholas, J., Manning,
  E.~N., Malone, I.~B., Biessels, G.~J., Ourselin, S., Carmichael, O.~T.,
  et~al.
\newblock Patterns of progressive atrophy vary with age in alzheimer's disease
  patients.
\newblock \emph{Neurobiology of aging}, 63:\penalty0 22--32, 2018.

\bibitem[Ghassami et~al.(2018)Ghassami, Kiyavash, Huang, and
  Zhang]{ghassami2018multi}
Ghassami, A., Kiyavash, N., Huang, B., and Zhang, K.
\newblock Multi-domain causal structure learning in linear systems.
\newblock In \emph{Proceedings of the 32nd International Conference on Neural
  Information Processing Systems}, pp.\  6269--6279, 2018.

\bibitem[Goetzl et~al.(1984)Goetzl, Foster, and Payan]{goetzl1984basophil}
Goetzl, E., Foster, D., and Payan, D.
\newblock A basophil-activating factor from human t lymphocytes.
\newblock \emph{Immunology}, 53\penalty0 (2):\penalty0 227, 1984.

\bibitem[Gretton et~al.(2007)Gretton, Fukumizu, Teo, Song, Sch{\"o}lkopf, and
  Smola]{gretton2007kernel}
Gretton, A., Fukumizu, K., Teo, C.~H., Song, L., Sch{\"o}lkopf, B., and Smola,
  A.~J.
\newblock A kernel statistical test of independence.
\newblock In \emph{Proceedings of the 20th International Conference on Neural
  Information Processing Systems}, pp.\  585--592, 2007.

\bibitem[Hendrycks et~al.(2021)Hendrycks, Basart, Mu, Kadavath, Wang, Dorundo,
  Desai, Zhu, Parajuli, Guo, et~al.]{hendrycks2021many}
Hendrycks, D., Basart, S., Mu, N., Kadavath, S., Wang, F., Dorundo, E., Desai,
  R., Zhu, T., Parajuli, S., Guo, M., et~al.
\newblock The many faces of robustness: A critical analysis of
  out-of-distribution generalization.
\newblock In \emph{Proceedings of the IEEE/CVF International Conference on
  Computer Vision}, pp.\  8340--8349, 2021.

\bibitem[Hornik et~al.(1989)Hornik, Stinchcombe, and
  White]{hornik1989multilayer}
Hornik, K., Stinchcombe, M., and White, H.
\newblock Multilayer feedforward networks are universal approximators.
\newblock \emph{Neural networks}, 2\penalty0 (5):\penalty0 359--366, 1989.

\bibitem[Huang et~al.(2020)Huang, Zhang, Zhang, Ramsey, Sanchez-Romero,
  Glymour, and Sch{\"o}lkopf]{huang2020causal}
Huang, B., Zhang, K., Zhang, J., Ramsey, J., Sanchez-Romero, R., Glymour, C.,
  and Sch{\"o}lkopf, B.
\newblock Causal discovery from heterogeneous/nonstationary data.
\newblock \emph{Journal of Machine Learning Research}, 21\penalty0
  (89):\penalty0 1--53, 2020.

\bibitem[Lee et~al.(2021)Lee, Pyo, Ahn, Song, Park, and Lee]{lee2021clinical}
Lee, L.~E., Pyo, J.~Y., Ahn, S.~S., Song, J.~J., Park, Y.-B., and Lee, S.-W.
\newblock Clinical significance of large unstained cell count in estimating the
  current activity of antineutrophil cytoplasmic antibody-associated
  vasculitis.
\newblock \emph{International Journal of Clinical Practice}, 75\penalty0
  (10):\penalty0 e14512, 2021.

\bibitem[Liu et~al.(2021)Liu, Hu, Cui, Li, and Shen]{liu2021heterogeneous}
Liu, J., Hu, Z., Cui, P., Li, B., and Shen, Z.
\newblock Heterogeneous risk minimization.
\newblock In \emph{International Conference on Machine Learning}, pp.\
  6804--6814. PMLR, 2021.

\bibitem[Magliacane et~al.(2018)Magliacane, Van~Ommen, Claassen, Bongers,
  Versteeg, and Mooij]{magliacane2018domain}
Magliacane, S., Van~Ommen, T., Claassen, T., Bongers, S., Versteeg, P., and
  Mooij, J.~M.
\newblock Domain adaptation by using causal inference to predict invariant
  conditional distributions.
\newblock \emph{Advances in neural information processing systems}, 31, 2018.

\bibitem[Martinet et~al.(2022)Martinet, Strzalkowski, and
  Engelhardt]{martinet2022variance}
Martinet, G.~G., Strzalkowski, A., and Engelhardt, B.
\newblock Variance minimization in the wasserstein space for invariant causal
  prediction.
\newblock In \emph{International Conference on Artificial Intelligence and
  Statistics}, pp.\  8803--8851. PMLR, 2022.

\bibitem[Mayo(2016)]{mayo2016use}
Mayo, A.~M.
\newblock Use of the functional activities questionnaire in older adults with
  dementia.
\newblock \emph{Hartford Inst Geriatr Nurs}, 13:\penalty0 2, 2016.

\bibitem[Mitrovic et~al.(2021)Mitrovic, McWilliams, Walker, Buesing, and
  Blundell]{mitrovic2021representation}
Mitrovic, J., McWilliams, B., Walker, J., Buesing, L., and Blundell, C.
\newblock Representation learning via invariant causal mechanisms.
\newblock In \emph{International Conference on Learning Representations
  (ICLR)}, 2021.

\bibitem[Mu{\~n}oz-Fuentes et~al.(2018)Mu{\~n}oz-Fuentes, Cacheiro, Meehan,
  Aguilar-Pimentel, Brown, Flenniken, Flicek, Galli, Mashhadi, Hrab{\v{e}}~de
  Angelis, et~al.]{munoz2018international}
Mu{\~n}oz-Fuentes, V., Cacheiro, P., Meehan, T.~F., Aguilar-Pimentel, J.~A.,
  Brown, S.~D., Flenniken, A.~M., Flicek, P., Galli, A., Mashhadi, H.~H.,
  Hrab{\v{e}}~de Angelis, M., et~al.
\newblock The international mouse phenotyping consortium (impc): a functional
  catalogue of the mammalian genome that informs conservation.
\newblock \emph{Conservation Genetics}, 19\penalty0 (4):\penalty0 995--1005,
  2018.

\bibitem[Pearl(2009)]{pearl2009causality}
Pearl, J.
\newblock \emph{Causality}.
\newblock Cambridge University Press, 2009.

\bibitem[Perry et~al.(2022)Perry, von K{\"u}gelgen, and
  Sch{\"o}lkopf]{perry2022causal}
Perry, R., von K{\"u}gelgen, J., and Sch{\"o}lkopf, B.
\newblock Causal discovery in heterogeneous environments under the sparse
  mechanism shift hypothesis.
\newblock \emph{arXiv preprint arXiv:2206.02013}, 2022.

\bibitem[Peters et~al.(2016)Peters, B{\"u}hlmann, and
  Meinshausen]{peters2016causal}
Peters, J., B{\"u}hlmann, P., and Meinshausen, N.
\newblock Causal inference by using invariant prediction: identification and
  confidence intervals.
\newblock \emph{Journal of the Royal Statistical Society. Series B (Statistical
  Methodology)}, pp.\  947--1012, 2016.

\bibitem[Petersen et~al.(2010)Petersen, Aisen, Beckett, Donohue, Gamst, Harvey,
  Jack, Jagust, Shaw, Toga, et~al.]{petersen2010alzheimer}
Petersen, R.~C., Aisen, P., Beckett, L.~A., Donohue, M., Gamst, A., Harvey,
  D.~J., Jack, C., Jagust, W., Shaw, L., Toga, A., et~al.
\newblock Alzheimer's disease neuroimaging initiative (adni): clinical
  characterization.
\newblock \emph{Neurology}, 74\penalty0 (3):\penalty0 201--209, 2010.

\bibitem[Quinonero et~al.(2008)Quinonero, Sugiyama, Schwaighofer, and
  Lawrence]{quinonero2008dataset}
Quinonero, J., Sugiyama, M., Schwaighofer, A., and Lawrence, N.~D.
\newblock \emph{Dataset shift in machine learning}.
\newblock Mit Press, 2008.

\bibitem[Rojas-Carulla et~al.(2018)Rojas-Carulla, Sch{\"o}lkopf, Turner, and
  Peters]{rojas2018invariant}
Rojas-Carulla, M., Sch{\"o}lkopf, B., Turner, R., and Peters, J.
\newblock Invariant models for causal transfer learning.
\newblock \emph{The Journal of Machine Learning Research}, 19\penalty0
  (1):\penalty0 1309--1342, 2018.

\bibitem[Rothenh{\"a}usler et~al.(2021)Rothenh{\"a}usler, Meinshausen,
  B{\"u}hlmann, and Peters]{rothenhausler2021anchor}
Rothenh{\"a}usler, D., Meinshausen, N., B{\"u}hlmann, P., and Peters, J.
\newblock Anchor regression: Heterogeneous data meet causality.
\newblock \emph{Journal of the Royal Statistical Society: Series B (Statistical
  Methodology)}, 83\penalty0 (2):\penalty0 215--246, 2021.

\bibitem[Sagawa et~al.(2019)Sagawa, Koh, Hashimoto, and
  Liang]{sagawa2019distributionally}
Sagawa, S., Koh, P.~W., Hashimoto, T.~B., and Liang, P.
\newblock Distributionally robust neural networks for group shifts: On the
  importance of regularization for worst-case generalization.
\newblock \emph{arXiv preprint arXiv:1911.08731}, 2019.

\bibitem[Sch{\"o}lkopf et~al.(2012)Sch{\"o}lkopf, Janzing, Peters, Sgouritsa,
  Zhang, and Mooij]{scholkopf2012causal}
Sch{\"o}lkopf, B., Janzing, D., Peters, J., Sgouritsa, E., Zhang, K., and
  Mooij, J.
\newblock On causal and anticausal learning.
\newblock \emph{arXiv preprint arXiv:1206.6471}, 2012.

\bibitem[Sch{\"o}lkopf et~al.(2021)Sch{\"o}lkopf, Locatello, Bauer, Ke,
  Kalchbrenner, Goyal, and Bengio]{scholkopf2021toward}
Sch{\"o}lkopf, B., Locatello, F., Bauer, S., Ke, N.~R., Kalchbrenner, N.,
  Goyal, A., and Bengio, Y.
\newblock Toward causal representation learning.
\newblock \emph{Proceedings of the IEEE}, 109\penalty0 (5):\penalty0 612--634,
  2021.

\bibitem[Sinha et~al.(2018)Sinha, Namkoong, and Duchi]{sinha2018certifiable}
Sinha, A., Namkoong, H., and Duchi, J.
\newblock Certifiable distributional robustness with principled adversarial
  training.
\newblock In \emph{International Conference on Learning Representations}, 2018.
\newblock URL \url{https://openreview.net/forum?id=Hk6kPgZA-}.

\bibitem[Spirtes et~al.(2000)Spirtes, Glymour, Scheines, and
  Heckerman]{spirtes2000causation}
Spirtes, P., Glymour, C.~N., Scheines, R., and Heckerman, D.
\newblock \emph{Causation, Prediction, and Search}.
\newblock MIT press, 2000.

\bibitem[Subbaswamy et~al.(2019)Subbaswamy, Schulam, and
  Saria]{subbaswamy2019preventing}
Subbaswamy, A., Schulam, P., and Saria, S.
\newblock Preventing failures due to dataset shift: Learning predictive models
  that transport.
\newblock In \emph{The 22nd International Conference on Artificial Intelligence
  and Statistics}, pp.\  3118--3127. PMLR, 2019.

\bibitem[Sun et~al.(2021)Sun, Wu, Zheng, Liu, Chen, Qin, and
  Liu]{sun2021recovering}
Sun, X., Wu, B., Zheng, X., Liu, C., Chen, W., Qin, T., and Liu, T.-Y.
\newblock Recovering latent causal factor for generalization to distributional
  shifts.
\newblock \emph{Advances in Neural Information Processing Systems}, 34, 2021.

\bibitem[Tzourio-Mazoyer et~al.(2002)Tzourio-Mazoyer, Landeau, Papathanassiou,
  Crivello, Etard, Delcroix, Mazoyer, and Joliot]{tzourio2002automated}
Tzourio-Mazoyer, N., Landeau, B., Papathanassiou, D., Crivello, F., Etard, O.,
  Delcroix, N., Mazoyer, B., and Joliot, M.
\newblock Automated anatomical labeling of activations in spm using a
  macroscopic anatomical parcellation of the {MNI MRI} single-subject brain.
\newblock \emph{Neuroimage}, 15\penalty0 (1):\penalty0 273--289, 2002.

\bibitem[Wu et~al.(2022)Wu, Li, and Mao]{wu2022generalization}
Wu, Q., Li, J. Y.-M., and Mao, T.
\newblock On generalization and regularization via wasserstein distributionally
  robust optimization.
\newblock \emph{arXiv preprint arXiv:2212.05716}, 2022.

\bibitem[Young et~al.(2018)Young, Marinescu, Oxtoby, Bocchetta, Yong, Firth,
  Cash, Thomas, Dick, Cardoso, et~al.]{young2018uncovering}
Young, A.~L., Marinescu, R.~V., Oxtoby, N.~P., Bocchetta, M., Yong, K., Firth,
  N.~C., Cash, D.~M., Thomas, D.~L., Dick, K.~M., Cardoso, J., et~al.
\newblock Uncovering the heterogeneity and temporal complexity of
  neurodegenerative diseases with subtype and stage inference.
\newblock \emph{Nature communications}, 9\penalty0 (1):\penalty0 1--16, 2018.

\bibitem[Young(2022)]{neal2022maxleaf}
Young, N.
\newblock Number of vertices that a connected dominating set can reach in
  densely connected graphs.
\newblock Theoretical Computer Science Stack Exchange, 2022.
\newblock URL \url{https://cstheory.stackexchange.com/q/52100}.

\bibitem[Zhang(2008)]{zhang2008completeness}
Zhang, J.
\newblock On the completeness of orientation rules for causal discovery in the
  presence of latent confounders and selection bias.
\newblock \emph{Artificial Intelligence}, 172\penalty0 (16-17):\penalty0
  1873--1896, 2008.

\end{thebibliography}
\bibliographystyle{icml2023}

%%%%%%%%%%%%%%%%%%%%%%%%%%%%%%%%%%%%%%%%%%%%%%%%%%%%%%%%%%%%%%%%%%%%%%%%%%%%%%%
%%%%%%%%%%%%%%%%%%%%%%%%%%%%%%%%%%%%%%%%%%%%%%%%%%%%%%%%%%%%%%%%%%%%%%%%%%%%%%%
% APPENDIX
%%%%%%%%%%%%%%%%%%%%%%%%%%%%%%%%%%%%%%%%%%%%%%%%%%%%%%%%%%%%%%%%%%%%%%%%%%%%%%%
%%%%%%%%%%%%%%%%%%%%%%%%%%%%%%%%%%%%%%%%%%%%%%%%%%%%%%%%%%%%%%%%%%%%%%%%%%%%%%%
\newpage
\appendix
\onecolumn
\renewcommand{\contentsname}{Appendix}
\tableofcontents

\newpage

\addtocontents{toc}{\protect\setcounter{tocdepth}{2}}

\section{Causal minimax theories}
\label{sec: appendix causal minimax theories}

\subsection{Proof of Thm.~\ref{thm:graph degenerate}: Graphical condition for $S^*=S$}
\label{sec: proof of graphical thm}

\noindent\textbf{Thereom \ref{thm:graph degenerate}.} \emph{Suppose Asm.~\ref{eq:assum-dag} holds. Denote $\mathbf{X}_{M}^{0}\!:=\!\mathbf{X}_{M}\cap \mathbf{Ch}(Y)$ as mutable variables in $Y$'s children, and $\mathbf{W}:=\mathbf{De}(\mathbf{X}_{M}^{0}) \backslash \mathbf{X}_M^0$ as descendants of $\mathbf{X}_{M}^{0}$. Then, we have $S^*=S$ if $Y$ does not point to any vertex in $\mathbf{W}$.}
\begin{proof}
    Define $\mathbf{W}_2:=\mathbf{X} \backslash (\mathbf{X}_M^0  \cup \mathbf{De}(\mathbf{X}_M^0))$ as variables beyond $\mathbf{X}_M^0$ and their descendants, $\mathbf{X}_M^1:=\mathbf{X}_M \backslash \mathbf{X}_M^0$ as mutable variables beyond $Y$'s children. 
    
    We first show the equivalence of the following conditions; then show under either of them, we have $S^*=S$.
    \begin{enumerate}[label=$\textbf{\text{(\arabic*)}}$]
    \item $Y \ind_{G_{\overline{\mathbf{X}_M^0}}} \mathbf{W}|\mathbf{W}_2$;
    \item $Y$ does not point to any vertex in $\mathbf{W}$;
    \item $P(Y|\mathbf{X}_S, do(\boldsymbol{x}_M))$ can degenerate to the conditional distribution $P(Y|\mathbf{W}_2)$. 
     \end{enumerate}
     
    We introduce some notations that will be used in the proof. For a vertex $V_i$, denote $\mathbf{An}(V_i)$ as the set of its ancestors,  $G_{\overline{V_i}}$as the graph obtained by deleting all arrows pointing into $V_i$, $G_{\underline{V_i}}$ as the graph obtained by deleting all arrows emerging from $V_i$. To represent the deletion of both pointing (to $V_i$) and emerging (from $V_j$) arrows, we use the notation $G_{\overline{V_i}\underline{V_j}}$.

    In the following, we will show the equivalence of conditions \textbf{(1)}, \textbf{(2)}, and \textbf{(3)}. Firstly note that \textbf{(2)} is equivalent to ``$Y$ is not adjacent to $\mathbf{W}$'' due to the assumed acyclic of $G$. Also note that $\mathbf{X}_S \cup \mathbf{X}_M^1 = \mathbf{W} \cup \mathbf{W}_2$.
     
     \paragraph{(1)$\Rightarrow $(2)} 
     Prove by contradiction. Suppose $Y$ and $\mathbf{W}$ are adjacent, then they are also adjacent in ${G_{{\overline{\mathbf{X}_M^0}}}}$ because $\mathbf{W}\!\cap\!\mathbf{X}_M^0=\emptyset$. As a result, $Y$ and $\mathbf{W}$ can not be \emph{d}-separated by any vertex in ${G_{{\overline{\mathbf{X}_M^0}}}}$, which contradicts with \textbf{(1)}.

     \paragraph{(2)$\Rightarrow$(3)}
     Since $Y \not \in \mathbf{Pa}(\mathbf{X}_M^1)$, we have:
     \begin{align}
         p(y|\boldsymbol{x}_S,do(\boldsymbol{x}_M)) &= \nonumber \frac{p(y|\boldsymbol{pa}(y)) \prod_{i\in S} p(x_i|\boldsymbol{pa}(x_i))}{\int p(y|\boldsymbol{pa}(y)) \prod_{i\in S} p(x_i|\boldsymbol{pa}(x_i)) dy}\\
         &= \frac{p(y|\boldsymbol{pa}(y)) \prod_{i\in S} p(x_i|\boldsymbol{pa}(x_i)) \prod_{X_i\in \mathbf{X}_M^1} p^e(x_i|\boldsymbol{pa}(x_i))}{\int p(y|\boldsymbol{pa}(y)) \prod_{i\in S} p(x_i|\boldsymbol{pa}(x_i)) \prod_{X_i\in \mathbf{X}_M^1} p^e(x_i|\boldsymbol{pa}(x_i)) dy} \nonumber \\
         &= \frac{p(y,\boldsymbol{x}_S,\boldsymbol{x}_M^1|do(\boldsymbol{x}_M^0))}{\int p(y,\boldsymbol{x}_S,\boldsymbol{x}_M^1|do(\boldsymbol{x}_M^0)) dy } = p(y|\boldsymbol{x}_S,\boldsymbol{x}_M^1, do(\boldsymbol{x}_M^0)),
     \end{align}
    which indicates $P(Y|\mathbf{X}_S,do(\boldsymbol{x}_M))=P(Y|\mathbf{X}_S,\mathbf{X}^1_M,do(\boldsymbol{x}_M^0))=P(Y|\mathbf{W},\mathbf{W}_2,do(\boldsymbol{x}_M^0))$. 
    
    Unfold $P(Y|\mathbf{W},\mathbf{W}_2,do(\boldsymbol{x}_M^0))$ with the definition of interventional distribution, we have: 
    \begin{equation}
p(y|\boldsymbol{w},\boldsymbol{w}_2,do(\boldsymbol{x}_M^0))=\frac{p(y|\boldsymbol{pa}(y))\prod_{X_j\in \mathbf{W}}p^e(x_j|\boldsymbol{pa}(x_j))\prod_{X_i\in \mathbf{W}_2}p^e(x_i|\boldsymbol{pa}(x_i))}{\int p(y|\boldsymbol{pa}(y))\prod_{X_j\in \mathbf{W}}p^e(x_j|\boldsymbol{pa}(x_j))\prod_{X_i\in \mathbf{W}_2}p^e(x_i|\boldsymbol{pa}(x_i))dy}.
    \end{equation}
    
    Since $\mathbf{Pa}(Y) \cap \{\mathbf{X}_M^0, \mathbf{W}\} = \emptyset$ and $\forall X_i\in \mathbf{W}_2, \mathbf{Pa}(X_i) \cap \{\mathbf{X}_M^0, \mathbf{W}\} = \emptyset$, we further have:
    \begin{equation}
p(y|\boldsymbol{w},\boldsymbol{w}_2,do(\boldsymbol{x}_M^0)) =\frac{p^e(y,\boldsymbol{w}_2)\prod_{X_j\in \mathbf{W}}p^e(x_j|\boldsymbol{pa}(x_j))}{\int p^e(y,\boldsymbol{w}_2)\prod_{X_j\in \mathbf{W}}p^e(x_j|\boldsymbol{pa}(x_j)dy}.
    \end{equation}

    If $Y$ and $\mathbf{W}$ are not adjacent, then $\forall X_j \in \mathbf{W}, Y \notin \mathbf{Pa}(X_j)$. As a result, $p(y|\boldsymbol{w},\boldsymbol{w}_2,do(\boldsymbol{x}_M^0))=\frac{p(y,\boldsymbol{w}_2)}{\int p(y,\boldsymbol{w}_2)dy}=p(y|\boldsymbol{w}_2)$, which means $P(Y|\mathbf{X}_S,do(\boldsymbol{x}_M))=P(Y|\mathbf{W},\mathbf{W}_2,do(\boldsymbol{x}_M^0))=P(Y|\mathbf{W}_2)$ can degenerate to a conditional distribution.

    \paragraph{(3)$\Rightarrow$(1)} Prove by contradiction. We show that if $Y \not \ind_{G_{\overline{\mathbf{X}_M^{0}}}} \mathbf{W} | \mathbf{W}_{2}$, \emph{i.e.}, \textbf{(1)} does not hold, then $P(Y|\mathbf{X}_S, do(\boldsymbol{x}_M))$ can not degenerate to any conditional distribution, \emph{i.e.}, \textbf{(3)} does not hold. 
    
    Specifically, we first show $Y \!\not\ind_{G_{\overline{\mathbf{X}_M^{0}}}}\! \mathbf{W} | \mathbf{W}_{2}  \Rightarrow P(Y|\mathbf{X}_S, do(\boldsymbol{x}_M)) \neq P(Y|\mathbf{W}_2, do(\boldsymbol{x}_M^0))$. Then, we show $P(Y|\mathbf{X}_S, do(\boldsymbol{x}_M)) \neq P(Y|\mathbf{W}_2, do(\boldsymbol{x}_M^0))$ means $P(Y| \mathbf{X}_S, do(\boldsymbol{x}_M))$ can not degenerate to any conditional distribution.
    
    The first derivation is straight-forward. If $Y \!\not\ind_{G_{\overline{\mathbf{X}_M^{0}}}}\! \mathbf{W} | \mathbf{W}_{2}$, then under Asm.~\ref{asm:faithfulness}, we have $P(Y|\mathbf{W},\mathbf{W}_2,do(\boldsymbol{x}_M^0))\neq P(Y|\mathbf{W}_2,do(\boldsymbol{x}_M^0))$, which means $P(Y|\mathbf{X}_S, do(\boldsymbol{x}_M))=P(Y|\mathbf{W},\mathbf{W}_2,do(\boldsymbol{x}_M^0))\neq P(Y|\mathbf{W}_2,do(\boldsymbol{x}_M^0))$.

    Next, we will prove the second derivation. Suppose $P(Y|\mathbf{X}_S, do(\boldsymbol{x}_M))=P(Y|\mathbf{W}^{\prime}, \mathbf{W}_{2}, do(\boldsymbol{x}_M^{0}))$. We will show if $\mathbf{W}^{\prime} \neq \emptyset$, then the ``do'' can not be removed with either rule-2 (action to observation) or rule-3 (deletion of action) in the inference rules \cite{pearl2009causality}. According to Corollary $3.4.2$ in \cite{pearl2009causality}, the inference rules are complete in the sense that if the intervention probability (with ``do'') can be reduced to a probability expression (without ``do''), the "reduction" can be realized by a sequence of transformations, each conforming to one of the inference Rules 1-3. Since only rule-2 and rule-3 are related to the disappearance of ``do'', it is sufficient to show that rule-2 and rule-3 can not remove the ``do'', in order to prove \textbf{(1)}.
    
    Denote {$\mathbf{X}_M^{0}$ as $\{X_{M,i}^{0}\}_{i=1}^{r}$} and {$P(Y|\mathbf{W}^{\prime}, \mathbf{W}_{2}, do(\boldsymbol{x}_M^{0}))$} as { $P(Y|\mathbf{W}^{\prime}, \mathbf{W}_{2}, do(x_{M,1}^{0}), \ldots, do(x_{M,r}^{0}))$}. We first show rule-2 can not remove the ``do'' on any $X_{M,i}^{0} \in \mathbf{X}_M^{0}$.
    
    Recall rule-2 states that ``$P(Y | \mathbf{B}, do(\boldsymbol{x}), do(\boldsymbol{z}))=P(Y | \mathbf{B},\mathbf{Z}, do(\boldsymbol{x}))$ if $Y \ind_{{G}_{\overline{\mathbf{X}} \underline{\mathbf{Z}}}} \mathbf{Z} | \mathbf{B}, \mathbf{X}$ for any disjoint vertex sets $\mathbf{B}, \mathbf{X}$, and $\mathbf{Z}$ ". Prove by contradiction. Suppose rule-2 can remove the ``do'' on $X_{M,i}^{0} \in \mathbf{X}_M^{0}$, then:
    \begin{equation}
        Y \ind_{{G}_{\overline{\mathbf{X}_M^{0} \backslash\left\{X_{M,i}^{0}\right\}} \underline{\{X_{M,i}^{0}\}}}} X_{M,i}^{0} | \mathbf{W}^{\prime}, \mathbf{W}_{2}, \mathbf{X}_M^{0} \backslash\left \{X_{M,i}^{0}\right\},
    \label{rul2-not-possible}
    \end{equation}
    where we have $\mathbf{Z}=\left\{X_{M,i}^{0}\right\}, \mathbf{X}=\mathbf{X}_M^{0} \backslash \{X_{M,i}^{0}\}, \mathbf{B}=\mathbf{W}^{\prime} \cup \mathbf{W}_{2}$ in the notation of rule-2. 
	
    We explain why Eq. \ref{rul2-not-possible} can not be true. Note that $X_{M,i}^{0} \in \mathbf{Ch}(Y)$ and the direct edge $Y \rightarrow X_{M,i}^{0}$ is reserved in the graph ${G}_{\overline{\mathbf{X}_M^{0} \backslash\left\{X_{M,i}^{0}\right\}} \underline{\{X_{M,i}^{0}\}}}$, which means that $Y$ and $X_{M,i}^{0}$ can not be \emph{d}-separated by any vertex set. Hence, Eq. \ref{rul2-not-possible} can not be true.

    Then, we show rule-3 can not remove the ``do'' on all $X_{M,i}^{0} \in \mathbf{X}_M^{0}$. Recall rule-3 states that ``$P(Y | \mathbf{B}, do(\boldsymbol{x}), do(\boldsymbol{z}) )=P(Y | \mathbf{B}, do(\boldsymbol{x}) )$ if $Y \ind_{G_{\overline{\mathbf{X}},\overline{\mathbf{Z(B)}}}} \mathbf{Z} | \mathbf{B}, \mathbf{X}$ for any disjoint vertices sets $\mathbf{B}, \mathbf{X}$, and $\mathbf{Z}$ ". Here, $\mathbf{Z(B)}$ is the set of $Z$-nodes that are not ancestors of any $B$-node in $G_{\overline{\mathbf{X}}}$. Prove by contradiction. Suppose rule-3 can remove the ``do'' on $\mathbf{X}_M^{0}$, then:
    \begin{equation}
        Y \ind_{{G}_{ \overline{\mathbf{X}_M^{0}\left(\mathbf{W}^{\prime} \cup \mathbf{W}_{2}\right)}}}  \mathbf{X}_M^{0} | \mathbf{W}^{\prime} \cup \mathbf{W}_{2},
    \label{rule3-not-hold}
    \end{equation}
    where we have $\mathbf{X}=\emptyset, \mathbf{Z}=\mathbf{X}_M^{0}, \mathbf{B}=\mathbf{W}^{\prime} \cup \mathbf{W}_{2}, \mathbf{Z}(\mathbf{B})=\mathbf{X}_M^0(\mathbf{W}^\prime \cup \mathbf{W}_2)$ in the notation of rule-3.

    We explain why Eq. \ref{rule3-not-hold} can not be true when $\mathbf{W}^{\prime} \neq \emptyset$. By definition we have $\mathbf{W}^{\prime} \subseteq \mathbf{De}\left(\mathbf{X}_M^{0}\right)$, which means when $\mathbf{W}^{\prime} \neq \emptyset$,  $\mathbf{An}\left(\mathbf{W}^{\prime}\right) \cap \mathbf{X}_M^{0} \neq \emptyset$. Therefore, we have $\mathbf{X}_M^{0}\left(\mathbf{W}^{\prime} \cup \mathbf{W}_{2}\right):=\mathbf{X}_M^{0} \backslash\left(\mathbf{An}(\mathbf{W}^{\prime}) \cup \mathbf{W}_{2}\right) \neq \mathbf{X}_M^{0}$, which is equivalent to $\mathbf{X}_M^{0} \backslash (\mathbf{X}_M^{0}\left(\mathbf{W}^{\prime} \cup \mathbf{W}_{2}\right)) \neq \emptyset$. Suppose $X_{M,i}^{0} \in \mathbf{X}_M^{0} \backslash (\mathbf{X}_M^{0}\left(\mathbf{W}^{\prime} \cup \mathbf{W}_{2}\right))$, then the edge $Y \rightarrow X_{M,i}^{0}$ is reserved in the graph ${G}_{ \overline{\mathbf{X}_M^{0}\left(\mathbf{W}^{\prime} \cup \mathbf{W}_{2}\right)}}$, which means $Y$ and $X_{M,i}^{0}$ can not be \emph{d}-separated by any vertex set and Eq. \ref{rule3-not-hold} can not be true. 
    
    To conclude, we have proved that when $\mathbf{W}^{\prime} \neq \emptyset$, the ``do'' on $\mathbf{X}_M^{0}$ can not be removed entirely by rule-2 or rule-3.
	
   In the following, we prove under either of conditions \textbf{(1)}, \textbf{(2)}, \textbf{(3)}, we have $S^*=S$. When the interventional distribution can degenerate to a conditional distribution, \cite{rojas2018invariant} showed that $f_S(\boldsymbol{x}):=\mathbb{E}[Y|\boldsymbol{x}_S,do(\boldsymbol{x}_M)]$ satisfies the following minimax property:

    \begin{equation}
        f_S(\boldsymbol{x}) = \arg\min_{f\in \mathcal{F}^s} \max_{e \in \mathcal{E}} \mathbb{E}_{P^e}[(Y-f(\boldsymbol{x}))^2],
    \label{eq: appendix minmax S}
    \end{equation}
    which means $S^*=S$. Specifically, under the degeneration condition, they proved the optimality of $f_S$ by constructing a probability distribution $P^e$ for any predictor $f\in \mathcal{F}^s$, where $f$ has a larger or equal quadratic loss than $f_S$. For the details of the proof, please refer to Thm.~$4$ in \cite{rojas2018invariant}.
   
\end{proof}

\subsection{Details of Claim~\ref{prop:counter}: Counter-example of $S^*\neq S$}
\label{sec: counter example appendix new}

\begin{counterexample}
Consider the DAG in Fig.~\ref{fig:counter example}, which is the same as Fig.~\ref{fig: example graphical condition} (b). We set $Y,X_s,X_m$ to binary variables. We will show that there exists $P(Y), P(X_s|X_m,Y)$ such that $f_{S}:=\mathbb{E}[Y|x_s, do(x_m)]$ is not minimax optimal.

\begin{figure}[h!]
    \centering
    \includegraphics[width=0.25\linewidth]{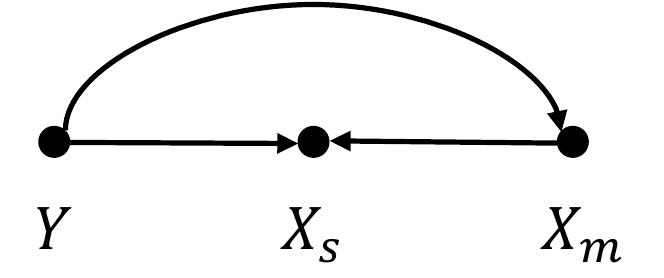}
    \caption{DAG of the counter example.}
    \label{fig:counter example}
\end{figure}

 We show this by proving the predictor $f_S$ has a larger quadratic loss than $f_{\emptyset}$: 
\begin{align}
\label{eq:goal}
\mathbb{E}[\left(Y - \mathbb{E}[Y|x_s,do(x_m)]\right)^2] > \mathbb{E}[\left( Y - \mathbb{E}[Y|do(x_m)]\right)^2]. 
\end{align}

Since we have: 
\begin{align*}
\mathbb{E}[\left( Y - \mathbb{E}[Y|x_s,do(x_m)]\right)^2] = \mathbb{E}[Y^2] + \mathbb{E}\left[ \mathbb{E}^2[Y|x_s, do(x_m)] \right] - 2\mathbb{E}[Y \cdot \mathbb{E}[Y|x_s, do(x_m)]],
\end{align*}
and $\mathbb{E}[\left( Y - \mathbb{E}(Y|do(x_m))\right)^2] = E[Y^2] - E[Y]^2$ due to that $P(Y|do(x_m)) = P(Y)$, Eq.~\eqref{eq:goal} is equivalent to:
\begin{align}
\label{eq:goal-simple}
\mathbb{E}\left[ \mathbb{E}^2[Y|x_s, do(x_m)] \right] > 2\mathbb{E}[Y \cdot \mathbb{E}[Y|x_s, do(x_m)]] - E^2[Y]. 
\end{align}

Besides, we have: 
\begin{align}
\mathbb{E}\left[ \mathbb{E}^2[Y|x_s, do(x_m)] \right] & =\sum_{x_s,x_m} \left[ \left[\sum_y p(x_s|x_m,y)p(x_m|y)p(y)\right] \cdot \mathbb{E}^2[Y|x_s, do(x_m)] \right], \label{eq:square} \\
\mathbb{E}\left[ Y\cdot\mathbb{E}[Y|x_s, do(x_m)] \right] & = \sum_{x_s,x_m} \left[ \left[\sum_y p(x_s|x_m,y)p(x_m|y)p(y)\cdot y\right] \cdot \mathbb{E}[Y|x_s, do(x_m)] \right]. \label{eq:cross}
\end{align}

Since we have $p(y|x_s, do(x_m)) = \frac{p(y)p(x_s|x_m,y)}{\sum_y p(y)p(x_s|x_m,y)}$, we have:
\begin{align}
\label{eq:y-s-m}
\mathbb{E}[Y|x_s, do(x_m)] = \frac{p(y=1)p(x_s|x_m,y=1)}{\sum_y p(y)p(x_s|x_m,y)}.
\end{align}

Substituting Eq.~\eqref{eq:y-s-m} into Eq.~\eqref{eq:square},~\eqref{eq:cross}, we have: 

{ \begin{align}
\mathbb{E}\left[ \mathbb{E}^2[Y|X_s, do(X_m)] \right] & \!=\!\sum_{x_s,x_m} \left[ \left[\sum_y p(x_s|x_m,y)p(x_m|y)p(y)\right] \cdot \left[ \frac{p(y=1)p(x_s|x_m,y=1)}{\sum_y p(y)p(x_s|x_m,y)} \right]^2 \right],  \\
\mathbb{E}\left[ Y\cdot\mathbb{E}[Y|X_s, do(X_m)] \right] & \!=\! \sum_{x_s,x_m} \left[ \left[\sum_y p(x_s|x_m,y)p(x_m|y)p(y)\cdot y \right] \cdot \left[ \frac{p(y=1)p(x_s|x_m,y=1)}{\sum_y p(y)p(x_s|x_m,y)} \right] \right] \nonumber \\
& \!=\! \sum_{x_s,x_m} \left[ \left[\sum_y p(x_s|x_m,y\!=\!1)p(x_m|y\!=\!1)p(y\!=\!1) \right] \cdot \left[ \frac{p(y\!=\!1)p(x_s|x_m,y\!=\!1)}{\sum_y p(y)p(x_s|x_m,y)} \right] \right]. 
\end{align}}

Denote $a_y \!:=\! p(y\!=\!1)$, $p(x_m\!=\!1|y)\!:=\!a_{my}$, $p(x_s\!=\!1|x_m,y) \!=\! a_{smy}$. Because $X_s,X_m$ are both binary variables, the summation over them traverses over four indicator functions $\mathbbm{1}(x_s=0,x_m=0)$, $\mathbbm{1}(x_s=0,x_m=1)$, $\mathbbm{1}(x_s=1,x_m=0)$, and $\mathbbm{1}(x_s=1,x_m=1)$, which means the left side of Eq.~\eqref{eq:goal-simple} is:

{ \begin{align}
& \mathbb{E}\left[ \mathbb{E}^2[Y|x_s, do(x_m)] \right] \!=\! \mathbbm{1}(x_s\!=\!1,x_m\!=\!1)\left(a_{s11}a_{m1}a_y+a_{s10}a_{m0}(1-a_y)\right)\left[ \frac{a_ya_{s11}}{a_ya_{s11} + (1-a_y)a_{s10}} \right]^2 \nonumber + \\
& \mathbbm{1}(x_s\!=\!1,x_m\!=\!0)\left[a_{s11}(1-a_{m1})a_y+a_{s10}(1-a_{m0})(1-a_y)\right]\left[ \frac{a_ya_{s01}}{a_ya_{s01} + (1-a_y)a_{s00}} \right]^2 \nonumber + \\
& \mathbbm{1}(x_s\!=\!0,x_m\!=\!1)\left[(1-a_{s11})a_{m1}a_y+(1-a_{s10})a_{m0}(1-a_y)\right]\left[ \frac{a_y(1-a_{s11})}{a_y(1-a_{s11}) + (1-a_y)(1-a_{s10})} \right]^2 + \nonumber \\
& \mathbbm{1}(x_s\!=\!0,x_m\!=\!0)\left[(1\!-\!a_{s01})(1\!-\!a_{m1})a_y+(1\!-\!a_{s00})(1\!-\!a_{m0})(1\!-\!a_y)\right]\left[ \frac{a_y(1\!-\!a_{s01})}{a_y(1\!-\!a_{s01}) +  (1-a_y)(1-a_{s00})} \right]^2 .
\end{align}}%

Similarly, the right side of Eq.~\eqref{eq:goal-simple} is:
{
\begin{align}
2\mathbb{E}\left[ Y\mathbb{E}[Y|x_s, do(x_m)] \right] - \mathbb{E}[Y^2] = &2 \big[ \mathbbm{1}(x_s=1,x_m=1)\frac{a_y^2a^2_{s11}a_{m1}}{a_ya_{s11}+(1-a_y)a_{s10}} +\nonumber\\
&\mathbbm{1}(x_s=1,x_m=0)\frac{a_y^2a^2_{s01}(1-a_{m1})}{a_ya_{s01}+(1-a_y)a_{s00}} + \nonumber\\
&\mathbbm{1}(x_s=0,x_m=1)\frac{a_y^2(1-a_{s11})^2a_{m1}}{a_y(1-a_{s11})+(1-a_y)a_{s10}} +\nonumber\\  &\mathbbm{1}(x_s=0,x_m=0)\frac{a_y^2(1-a_{s01})(1-a_{m1})}{a_y(1-a_{s01})+(1-a_y)(1-a_{s00})} \big] - a_y^2.
\end{align}}%

Let $a_{s10}=0.001, a_{s11}=0.999,  a_{s00}=a_{s01}=a_{s10}=0.5, a_{m0}-2a_{m1}=1, a_y=0.001$, Eq.~\ref{eq:goal-simple} becomes ``$994>-1$'', which means Eq.~\ref{eq:goal} holds and $S^*\neq S$.
\end{counterexample}

\subsection{Proof of Thm.~\ref{thm:min-max}: Worst-case risk identification}

\noindent\textbf{Theorem \ref{thm:min-max}.} \emph{Let $\mathcal{L}_{S^\prime}:=\max_{h \in  \mathcal{B}}\mathbb{E}_{P_h}[(Y \!-\! f_{S^\prime}(\boldsymbol{x}))^2]$ be the maximal population loss over $\{P_h\}_{h \in \mathcal{B}}$ for subset $S^\prime$. Then, we have $\mathcal{L}_{S^\prime}=\mathcal{R}_{S^\prime}$. Therefore, we have $S^*:= \mathop{\mathrm{argmin}}_{S^\prime \subseteq S} \mathcal{L}_{S^\prime}$.}

\begin{proof}
    Recall that $P_h\! :=\! P(Y,\mathbf{X}_S| do(\mathbf{X}_M\!=\!h(\boldsymbol{\boldsymbol{pa}}(\boldsymbol{x}_M)))$, where $h$ is a Borel measurable function from $\mathcal{P}a(\mathcal{X}_M)$ to $\mathcal{X}_M$. 
    
    To prove the theorem, we show that the worst-case risk $\mathcal{R}_{S^\prime}$ is attained when the causal factor $P^e(\mathbf{X}_M|\mathbf{Pa}(\mathbf{X}_M))$ degenerates to a delta function $\mathbbm{1}(\mathbf{X}_M=h^*(\boldsymbol{pa}(\boldsymbol{x}_M)))$, for some Borel function $h^*:\mathcal{P}a(\mathcal{X}_M) \to \mathcal{X}_M$.

    First, consider the case where $\mathbf{X}_M=\{X_m\}$. The $\mathcal{R}_{S^\prime}$ expands into:

    \begin{equation}
    \mathcal{R}_{S^\prime} = \max_{e \in \mathcal{E}} \int_{y} \int_{\boldsymbol{x}} [y-f_{S^\prime}(\boldsymbol{x})]^2 p(y|\boldsymbol{\boldsymbol{pa}}(y)) p^e(x_m|\boldsymbol{\boldsymbol{pa}}(x_m)) \prod_{i \in S} p(x_i|\boldsymbol{pa}(x_i)) dy d\boldsymbol{x}.
    \label{eq: risk}
    \end{equation}

    Let $\tilde{\mathbf{X}}:=\mathbf{X}\backslash(X_m \cup \mathbf{Pa}(X_m))$ be variables beyond $X_m$ and its parents. Split the integral in Eq.~\ref{eq: risk} into three parts: the integral over $x_m$, the integral over $\boldsymbol{pa}(x_m)$, and the integral over $y,\tilde{\boldsymbol{x}}$. Denote the last part as:
    
    \begin{equation}
        l(x_m,\boldsymbol{pa}(x_m)):=\int_y \int_{\tilde{\boldsymbol{x}}} [y-f_{S^\prime}(\boldsymbol{x})]^2 p(y|\boldsymbol{pa}(y)) \prod_{X_i \in \tilde{\mathbf{X}}} p(x_i|\boldsymbol{pa}(x_i)) dy d\tilde{\boldsymbol{x}}.
    \end{equation}

    Then, Eq.~\ref{eq: risk} becomes:
    \begin{equation}
        \mathcal{R}_{S^\prime} = \max_{e \in \mathcal{E}} \int_{\boldsymbol{pa}(x_m)} \int_{x_m} l(x_m,\boldsymbol{pa}(x_m)) p^e(x_m|\boldsymbol{pa}(x_m)) dx_m \prod_{X_i \in \mathbf{Pa}(\mathbf{X}_m)} p(x_i|\boldsymbol{pa}(x_i)) dpa(x_m)
    \label{eq: only vary with e}
    \end{equation}

    Since in Eq.~\ref{eq: only vary with e}, the only item that varies with $e$ is $p^e(x_m|\boldsymbol{pa}(x_m))$, we can move the $\max_{e \in \mathcal{E}}$ into the inner integral and have:
    \begin{equation}
        \mathcal{R}_{S^\prime}= \int_{\boldsymbol{pa}(x_m)} \max_{e \in \mathcal{E}} \int_{x_m} l(x_m,\boldsymbol{pa}(x_m)) p^e(x_m|\boldsymbol{pa}(x_m)) dx_m \prod_{X_i \in \mathbf{Pa}(\mathbf{X}_m)} p(x_i|\boldsymbol{pa}(x_i)) dpa(x_m).
    \end{equation}

    Let $h^*(\boldsymbol{pa}(x_m)):=\arg\max_{x_m} l(x_m,\boldsymbol{pa}(x_m))$ be a function from $\mathcal{P}a(\mathcal{X}_M)$ to $\mathcal{X}_M$, we have:
    \begin{equation}
        \mathcal{R}_{S^\prime} = \int_{\boldsymbol{pa}(x_m)} l(h^*(\boldsymbol{pa}(x_m)),\boldsymbol{pa}(x_m)) \prod_{X_i \in \mathbf{Pa}(\mathbf{X}_m)} p(x_i|\boldsymbol{pa}(x_i)) dpa(x_m),
    \end{equation}
    which means the worst-case risk is attained when the causal factor $P(X_m|\mathbf{Pa}(X_m))$ degenerates to a delta function $\mathbbm{1}(X_m=h^*(\boldsymbol{pa}(x_m)))$. In addition, under Asm.~\ref{eq:assum-dag}, $l(x_m,\boldsymbol{pa}(x_m))$ is a continues function. By the Maximum Theorem \cite{berge1963topologies}, $h^*:=\arg\max_{x_m} l(x_m,\boldsymbol{pa}(x_m))$ is upper semi-continuous and thus a Borel function.

    When $\mathbf{X}_M$ contains multiple mutable variables, we can consider the maximization according to the topology order $\{X_{M,1}, X_{M,2}, ..., X_{M,d_M}\}$, where $X_{M,i}$ is a mutable variable that is not the ancestor of any other variable in $\{X_{M,j}|j < i\}$. That is, we consider the $\max_{e \in \mathcal{E}} \int_{x_{M,i}} l(x_{M,i},\boldsymbol{pa}(x_{M,i})) p^e(x_{M,i}|\boldsymbol{pa}(x_{M,i})) dx_{M,i}$ sequentially for $i=1,2,...,d_M$.

    Such a sequential maximization is plausible because the topology order of mutable variables is identifiable. Please refer to the discovery of $\mathbf{De}(X_i)$ for $X_i \in \mathbf{X}_M$ in Appx.~\ref{sec: appendix discovery basic structures} for details.
\end{proof}
\newpage
\section{Causal discovery and structural identifiability}
\label{sec: appendix discovery}

Minimax theories in Sec.~\ref{sec.identify} rely on the identifiability of specific causal structures, such as $\mathbf{X}_M, \mathbf{W}$. In this section, we will prove the structural identifiability by offering causal discovery algorithms to recover them, with data from $\mathcal{E}_{\text{tr}}$. Specifically, we first show the discovery of several \emph{basic} causal structures, then use them to prove Prop.~\ref{prop.thm3.2 justifiability} and Prop.~\ref{prop.thm3.3 justifiability}.

\subsection{Basic causal structures}
\label{sec: appendix discovery basic structures}

In this section, we show the discovery of several basic causal structures: $\mathbf{X}_M$, $\mathbf{X}_M^0$, $\mathbf{X}_M^0 \cup \mathbf{De}(\mathbf{X}_M^0)$, $\mathbf{De}(X_i)$ for $X_i \in \mathbf{X}_M$, and $\mathbf{Pa}(X_i)$ for $X_i \in \mathbf{X}_M \cup \mathbf{De}(\mathbf{X}_M)$. Our algorithms are inspired by \citep{huang2020causal}.

We first introduce some notations. We use the subscript $X_i, X_j \in \mathbf{X}$, $V_i, V_j \in \mathbf{V}$ to denote vertices; the superscript $\mathbf{V}^i, \mathbf{V}^j \subseteq \mathbf{V}$ to denote vertex sets. Denote $E_{\text{tr}}$ as the environmental indicator variable with support $\mathcal{E}_{\text{tr}}$. Let $G_{\text{aug}}$ be the augmented graph \cite{huang2020causal} over $\mathbf{V} \cup E_{\text{tr}}$. We consider the causal DAG $G$ as the induced subgraph of $G_{\text{aug}}$ over $\mathbf{V}$. We have the following notations in $G_{\text{aug}}$. For two vertex sets $\mathbf{V}^i, \mathbf{V}^j \subseteq \mathbf{V}$, let $\mathbf{Z}^{i,j} \subseteq \mathbf{V}\backslash \{\mathbf{V}^i,\mathbf{V}^j\}$ be the separating set such that $\mathbf{V}^i \ind \mathbf{V}^j | \mathbf{Z}^{i,j}$. Denote $\mathbf{D}^{e,i}$ as the set of vertices along the directed path $E_{\text{tr}} \to \cdots \to V_i$. Let $\widehat{\Delta}_{i \to j}$ be the estimated Hilbert Schmidt Independence Criterion (HSIC) \cite{gretton2007kernel} for $V_i \to V_j$.

\noindent\textbf{Discovery of $\mathbf{X}_M$ and the causal skeleton.} These structures can be identified via Alg.~\ref{alg.1}. Specifically, under Asm.~\ref{assum:e-train}, any mutable variable in $\mathcal{E}$ is a mutable variable in $\mathcal{E}_{\text{tr}}$. Following \cite{huang2020causal}, we assume that if $X_i$ is a mutable variable, then $X_i$ and $E_{\text{tr}}$ are not independent given any other subset of $\mathbf{V}\backslash \{X_i\}$. Under the above assumptions and Asm.~\ref{asm:faithfulness}, we have $X_i \in \mathbf{X}_M$ iff $E\to X_i$ in $G_{\text{aug}}$.

\begin{algorithm}[h!]
\caption{Recovery of $\mathbf{X}_M$ and the causal skeleton.} 
\label{alg.1}
\begin{enumerate}
    \item Start with $\mathbf{X}_M\!\leftarrow\!\emptyset$. For each $i$, test if $V_i\ind E_{\text{tr}}$ or if there exist a separating set $\mathbf{Z}_{i, e}$. If $V_i\not\ind E_{\text{tr}}$ and there is no such seperating set, update $\mathbf{X}_M\!\leftarrow\!\mathbf{X}_M\cup V_i$.
    \item  Start with an undirected graph $G_0$ including edges between any two vertices in $\mathbf{V}$ and the arrow $E_{\text{tr}}\rightarrow V_i$ for $V_i\in \mathbf{X}_M$.
    For each pair $i,j$, if {$V_i\ind V_j$} or there exists a separating set {$\mathbf{Z}_{i, j}$}, remove the edge {$V_i-V_j$} from $G_0$.
\end{enumerate}
\end{algorithm}

\noindent\textbf{Discovery of $\mathbf{X}_M^0$.} We can use $E_{\text{tr}}$ and the \emph{v}-structure $E_{\text{tr}} \to X_i \leftarrow Y$ to detect $\mathbf{X}_M^0:=\mathbf{X}_M \cap \mathbf{Ch}(Y)$. Specifically, for $X_i\in \mathbf{X}_M$ that is adjacent to $Y$, test whether $Y\not \ind E_{\text{tr}} | \mathbf{Z}^{y,e}\cup X_i$. If the $\not \ind$ holds, orient $Y \to X_i$ and add $X_i$ to $\mathbf{X}_M^0$.

\noindent\textbf{Discovery of $\mathbf{X}_M^0 \cup \mathbf{De}(\mathbf{X}_M^0)$.} This structure can be identified via Alg.~\ref{alg.2}. Alg.~\ref{alg.2} searches vertices adjacent to $\mathbf{X}_M^0$ in a breadth-first manner. The set $\mathbf{A}$ defined in line-\ref{exp.1} is the final output. The set $\mathbf{B}$ is an instrumental set that starts with $\mathbf{X}_M^0$ and ends with $\emptyset$. During the search, $\mathbf{B}$ stores the vertices in $\mathbf{X}_M^0$ that has not been searched for the children. Once a vertex $X_i \in \mathbf{B}$ has been searched, it is excluded from the set $\mathbf{B}$ (line-\ref{exp.12}) and the children of $X_i$ are added to $\mathbf{B}$ if it has not been visited (line-\ref{exp.include1} and line-\ref{exp.include2}). 

Specifically, in lines \ref{exp.notinM1} to \ref{exp.notinM2}, we consider the vertex $X_j\in\mathbf{Neig}(X_i)$ such that $X_j \not\in \mathbf{X}_M$. Since $X_j \not\in \mathbf{X}_M$, $E_{\text{tr}}$ and $X_j$ are not adjacent. Since $X_i\in \mathbf{X}_M^0 \cup \mathbf{De}(\mathbf{X}_M^0)$, we have $E_{\text{tr}} \rightarrow \cdots \rightarrow X_i - X_j$. Together, these mean we can use the \emph{v}-structure $E_{\text{tr}} \to  \cdots \to X_i \leftarrow X_j$ to decide whether $X_j \in \mathbf{X}_M^0 \cup \mathbf{De}(\mathbf{X}_M^0)$. In lines \ref{exp.inM1} to \ref{exp.inM2}, we consider the vertex $X_j\in\mathbf{Neig}(X_i)$ such that $X_j \in \mathbf{X}_M$. We first explain why it is unnecessary to consider the case of $X_j\in\mathbf{X}_M$ and $X_j\in\mathbf{Neig}(Y)$. If $X_i\in\mathbf{Pa}(Y)$, $X_j$ can not be in $\mathbf{X}_M^0 \cup \mathbf{De}(\mathbf{X}_M^0)$ because otherwise it would induce a directed cycle. If $X_j\in\mathbf{Ch}(Y)$, we have $X_j \in \mathbf{X}_M^0$ and has been included in set $\mathbf{A}$ in the beginning. As a result, the remaining case is $X_j\in\mathbf{X}_M$ and $X_j\not\in\mathbf{Neig}(Y)$. In this case, we have $X_i \in \mathbf{De}(Y)$, which means we can use the \emph{v}-structure $Y \to  \cdots \to X_i \leftarrow X_j$ to decide whether $X_j \in \mathbf{X}_M^0 \cup \mathbf{De}(\mathbf{X}_M^0)$.

\noindent\textbf{Discovery of $\mathbf{De}(X_i)$ for $X_i \in \mathbf{X}_M$.} This structure can be identified via Alg.~\ref{alg.3}. Alg.~\ref{alg.3} first searches vertices adjacent to $\mathbf{X}_M$ and orients $X_i-X_j$ for $X_i \in \mathbf{X}_M$ in order to detect $X_i$'s children. It then searches $X_i$'s children in a similar manner to identify $X_i$'s descendants. 

Specifically, in lines \ref{alg3.notinM1} to \ref{alg3.notinM2}, we consider the vertex $X_j\in\mathbf{Neig}(X_i)$ such that $X_j \not\in \mathbf{X}_M$. The orientation of $X_i-X_j$ can be decided by the \emph{v}-structure $E_{\text{tr}}\to \cdots \to X_i \leftarrow X_j$ since $E_{\text{tr}}$ is not adjacent to $X_j$. In lines \ref{alg3.inM1} to \ref{alg3.inM2}, we consider the vertex $X_j\in\mathbf{Neig}(X_i)$ such that $X_j \in \mathbf{X}_M$. Following \cite{huang2020causal}, we decide the orientation of $X_i-X_j$ by comparing the estimated HSIC values $\widehat{\Delta}_{i\rightarrow j}$ and $ \widehat{\Delta}_{j\rightarrow i}$.

\noindent\textbf{Discovery of $\mathbf{Pa}(X_i)$ for $X_i \in \mathbf{X}_M \cup \mathbf{De}(\mathbf{X}_M)$.} This structure has been identified in lines \ref{alg3.pa1} and \ref{alg3.pa2} of Alg.~\ref{alg.3}.

\begin{algorithm}[h!]
\caption{Recovery of $\mathbf{X}_M^0\cup\mathbf{De}(\mathbf{X}_M^0)$}
\label{alg.2}
    \begin{algorithmic}[1]
        \STATE Start with $\mathbf{A},\mathbf{B}\leftarrow\mathbf{X}_M^0$ and $\mathrm{visited}(X_i)\leftarrow \mathrm{false}$. \alglinelabel{exp.1}
        \WHILE{$\mathbf{B}\neq \emptyset$}
        \FOR{$X_i\in \mathbf{B}$}
        \FOR{$X_j\in \mathbf{Neig}(X_i)$ } 
        \IF{$X_j\not\in\mathbf{X}_M$ and $X_j\ind E_{\text{tr}}|(\mathbf{Z}^{e,j}\cup X_i)\backslash \mathbf{D}^{j, e}$} \alglinelabel{exp.notinM1}
        \STATE $\mathbf{A}\!\leftarrow\!\mathbf{A}\cup X_j$. \alglinelabel{exp.6}
        \IF{$\mathrm{visited}(X_j)=\mathrm{false}$}
        \STATE $\mathbf{B}\leftarrow\mathbf{B}\cup X_j$.\alglinelabel{exp.include1}
        \ENDIF
        \ENDIF \alglinelabel{exp.notinM2}
        \IF{$X_j\in\mathbf{X}_M$ and $X_j\not\in\mathbf{Neig}(Y)$ and $X_j\ind Y|(\mathbf{Z}^{j,y}\cup X_i)\backslash \mathbf{D}^{y,i}$ }\alglinelabel{exp.inM1}
        \STATE   $\mathbf{A}\leftarrow\mathbf{A}\cup X_j$.
        \IF{$\mathrm{visited}(X_j)=\mathrm{false}$}
        \STATE $\mathbf{B}\leftarrow\mathbf{B}\cup X_i$.\alglinelabel{exp.include2}
        \ENDIF
        \ENDIF
        \ENDFOR
        \STATE $\mathbf{B}\leftarrow\mathbf{B}\setminus\{X_i\}$. \alglinelabel{exp.12}
        \ENDFOR\alglinelabel{exp.inM2}
        \ENDWHILE
    \end{algorithmic}
\end{algorithm}

\begin{algorithm}[h!]
\caption{Recovery of $\mathbf{De}(X_i)$ for $X_i \in \mathbf{X}_M$.}
\label{alg.3}
\begin{algorithmic}[1]
    \STATE Start with $\mathbf{B}\leftarrow\mathbf{X}_M$ and $\mathrm{visited}(X_i)\leftarrow \mathrm{false
    }$.
    \WHILE{$\mathbf{B}\neq \emptyset$}
    \FOR{$X_i\in \mathbf{B}$}
    \FOR{$X_j\in \mathbf{Neig}(X_i)$ } 
    \IF{$X_j\not\in\mathbf{X}_M$ and $X_j\ind E_{\text{tr}}|(\mathbf{Z}^{e, j}\cup X_i)\backslash \mathbf{D}^{j, e}$} \alglinelabel{alg3.notinM1}
    \STATE orient $X_i-X_j$ as $X_i \to X_j$.
    \IF{$\mathrm{visited}(X_j)=\mathrm{false}$}
    \STATE $\mathbf{B}\leftarrow\mathbf{B}\cup X_j$.
    \ENDIF
    \ELSE 
    \STATE  orient $X_i-X_j$ as $X_i \leftarrow X_j$. \alglinelabel{alg3.pa1}
    \ENDIF \alglinelabel{alg3.notinM2}
    \IF{$X_j\in\mathbf{X}_M$ and $\widehat{\Delta}_{i\rightarrow j}< \widehat{\Delta}_{j\rightarrow i}$}\alglinelabel{alg3.inM1}
    \STATE   orient $X_i-X_j$ as $X_i \to X_j$.
    \IF{$\mathrm{visited}(X_j)=\mathrm{false}$}
    \STATE $\mathbf{B}\leftarrow\mathbf{B}\cup X_j$.
    \ENDIF
    \ELSE
    \STATE orient $X_i-X_j$ as $X_i \leftarrow X_j$. \alglinelabel{alg3.pa2}
    \ENDIF
    \ENDFOR
    \STATE $\mathbf{B}\leftarrow \mathbf{B}\backslash X_i$. 
    \ENDFOR \alglinelabel{alg3.inM2}
    \ENDWHILE
\end{algorithmic}
\end{algorithm}

\clearpage
\subsection{Proof of Prop.~\ref{prop.thm3.2 justifiability}: Testability of Thm.~\ref{thm:graph degenerate}}

\noindent\textbf{Proposition~\ref{prop.thm3.2 justifiability}.} \emph{Under Asm.~\ref{eq:assum-dag}-\ref{assum:e-train}, we have that \textbf{i)} the $\mathbf{W}$ is identifiable; and \textbf{ii)} the condition $Y \not\to \mathbf{W}$ is testable from $\{\mathcal{D}_e\}_{e \in \mathcal{E}_{\mathrm{tr}}}$.}

\begin{proof}
    $\mathbf{W}\!=\!(\mathbf{X} \backslash \mathbf{X}_M^0) \cap \mathbf{De}(\mathbf{X}_M^0) \!=\!(\mathbf{X} \backslash \mathbf{X}_M^0) \cap \big\{\mathbf{X}_M^0 \cup \mathbf{De}(\mathbf{X}_M^0)\big\}$ is identifiable because $\mathbf{X}_M^0$ and $\mathbf{X}_M^0 \cup \mathbf{De}(\mathbf{X}_M^0)$ are identifiable, as shown in Appx.~\ref{sec: appendix discovery basic structures}. Since all vertices in $\mathbf{W}$ are descendants of $Y$, we have $Y \not\to X_i, X_i \in \mathbf{W}$ iff $X_i$ is not adjacent to $Y$ in the causal skeleton of $G_{\text{aug}}$.
\end{proof}

\subsection{Proof of Prop.~\ref{prop.thm3.3 justifiability}: Identifiability of Thm.~\ref{thm:min-max}}

\noindent\textbf{Proposition~\ref{prop.thm3.3 justifiability}.} \emph{Under Asm.~\ref{eq:assum-dag}-\ref{assum:e-train}, the  $P_h$, $f_{S^\prime}$, and hence $\mathcal{L}_{S^\prime}(h)$ are identifiable.} 

\begin{proof}
    To identify $P_h$, we need to use $h(\mathbf{Pa}(\mathbf{X}_M))$ to replace $\mathbf{X}_M$, followed by regenerating $X_i$ from $\mathbf{Pa}_{{G_{\overline{\mathbf{X}_M}}}}(X_i)$ for $X_i \in \mathbf{De}_{G_{\overline{\mathbf{X}_M}}}(\mathbf{X}_M)$. 
    Here, $\mathbf{Pa}_{G_{\overline{\mathbf{X}_M}}}(X_i)$ denotes the parents of $X_i$ in the graph $G_{\overline{\mathbf{X}_M}}$. 

    To identify $f_{S^\prime}$, we need to sample from $P(Y,\mathbf{X}_{S^\prime}|do(\boldsymbol{x}_M))$, which involves intervening $\mathbf{X}_M$ and regenerating $X_i$ from $\mathbf{Pa}_{G_{\overline{\mathbf{X}_M}}}(X_i)$ for $X_i \in \mathbf{De}_{G_{\overline{\mathbf{X}_M}}}(\mathbf{X}_M)$.

    These structures, \emph{i.e.}, $\mathbf{X}_M$, $\mathbf{De}(X_i)$ for $X_i \in \mathbf{X}_M$, and $\mathbf{Pa}(X_i)$ for $X_i \in \mathbf{X}_M \cup \mathbf{De}(\mathbf{X}_M)$ are readily identified in Appx.~\ref{sec: appendix discovery basic structures}.
\end{proof}

\newpage
\section{Empirical estimation methods}

\subsection{Estimation of $f_{S^\prime}$}

We adopt soft-intervention to replace $P^e(\mathbf{X}_M|\mathbf{Pa}(\mathbf{X}_M))$ with {$P(\mathbf{X}_M)$} and hence define:
\begin{equation}
    P^\prime(\mathbf{X},Y)=P(Y|\mathbf{Pa}(Y))P(\mathbf{X}_M)\prod_{i\in S}P(X_i|\mathbf{Pa}(X_i)),
\end{equation}
which converts the estimation of $f_{S^\prime}$ to a regression problem, \emph{i.e.}, $f_{S^\prime}(\boldsymbol{x})= \mathbb{E}_{P^\prime}[Y|\boldsymbol{x}_{S^\prime}, \boldsymbol{x}_M]$. To generate data distributed as $P^{\prime}$, we first randomly permute $\mathbf{X}_M$ in a sample-wise manner to generate data from $P(\mathbf{X}_M)$. We then regenerate data for $X_i \in \mathrm{De}_{G_{\overline{\mathbf{X}_M}}}(\mathbf{X}_M)$ from $\mathrm{Pa}_{G_{\overline{\mathbf{X}_M}}}(X_i)$ via estimating the structural equation. 

Indeed, we only need to regenerate $\mathbf{De}_{G_{\overline{\mathbf{X}_M}}}(\mathbf{X}_M) \cap \mathbf{Blanket}_{G_{\overline{\mathbf{X}_M}}}(Y)$ since $P^\prime(Y|\mathbf{X})=P^\prime(Y|\mathbf{Blanket}_{G_{\overline{\mathbf{X}_M}}}(Y))$. Here, $\mathbf{Blanket}_{G_{\overline{\mathbf{X}_M}}}(Y))$ is the Markovian blanket of $Y$ in the graph $G_{\overline{\mathbf{X}_M}}$. Following this intuition, we consider intervening on another variable set $X_{do}^*:=\mathbf{X}_M^0 \cup \{\mathbf{De}(\mathbf{X}_M^0) \backslash \mathbf{Ch}(Y)\}$ and regenerate $X_i \in \mathbf{De}_{G_{\overline{\mathbf{X}_{do}^*}}}(\mathbf{X}_{do}^*)$. We show $\mathbf{De}_{G_{\overline{\mathbf{X}_{do}^*}}}(\mathbf{X}_{do}^*)$ is the minimum regeneration set in Prop.~\ref{prop.minimal}.

\begin{proposition}
For any admissible set $\mathbf{X}_{do}$, we have $\mathbf{De}_{G_{\overline{\mathbf{X}^*_{do}}}}(\mathbf{X}^*_{do}) \subseteq \big\{\mathbf{De}_{G_{\overline{\mathbf{X}_{do}}}} (\mathbf{X}_{do})\cap\mathbf{Blanket}_{G_{\overline{\mathbf{X}_{do}}}}(Y) \big\}$, which means $\mathbf{De}_{G_{\overline{\mathbf{X}_{do}^*}}}(\mathbf{X}_{do}^*)$ is the minimum regeneration set.
\label{prop.minimal}
\end{proposition}

\begin{proof}
    We first prove a set $\mathbf{X}_{do}$ is admissible, \emph{i.e.}, $P(Y|\mathbf{X}\backslash\mathbf{X}_{do}, do(\boldsymbol{x}_{do}))=P(Y|\mathbf{X}_S,do(\boldsymbol{x}_M))$ if and only if $\mathbf{X}_M^0 \subseteq \mathbf{X}_{do}$ and $\{\mathbf{X}_S \cap \mathbf{Ch}(Y)\} \cap \mathbf{X}_{do} = \emptyset$. Note that:
    \begin{align}
        p(y|\boldsymbol{x}\backslash \boldsymbol{x}_{do}, do(\boldsymbol{x}_{do}))&=\frac{p(y|\boldsymbol{pa}(y))\prod_{X_i\in \{\mathbf{X}\setminus\mathbf{X}_{do}\}}p(x_i|\boldsymbol{pa}(x_i))}{\int p(y|\boldsymbol{pa}(y))\prod_{X_i\in \{\mathbf{X}\setminus\mathbf{X}_{do}\}}p(x_i|\boldsymbol{pa}(x_i))dy} \nonumber \\ 
         &=\frac{p(y|\boldsymbol{pa}(y))\prod_{X_i\in \{\mathbf{X}\setminus\mathbf{X}_{do}\}\cap\mathbf{Ch}(Y)}p(x_i|\boldsymbol{pa}(x_i))}{\int p(y|\boldsymbol{pa}(y))\prod_{X_i\in \{\mathbf{X}\setminus\mathbf{X}_{do}\}\cap\mathbf{Ch}(Y)}p(x_i|\boldsymbol{pa}(x_i))dy},
    \label{eq: unford 1}
    \end{align}
    and 
    \begin{align}
        p(y|\boldsymbol{x}_S,do(\boldsymbol{x}_M)) &= \nonumber \frac{p(y|\boldsymbol{pa}(y)) \prod_{i\in S} p(x_i|\boldsymbol{pa}(x_i))}{\int p(y|\boldsymbol{pa}(y)) \prod_{i\in S} p(x_i|\boldsymbol{pa}(x_i)) dy} \nonumber \\ 
        &= \frac{p(y|\boldsymbol{pa}(y)) \prod_{X_i\in \mathbf{X}_S\cap\mathbf{Ch}(Y)} p(x_i|\boldsymbol{pa}(x_i))}{\int p(y|\boldsymbol{pa}(y)) \prod_{X_i\in \mathbf{X}_S\cap\mathbf{Ch}(Y)} p(x_i|\boldsymbol{pa}(x_i)) dy}.
    \label{eq: unford 2}
    \end{align}

Together, Eq.~\ref{eq: unford 1} and Eq.~\ref{eq: unford 2} indicate $P(Y|\mathbf{X}\backslash\mathbf{X}_{do}, do(\boldsymbol{x}_{do}))=P(Y|\mathbf{X}_S,do(\boldsymbol{x}_M))$ if and only if $\{\mathbf{X} \setminus\mathbf{X}_{do} \}\cap\mathbf{Ch}(Y)=\mathbf{X}_{S}\cap \mathbf{Ch}(Y)$, which can be re-written as:
\begin{equation}
    \{\mathbf{X}_M^0 \cap  \mathbf{X}_{do}^{\mathrm{c}}\}
	\cup \{\mathbf{X}_S \cap \mathbf{Ch}(Y) \cap  \mathbf{X}_{do}^{\mathrm{c}}\}
	=\mathbf{X}_{S} \cap \mathbf{Ch}(Y),
 \label{eq: hold iff}
\end{equation}
where $\mathbf{X}_{do}^c$ is the complementary set of $\mathbf{X}_{do}$. Eq.~\ref{eq: hold iff} holds if and only if $\mathbf{X}_M^0 \subseteq \mathbf{X}_{do}$ and $\{\mathbf{X}_S \cap \mathbf{Ch}(Y)\} \cap \mathbf{X}_{do} = \emptyset$.

We then prove $X_{do}^*$ is an admissible set and $\mathbf{De}_{G_{\overline{\mathbf{X}^*_{do}}}}(\mathbf{X}^{*}_{do})=\mathbf{De}(\mathbf{X}_M^0)\cap (\mathbf{X}_S\cap \mathbf{Ch}(Y))$. $X_{do}^*$ is admissible as the conditions $\mathbf{X}_M^0 \subseteq \mathbf{X}_{do}^*$ and $\{\mathbf{X}_S\!\cap\! \mathbf{Ch}(Y)\}\cap \mathbf{X}_{do}^*=\emptyset$ hold by definition. We show 
$\mathbf{De}_{G_{\overline{\mathbf{X}^*_{do}}}}(\mathbf{X}^{*}_{do})=\mathbf{De}(\mathbf{X}_M^0)\cap (\mathbf{X}_S\cap \mathbf{Ch}(Y))$ by showing \textbf{i)} $\mathbf{De}_{G_{\overline{\mathbf{X}^*_{do}}}}(\mathbf{X}^{*}_{do}) \subseteq \mathbf{De}(\mathbf{X}_M^0)\cap (\mathbf{X}_S\cap \mathbf{Ch}(Y))$ and \textbf{ii)} $\mathbf{De}_{G_{\overline{\mathbf{X}^*_{do}}}}(\mathbf{X}^{*}_{do}) \supseteq \mathbf{De}(\mathbf{X}_M^0)\cap (\mathbf{X}_S\cap \mathbf{Ch}(Y))$.

\textbf{i)} $\mathbf{De}_{G_{\overline{\mathbf{X}^*_{do}}}}(\mathbf{X}^{*}_{do}) \subseteq \mathbf{De}(\mathbf{X}_M^0)\cap (\mathbf{X}_S\cap \mathbf{Ch}(Y))$. Note that $\mathbf{X}_{do}^*\subseteq \mathbf{X}_M^0 \cup \mathbf{De}(\mathbf{X}_M)^0$, which means $\mathbf{De}(\mathbf{X}_{do}^*) \subseteq \mathbf{De}(\mathbf{X}_M^0)$. Then, we have:
\begin{align}
\mathbf{De}_{G_{\overline{\mathbf{X}_{do}^*}}}(\mathbf{X}_{do}^*)
&=\mathbf{De}(\mathbf{X}_{do}^*)\cap (\mathbf{X}_{do}^*)^c
=
\mathbf{De}(\mathbf{X}_{do}^*)\cap (\mathbf{X}_M^0)^c \cap \{\mathbf{De}(\mathbf{X}_M^0) \setminus \mathbf{Ch}(Y)\}^c \nonumber \\
&=
\mathbf{De}(\mathbf{X}_{do}^*)\cap \{\mathbf{X}_M^{c} \cup \mathbf{Ch}(Y)\}\} \cap \{\mathbf{De}(\mathbf{X}_M^0)^c \cup \mathbf{Ch}(Y)\} \nonumber \\
&\subseteq 
\mathbf{De}(\mathbf{X}_M^0) \cap \{\mathbf{X}_M^c \cup \mathbf{Ch}(Y)^c\} \cap \{\mathbf{De}(\mathbf{X}_M^0)^c \cup \mathbf{Ch}(Y)\} \nonumber \\
& =
\mathbf{De}(\mathbf{X}_M^0)\cap\mathbf{X}_M^c\cap \mathbf{Ch}(Y)
= \mathbf{De}(\mathbf{X}_M^0) \cap \mathbf{X}_S \cap \mathbf{Ch}(Y) \nonumber \\
&\subseteq  \mathbf{De}(\mathbf{X}_M^0)\cap (\mathbf{X}_S\cap \mathbf{Ch}(Y)).
\end{align}

\textbf{ii)} $\mathbf{De}_{G_{\overline{\mathbf{X}^*_{do}}}}(\mathbf{X}^{*}_{do}) \supseteq \mathbf{De}(\mathbf{X}_M^0)\cap (\mathbf{X}_S\cap \mathbf{Ch}(Y))$. Since $\mathbf{X}_M^0 \subset \mathbf{X}_{do}^*$, $\mathbf{De}(\mathbf{X}_M^0) \subseteq \mathbf{De}(\mathbf{X}_{do}^*)$. As a result, we have $\mathbf{De}(\mathbf{X}_M^0)\cap (\mathbf{X}_S\cap \mathbf{Ch}(Y)) \subseteq  \mathbf{De}(\mathbf{X}_M^0) \subseteq \mathbf{De}(\mathbf{X}_{do}^*)$ and hence $\{\mathbf{De}(\mathbf{X}_M^0)\cap (\mathbf{X}_S\cap \mathbf{Ch}(Y)) \setminus \mathbf{X}_{do}^* \} \subseteq \{\mathbf{De}(\mathbf{X}_{do}^*) \setminus \mathbf{X}_{do}^*\}$. Besides, note that $\mathbf{X}_{do}^* \cap \mathbf{De}(\mathbf{X}_M^0)\cap (\mathbf{X}_S\cap \mathbf{Ch}(Y)) =\emptyset$, which indicates $\mathbf{De}(\mathbf{X}_M^0)\cap (\mathbf{X}_S\cap \mathbf{Ch}(Y)) \setminus \mathbf{X}_{do}^*=\mathbf{X}_{do}^*$ and $\mathbf{De}_{G_{\overline{\mathbf{X}}_{do}^*}}(\mathbf{X}_{do}^*)=\mathbf{De}(\mathbf{X}_{do}^*) \setminus \{\mathbf{De}(\mathbf{X}_M^0)\cap (\mathbf{X}_S\cap \mathbf{Ch}(Y))\}$. As a result, we have $\mathbf{De}(\mathbf{X}_M^0)\cap (\mathbf{X}_S\cap \mathbf{Ch}(Y)) \subseteq \mathbf{De}_{G_{\overline{\mathbf{X}}_{do}^*}}(\mathbf{X}_{do}^*)$.

Given that any $\mathbf{X}_{do}$ needs to satisfy the two conditions, we have:
\begin{align}
\mathbf{X}_M^0 \subseteq \mathbf{X}_{do}
&\Rightarrow
\mathbf{De}(\mathbf{X}_M^0) \subseteq \mathbf{De}(\mathbf{X}_{do}), \nonumber \\
\mathbf{X}_{do}\subseteq \{\mathbf{X}_S \cap \mathbf{Ch}(Y)\}^{c}
&\Rightarrow
\{\mathbf{X}_S \cap \mathbf{Ch}(Y)\}\subseteq \mathbf{X}_{do}^{c}.
\end{align}
Therefore, we have:
\begin{equation}
    \mathbf{De}(\mathbf{X}_M^0)\cap\{\mathbf{X}_S \cap \mathbf{Ch}(Y)\}
\subseteq 
\mathbf{De}(\mathbf{X}_{do})\cap \mathbf{X}_{do}^{c},
\end{equation}
which means $\mathbf{De}_{G_{\overline{\mathbf{X}_{do}^*}}}(\mathbf{X}_{do}^*) = \mathbf{De}(\mathbf{X}_M^0)\cap (\mathbf{X}_S\cap \mathbf{Ch}(Y)) \subseteq \mathbf{De}_{G_{\overline{\mathbf{X}}_{do}}} (\mathbf{X}_{do})$ for any admissible set $\mathbf{X}_{do}$.
\end{proof}

\begin{remark}
    The $\mathbf{X}_{do}^*$, $\mathbf{De}_{G_{\overline{\mathbf{X}_{do}^*}}}(\mathbf{X}_{do}^*)$, and $\mathbf{Pa}(X_i)$ for $X_i$ in $\mathbf{De}_{G_{\overline{\mathbf{X}_{do}^*}}}(\mathbf{X}_{do}^*)$ are identifiable according to Appx.~\ref{sec: appendix discovery basic structures}.
\end{remark}

\subsection{Estimation of $\mathcal{L}_{S^\prime}$}

We first sample from $P_h$. Specifically, we replace $X_i$ with $h(\mathbf{Pa}(X_i))$ for $X_i \in \mathbf{X}_M$ and regenerate data for $X_i \in \mathbf{De}_{G_{\overline{\mathbf{X}_M}}}$ from $\mathbf{Pa}_{G_{\overline{\mathbf{X}_M}}}(X_i)$ via estimating the structural equation. We then maximize $\mathbb{E}_{P_h}[(Y-f_{S^\prime}(\boldsymbol{x}))^2]$ over $h$ to obtain $\mathcal{L}_{S^\prime}$.

\newpage
\section{Equivalence relation and the recovery algorithm}

We first introduce some notations that will be used in this section. We use the subscript $X_i, X_j \in \mathbf{X}$, $V_i, V_j \in \mathbf{V}$ to denote variables and vertices; the superscript $S^\prime, S^i, S^j \subseteq S$, $\mathbf{V}^\prime, \mathbf{V}^i, \mathbf{V}^j \subseteq \mathbf{V}$ to denote variable and vertex subsets. A path $p:=<V_1,V_2,...,V_l>$ is a sequence of distinct vertices with $V_i$ being adjacent to $V_{i+1}$ for $i=1,2,...,l-1$. We use $l$ to denote the length of the path. The path $p$ can be \emph{blocked} by a vertex set $\mathbf{V}^\prime$ means it can be \emph{d}-separated by $\mathbf{V}^\prime$ when $G$ is a DAG, and \emph{m}-separated by $\mathbf{V}^\prime$ when $G$ is a Maximal Ancestral Graph (MAG). For a vertex $V_i$, denote $\mathrm{deg}(V_i):=|\mathbf{Neig}(V_i)|$ as its degree. In a MAG, we use $\mathbf{C}$, $\mathbf{L}$ to denote the selection set and the latent set, respectively.

\subsection{Details of Def.~\ref{def:equi}: Equivalence relation}

We first introduce the following lemma, which studies the property of \emph{d}-separation and \emph{m}-separation in the difference set.

\begin{lemma}
\label{lemma: union d/m-separation}
Consider two vertex sets $\mathbf{V}^1, \mathbf{V}^2$, and a path $p$. If $p$ can be blocked by $\mathbf{V}^1 \cup \mathbf{V}^2$ but can not be blocked by the difference set $(\mathbf{V}^1 \cup \mathbf{V}^2)\backslash \mathbf{V}^2 = \mathbf{V}^1$, then the set $\mathbf{V}^2$ contains a non-collider on $p$.
\end{lemma}

\begin{proof}
    We first show $p$ contains at least one non-collier. Prove by contradiction. Suppose all vertices on $p$ are colliders. Since $p$ can not be blocked by $\mathbf{V}^1$, we have $\forall\, V_i \in p$, $V_i \in \mathbf{V}^1$ or $\exists\, V_j\in \mathbf{De}(V_i)$ such that $V_j \in \mathbf{V}^1$. This means $p$ can not be blocked $\mathbf{V}^1 \cup \mathbf{V}^2$, which is a contradiction.

    We then prove the lemma by considering two cases: \textbf{i)} $p$ contains only non-colliders; \textbf{ii)} $p$ contains both colliders and non-colliders. For \textbf{i)}, since $p$ can not be blocked by $\mathbf{V}^1$, all vertices on $p$ are not in $\mathbf{V}^1$. Since $p$ can be blocked by $\mathbf{V}^1 \cup \mathbf{V}^2$, at least a vertex on $p$ is in $\mathbf{V}^2$, thus proving the lemma. For \textbf{ii)}, since $p$ can not be blocked by $\mathbf{V}^1$, $\forall\, V_i \in p$, we have: if $V_i$ is a non-collider on $p$, $V_i \not \in \mathbf{V}^1$; otherwise $V_i$ is a collider on $p$, $V_i \in \mathbf{V}^1$ or $\exists\, V_j\in \mathbf{De}(V_i)$ such that $V_j \in \mathbf{V}^1$, thus in the set $\mathbf{V}^1 \cup \mathbf{V}^2$. Therefore, the set $\mathbf{V}^2$ must contain a non-collider on $p$, otherwise, $p$ will not be blocked by $\mathbf{V}^1 \cup \mathbf{V}^2$.
\end{proof}

\noindent\textbf{Definition \ref{def:equi}.} Consider a general causal graph $G$ over an output $Y$ and covariates $\mathbf{X}$. Let $\sim_G$ be an equivalence relation on all subsets of $\{1,...,\mathrm{dim}(X)\}$. We say $S^i \sim_G S^j$ if $\exists \, S^{ij} \subseteq S^i\cap S^{j}$ such that:
\begin{align}
    \label{eq:sim-G-def}
    Y \ind_{G}  \mathbf{X}_{(S^{ij})^c} | \mathbf{X}_{S^{ij}}, \text{ where } (S^{ij})^c:=(S^i \cup S^j) \backslash S^{ij}.
\end{align}

\begin{proof}
    It is obvious that the $\sim_G$ is reflective ($S^i \sim_G S^i$) and symmetric ($S^i \sim_G S^j \Rightarrow S^j \sim_G S^i$). In the following, we will show it is also transitive, \emph{i.e.}, $S^i \sim_G S^j, S^j \sim_G S^k \Rightarrow S^i \sim_G S^k$. We show this by constructing an intersection set $S^{ik} \subseteq S^i \cap S^k$ such that $Y \ind_G \mathbf{X}_{(S^{ik})^c} | \mathbf{X}_{S^{ik}}$.

    Since $S^i \sim_G S^j$, we have $\exists\, S^{ij}$ s.t. $Y \ind_G \mathbf{X}_{(S^{ij})^c} | \mathbf{X}_{S^{ij}}$. Similarly for $S^j \sim_G S^k$, we have $\exists\, S^{jk}$ s.t. $Y \ind_G \mathbf{X}_{(S^{jk})^c}| \mathbf{X}_{S^{jk}}$. In the following, we will construct the intersection set $S^{ik}$ from $S^{ij} \cap S^{jk}$.

    \begin{figure}[h]
    \centering
    \includegraphics[width=.4\textwidth]{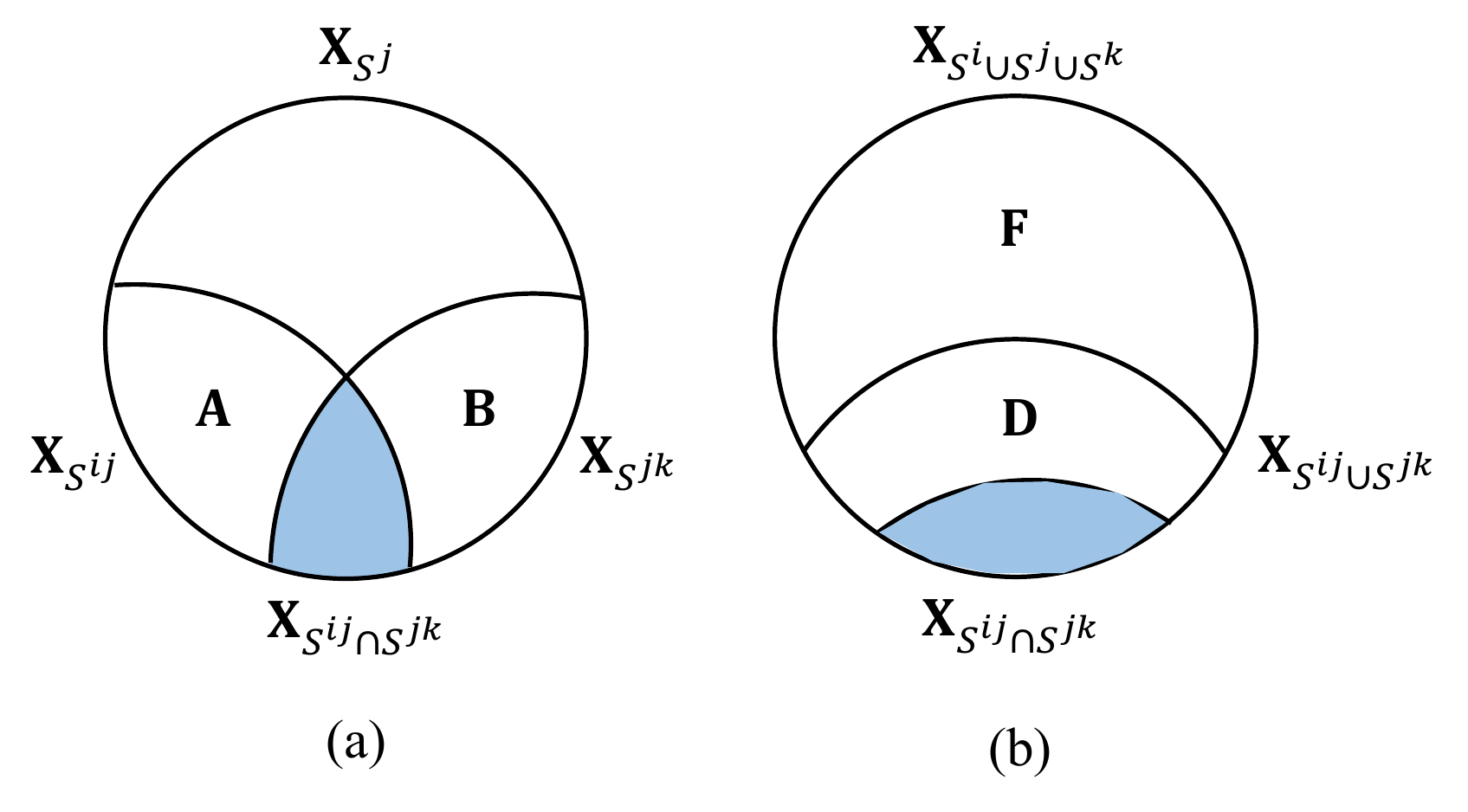}
    \caption{Illustration of union and intersection of $\mathbf{X}_{S_{ij}}$ and $\mathbf{X}_{S_{jk}}$.}
    \label{fig: subsets}
    \end{figure}
    
    Denote $\mathbf{A}:=\mathbf{X}_{S^{ij}\backslash(S^{ij}\cap S^{jk})}$ and $\mathbf{B}:=\mathbf{X}_{S^{jk}\backslash(S^{ij}\cap S^{jk})}$, as shown by Fig.~\ref{fig: subsets} (a). We first show $Y \ind_{G} \mathbf{A} | \mathbf{X}_{S^{ij}\cap S^{jk}}$ and $Y \ind_{G} \mathbf{B} | \mathbf{X}_{S^{ij}\cap S^{jk}}$. We show this by proving that any path between $Y$ and $\mathbf{A}$ (similarly $\mathbf{B}$) can be blocked by $\mathbf{X}_{S^{ij}\cap S^{jk}}$. Prove by contradiction. Suppose there is a path $p_0:=<Y,X_1,...,X_{l_0}>$ between $Y$ and $X_{l_0}\in \mathbf{A}$ such that $p_0$ can not be blocked by $\mathbf{X}_{S^{ij} \cap S^{jk}}$. We have $p_0$ can be blocked by the set $\mathbf{X}_{S^{jk}}$. This is because $X_{l_0} \in \mathbf{A} \subseteq \mathbf{X}_{S^j \backslash S^{jk}} \subseteq \mathbf{X}_{(S^{jk})^c}$ and $Y\ind_G \mathbf{X}_{(S^{jk})^c} | \mathbf{X}_{S^{jk}}$. Therefore, by Lemma \ref{lemma: union d/m-separation}, the set $\mathbf{B}=\mathbf{X}_{S^{jk}} \backslash \mathbf{X}_{S^{ij} \cap S^{jk}}$ contains a non-collider denoted as $X_{l_1}$ on $p_0$. Hence, we have a subpath of $p_0$, \emph{i.e.}, $p_1:=<Y,X_1,...,X_{l_1}>$ between $Y$ and $X_{l_1}\in \mathbf{B}$ such that $p_1$ can not be blocked by $\mathbf{X}_{S^{ij}\cap S^{jk}}$. Here, we have $p_1$ can be blocked by the set $\mathbf{X}_{S^{ij}}$. This is because $X_{l_1} \in \mathbf{B} \subseteq \mathbf{X}_{S^j \backslash S^{ij}} \subseteq \mathbf{X}_{(S^{ij})^c}$ and $Y \ind_G \mathbf{X}_{(S^{ij})^c} | \mathbf{X}_{S^{ij}}$. Therefore, by Lemma \ref{lemma: union d/m-separation}, the set $\mathbf{A}=\mathbf{X}_{S^{ij}} \backslash \mathbf{X}_{S^{ij}\cap S^{jk}}$ contains a non-collider denoted as $X_{l_2}$ on $p_1$. Repeating like this, we have either $X_1 \in \mathbf{A}\subseteq \mathbf{X}_{(S^{jk})^c}$ or $X_1 \in \mathbf{B}\subseteq \mathbf{X}_{(S^{ij})^c}$. Since $X_1$ is adjacent to $Y$, this contradicts with $Y \ind_G \mathbf{X}_{(S^{jk})^c}|\mathbf{X}_{S^{jk}}$ or $Y \ind_G \mathbf{X}_{(S^{ij})^c} | \mathbf{X}_{S^{ij}}$.

    Further, denote $\mathbf{D}:=\mathbf{X}_{(S^{ij}\cup S^{jk})\backslash(S^{ij}\cap S^{jk})}$ and $\mathbf{F}:=\mathbf{X}_{(S^i \cup S^j \cup S^k)\backslash (S^{ij} \cap S^{jk})}$, as shown in Fig.~\ref{fig: subsets} (b). We have shown $Y \ind_G \mathbf{D} | \mathbf{X}_{S^{ij}\cap S^{jk}}$ by combining the statements $Y \ind_{G} \mathbf{A} | \mathbf{X}_{S^{ij}\cap S^{jk}}$ and $Y \ind_{G} \mathbf{B} | \mathbf{X}_{S^{ij}\cap S^{jk}}$. Next, we will show $Y \ind_G \mathbf{F} | \mathbf{X}_{S^{ij}\cap S^{jk}}$. This means we can construct the intersection set $S^{ik}:=(S^{ij}\cap S^{jk})\subseteq (S^i \cap S^k)$ such that $Y\ind_G \mathbf{X}_{(S^{ik})^c} | \mathbf{X}_{S^{ik}}$, and hence proving $S^i \sim_G S^k$ by definition.

    We show this by proving that any path between $Y$ and $\mathbf{F}$ can be blocked by $\mathbf{X}_{S^{ij}\cap S^{jk}}$. Prove by contradiction. Suppose there is a path $p_0:=<Y,X_1,...,X_{l_0}>$ between $Y$ and $X_{l_0} \in \mathbf{F}$ such that $p_0$ can not be blocked by $\mathbf{X}_{S^{ij}\cap S^{jk}}$. We have $p_0$ can be blocked by $\mathbf{X}_{S^{ij}}$ or $\mathbf{X}_{S^{jk}}$. This is because we have either $X_{l_0} \in \mathbf{F} \subseteq \mathbf{X}_{(S^{ij})^c}$ or $X_{l_0} \in \mathbf{F} \subseteq \mathbf{X}_{(S^{jk})^c}$ and $Y \ind_G \mathbf{X}_{(S^{ij})^c} | \mathbf{X}_{S^{ij}}$, $Y \ind_G \mathbf{X}_{(S^{jk})^c} | \mathbf{X}_{S^{jk}}$. Without loss of generality, we consider $X_{l_0} \in \mathbf{X}_{(S^{ij})^c}$ and $p_0$ can be blocked by  $\mathbf{X}_{S^{ij}}$. By Lemma~\ref{lemma: union d/m-separation}, the set $\mathbf{X}_{S^{ij}\backslash(S^{ij}\cap S^{jk})}$ contains a non-collider denoted as $X_{l_1}$ on $p_1$. Hence, we have a subpath of $p_0$, \emph{i.e.}, $p_1:=<Y,X_1,...,X_{l_1}>$ between $Y$ and $X_{l_1}\in \mathbf{X}_{S^{ij}\backslash(S^{ij}\cap S^{jk})}$ such that $p_1$ can not be blocked by $\mathbf{X}_{S^{ij}\cap S^{jk}}$. This contradicts with the statement $Y\ind_G \mathbf{D} | \mathbf{X}_{S^{ij}\cap S^{jk}}$,  because $\mathbf{X}_{S^{ij}\backslash (S^{ij}\cap S^{jk})} \subseteq \mathbf{X}_{(S^{ij}\cup S^{jk})\backslash (S^{ij}\cap S^{jk})}=\mathbf{D}$.
    
    To conclude, we have proved $\sim_G$ is reflective, symmetric, and transitive. Hence, $\sim_G$ is a legitimate equivalence relation.
\end{proof}

\subsection{Proof of Prop.~\ref{prop:recover}: Correctness of Alg.~\ref{alg: recover g-equivalence}}

\noindent\textbf{Proposition~\ref{prop:recover}.} \emph{For each input graph that is Markov equivalent to the ground-truth graph $G$, Alg.~\ref{alg: recover g-equivalence} can correctly recover $\mathbf{Pow}(S) / \!\sim_G$.}

\begin{proof}
    We first show, under Asm~.\ref{eq:assum-dag}, \ref{asm:faithfulness}, all Markovian equivalent graphs have the same equivalence classes. Specifically, Markovian equivalent graphs have the same \emph{d}-separation and \emph{m}-separation \cite{pearl2009causality,zhang2008completeness}. Because the equivalence relation is defined on \emph{d}-separation and \emph{m}-separation, they also have the same equivalence classes.
    
    We then introduce some notions that will be used in the proof. We use the unbolded letter, \emph{e.g.}, $S^i, T^i$, to denote variable sets, and the \textbf{bolded} letter, \emph{e.g.}, $\mathbf{Pow}(S), \mathbf{R}^i$, to denote sets whose elements are variable sets. Recall that the equivalence class of subset $S^i$ is denoted as $\text{\textbf{[}}S^i\text{\textbf{]}}:=\{S^j|S^j \sim_G S^i\}$. We say a vertex $X_i$ is $Y$'s $l$-neighbour if the shortest path between $Y$ and $X_i$ has length $l$. As a special case, say $X_i$ as the $0$-neighbour of $Y$ if there is no path between $Y$ and $X_i$. Define $l_G=0$ if $\mathrm{Neig}(Y)=\emptyset$, and $l_G=1,2,...,l$ if $Y$ has $1,2,...,l$-neighbours, respectively. 
    
    In the following, we will prove the correctness of Alg.~\ref{alg: recover g-equivalence} by induction on $l_G$.    

    \textbf{Base.} {\small $l_G=0 \Rightarrow \mathrm{Neig}(Y)=\emptyset$. Hence, any two subsets $S^i, S^j \subseteq S$ are equivalent and $\mathbf{Pow}(S)/\!\sim_G = \{\text{\textbf{[}}S\text{\textbf{]}}\} = \text{\textbf{recover}}(G)$.}

    \textbf{Induction hypothesis.} Suppose any graph $G$ with $l_G \leq l$ has $\mathbf{Pow}(S)/\!\sim_G = \text{\textbf{recover}}(G)$.

    \textbf{Step.} Consider $G$ with $l_G=l+1$. 
    
    Denote all the $2^{\mathrm{deg}(Y)}$ subsets of $\mathrm{Neig}(Y)$ as $T^1,T^2,...,T^{2^{\mathrm{deg}(Y)}}$. We can partition the $\mathbf{Pow}(S)$ into $2^{\mathrm{deg}(Y)}$ sets $\mathbf{R}^1, \mathbf{R}^2,..., \mathbf{R}^{2^{\mathrm{deg}(Y)}}$, with $\mathbf{Pow}(S)=\cup_{i=1}^{2^{\mathrm{deg}(Y)}} \mathbf{R}^i$, $\mathbf{R}^i\cap \mathbf{R}^j=\emptyset$ for $i\neq j$, and $\mathbf{R}^i:=\{S^i|S^i \subseteq S, S^i\cap \mathrm{Neig}(Y)=T^i\}$. Now, consider a subset $S^i\in \mathbf{R}^i$ and another subset $S^j\in \mathbf{R}^j$, we have $S^i \not \sim_G S^j$, because $S^i \cap \mathrm{Neig}(Y) \neq S^j \cap \mathrm{Neig}(Y)$. Therefore, the equivalence classes in $\mathbf{Pow}(S)$ is the union of the equivalence classes in $\mathbf{R}^1,\mathbf{R}^2,...,\mathbf{R}^{2^{\mathrm{deg}(Y)}}$. Formally:
    \begin{equation}
        \mathbf{Pow}(S)/\!\sim_{G}=\cup_{i=1}^{2^{\mathrm{deg}(Y)}} \mathbf{R}_i/\!\sim_{G}.
    \label{eq: recover 1}
    \end{equation}
    A distinct virtue of the MAG constructed in Alg.~\ref{alg: recover g-equivalence} in line-7 is that it can represent \emph{d}-separation and \emph{m}-separation when selection and latent variables exist. Specifically, given any causal graph $G$ over $\mathbf{V}=\mathbf{O}\cup \mathbf{L} \cup \mathbf{C}$, the MAG $M_G$ over $\mathbf{O}$, with $\mathbf{C}$ as the selection set and $\mathbf{L}$ as the latent set, satisfies that for any disjoint subsets $\mathbf{A},\mathbf{B},\mathbf{Z} \subseteq \mathbf{O}$, $\mathbf{A} \ind_{M_G} \mathbf{B} | \mathbf{Z}$ if and only if $\mathbf{A} \ind_{G} \mathbf{B} | \mathbf{Z} \cup \mathbf{C}$ \cite{zhang2008completeness}. Therefore, for the $M_G$ over $S\backslash \mathrm{Neig}(Y)$, with $S^\prime$ as the selection set, $\mathrm{Neig}(Y) \backslash S^\prime$ as the latent set, constructed in line-7, we have $S^i, S^j \subseteq S\backslash \mathrm{Neig}(Y)$ are equivalent in $M_G$ if and only if $S^i\cup S^\prime$, $S^j \cup S^\prime$ are equivalent in $G$. 
    
    Formally, denote $\mathbf{R}^{i\prime}$ as the set attained via removing $T^i$\footnote{Note that the $T^i$ here equals the $S^\prime$ in the $i$-th loop, in line-6 of Alg.~\ref{alg: recover g-equivalence}.} from each element of $\mathbf{R}^i$, denote $M_G^i$ as the MAG constructed in line-7 with $T^i$ as the selection set, and $\mathbf{P}^i$ as the set attained via adding $T^i$ to each subset in each equivalence class in $\mathbf{R}^{i\prime}/\!\sim_{M_G^i}$, we have:
    \begin{equation}
    \mathbf{R}_i /\! \sim_{G}= \mathbf{P}_i.
    \label{eq: recover 2}
    \end{equation}
    Then, by Eq.~\ref{eq: recover 1} and Eq.~\ref{eq: recover 2}, we have:
    \begin{equation}
    \mathbf{Pow}(S)/\!\sim_{G} = \cup_{i=1}^{2^{\mathrm{deg}(Y)}} \mathbf{P}_i.
    \end{equation}
    Since $l_{M_{G}^i} \leq l$, by the induction hypothesis, we have $\mathbf{P}_i = \text{\textbf{recover}}(M_{G}^i)$. According to lines 9 and 10 of the Alg.~\ref{alg: recover g-equivalence}, we have $\mathbf{Pow}(S)/\!\sim_G = \cup_{i=1}^{2^{\mathrm{deg}(Y)}} \text{\textbf{recover}}(M_G^i) = \text{\textbf{recover}}(G)$.
\end{proof}
\newpage
\section{Complexity analysis}
\label{sec: appendix complexity}

We first introduce some notations and definitions that will be used in this section. We omit the subscript and denote $G_{S}$ as $G$, $d_S$ as $d$ for brevity. We use the subscript $X_i, X_j \in \mathbf{X}$, $V_i, V_j \in \mathbf{V}$ to denote variables and vertices; the superscript $S^\prime, S^{\prime \prime}\subseteq S$, $\mathbf{X}^i, \mathbf{X}^j \subseteq \mathbf{X}$, and $\mathbf{V}^i, \mathbf{V}^j \subseteq \mathbf{V}$ to denote variable and vertex subsets. For a vertex $V_i$, denote $\mathrm{deg}(V_i):=|\mathbf{Neig}(V_i)|$ as its degree. Unless otherwise specified, the causal graph in this section can be either a DAG or a Maximal Ancestral Graph (MAG). In a MAG, we use $\mathbf{C}$, $\mathbf{L}$ to denote the selection set and the latent set, respectively. In a causal graph, we use $*\!-\!*$ to denote an edge with any possible orientation ($\to, \leftarrow$ for a DAG; $\to, \leftarrow, \leftrightarrow, - $ for a MAG). 

A chunk vertex is a vertex of degree $2$. Recall a chain vertex if a vertex of degree $\leq 2$. A path $p:=<V_1,V_2,...,V_l>$ is a sequence of distinct vertices with $V_i$ being adjacent to $V_{i+1}$ for $i=1,2,...,l-1$. The length of the path $p$ is $l$. The path $p$ can be \emph{blocked} by a vertex set $\mathbf{V}^i$ means it can be \emph{d}-separated by $\mathbf{V}^i$ when $G$ is a DAG, and \emph{m}-separated by $\mathbf{V}^\prime$ when $G$ is a MAG. A tree is an undirected graph in which any two vertices are connected by exactly one path. In a rooted tree, the distance of a vertex $V_i$ to the root is the length of the path between them. The parent of a vertex $V_i$ is the vertex connected to $V_i$ on the path to the root. A child of a vertex $V_i$ is a vertex of which $V_i$ is the parent. A leaf is a vertex with no child. An internal vertex is a vertex that is not a leaf.

We represent time complexity with the following notions:
\begin{enumerate}
    \item the Big-O notation $f(d)=O(g(d))$, which means $f$ is bounded above by $g$ asymptotically, \emph{i.e.}, $\forall k>0, \exists\, d_0, \forall d>d_0, |f(d)| \leq kg(d)$.
    \item the Small-$\omega$ notation $f(d)=\omega(g(d))$, which means $f$ dominates $g$ asymptotically, \emph{i.e.}, $\forall k>0, \exists\, d_0, \forall d>d_0, f(d) > kg(d)$.
    \item the Big-$\Theta$ notation $f(d)=\Theta(g(d))$, which means $f$ and $g$ have asymptotically the same rank, \emph{i.e.}, $\exists\, k_1>0, \exists\, k_2>0, \exists\, d_0, \forall d>d_0, k_1 g(d)\leq |f(d)| \leq k_2 g(d)$.
    \item $f=\mathrm{P}(d)$ if $f$ has a polynomial complexity w.r.t. $d$, $f=\mathrm{NP}(d)$ if the complexity is larger than any polynomial function.
\end{enumerate}

\subsection{Complexity of Alg.~\ref{alg: recover g-equivalence}: Equivalence classes recovery}
\label{sec:appendix complexity alg2}

We first introduce the following lemma, which studies the number of leaf vertices in a tree.

\begin{proposition}[Number of leaf vertices in a tree]
\label{lemma: property of tree}
In a tree, denote $d_L$ as the number of leaf vertices, $d_{>2}$ as the number of non-chain vertices. Then, we have $d_L \geq d_{>2}+2$.
\end{proposition}
\begin{proof}
    Denote $d_T$ as the number of all vertices, then, by the handshaking lemma, we have:
    \begin{equation}
        d_L+2(d_T-d_L-d_{>2}) + 3d_{>2} \leq \sum_{i=1}^{d_T} \mathrm{deg}(V_i) = 2(d_T-1),
    \end{equation}
    which indicates $d_L \geq d_{>2}+2$.
\end{proof}

\begin{proposition}
\label{appx:complexity alg recover}
     The time complexity of Alg. \ref{alg: recover g-equivalence} is $\Theta(N_G)$, hence it can be bounded by $O(N_G)$.
\end{proposition}

\begin{proof} 
Alg.~\ref{alg: recover g-equivalence} is a recursive algorithm, its complexity is decided by the size of the recursion tree. 

Specifically, in the recursion tree of Alg.~\ref{alg: recover g-equivalence}, the number of all vertices $d_T$ equals to the complexity of Alg. \ref{alg: recover g-equivalence}, while the number of leaf vertices $d_L$ equals to $N_G$. Each internal vertex in the recursion tree has at least two children because the for-loop in line 6 executes at least twice. Since each internal vertex also has a parent, its degree $>2$. Then, by Lemma~\ref{lemma: property of tree}, $d_T$ is at most twice as $d_L$. Hence, the complexity of Alg.~\ref{alg: recover g-equivalence} is $\Theta(N_G)$.
\end{proof}

\clearpage
\subsection{Preliminary results for complexity analysis}

\begin{lemma}
If $f(d)=\omega(\mathrm{log}(d))$, then $2^{f(d)}=\omega(d^m)$ for any constant $m$. In other words, $2^{f(d)}=\mathrm{NP}(d)$.
\label{lemma: property of small omega}
\end{lemma}

\begin{proof}
By the definition of $f(d)=\omega(\mathrm{log}(d))$, $\forall k+1 >0, \exists\, d_0$ such that $\forall d>d_0 f(d)>(k+1)(\mathrm{log}(d))=k\mathrm{log(d)}+log(d)$. As a result, $\forall k>0, \forall m+1>0, \exists\, d_1:=\max\{d_0,\mathrm{log}(k)\}$ such that $\forall d>d_1, f(d)>m\mathrm{log}(d)+\mathrm{log}(d)>m\mathrm{log}(d)+\mathrm{log}(k)$, which is equivalent to have $2^{f(d)}>kn^m$. Thus, we have $2^{f(d)}=\omega(d^m)$ by definition.
\end{proof}

\begin{claim}[Chain]
For any causal graph $G$ whose skeleton is a chain, \emph{i.e.}, $Y*\!-\!*X_{d}*\!-\!*X_{d-1}*\!-\!*\cdots *\!-\!*X_1$, we have $N_G=d+1$.
\label{example: chain}
\end{claim}
\begin{proof} We prove this claim with Alg.~\ref{alg: recover g-equivalence} and an induction on $d$.

\textbf{Base.} $d=1$, $N_G=2=d+1$.

\textbf{Induction hypotheses.} Suppose $N_G=d+1$ holds for any chain with $d$ vertices.

\textbf{Step.} For a chain with $d+1$ vertices. We consider the case when $X_{d+1}$ is a collider (similarly a non-collider). With $\{X_{d+1}\}$ as the selection set, the induced MAG is $Y*\!-\!*X_{d}*\!-\!*X_{d-1}*\!-\!*\cdots*\!-\!*X_1$, which is a chain with $d$ vertices and has $d+1$ equivalence classes by the induction hypotheses. With $\emptyset$ as the selection set, the induced MAG is $Y \ \ X_d*\!-\!*X_{d-1}*\!-\!*\cdots*\!-\!*X_1$ and has $1$ equivalence class. Therefore, we have $N_G=d+1+1=d+2$.
\end{proof} 

\begin{claim}[Circle]
For any causal graph $G$ whose skeleton is a circle, \emph{i.e.}, $Y*\!-\!*X_d*\!-\!*X_{d-1}*\!-\!*\cdots*\!-\!*X_1$ and $Y*\!-\!*X_1$, we have $N_{G}=(d^2+d+2)/2=\Theta(d^2)$.
\label{example: circle}
\end{claim}
\begin{proof}
We prove the claim with Alg.~\ref{alg: recover g-equivalence} and Claim~\ref{example: chain}. Denote a circle with $d$ vertices as $G_d$.

Consider the case when $X_d$ is a collider (similarly a non-collider). With $\{X_{d+1}\}$ as the selection set, the induced MAG is $Y*\!-\!*X_{d}*\!-\!*\cdots*\!-\!*X_1$ and $Y*\!-\!*X_1$, \emph{i.e.}, a circle with $d-1$ vertices. With $\emptyset$ as the selection set, the induced MAG is $Y*\!-\!*X_1*\!-\!*\cdots *\!-\!* X_{d-1}$, \emph{i.e.}, a chain with $d-1$ vertices. Hence, we have $N_{G_d}=d+N_{G_{d-1}}$, which means $\{N_{G_d}\}_d$ is an arithmetic sequence and $N_{G_d}=\Theta(d^2)$.
\end{proof}

\begin{lemma}[Adding/deleting an edge]
For any causal graph $G$, adding an edge does not decrease $N_G$, deleting an edge does not increase $N_G$.
\label{lemma: add delete edges}
\end{lemma}
\begin{proof}
For a causal graph $G_0$, add an edge in it and call the resulting graph $G_1$ (which can also be viewed as deleting an edge in $G_1$ and getting a graph $G_0$). We prove $N_{G_0} \leq N_{G_1}$ by showing $\forall \, S^\prime, S^{\prime \prime}$, $S^\prime \not \sim_{G_0} S^{\prime \prime} \Rightarrow S^\prime \not \sim_{G_1} S^{\prime \prime}$.

Prove by contradiction. Suppose there are $S^\prime, S^{\prime \prime}$ such that $S^\prime \not \sim_{G_0} S^{\prime \prime}$ and $S^\prime \sim_{G_1} S^{\prime \prime}$. By $S^\prime \sim_{G_1} S^{\prime \prime}$, we have $\exists\, S_{\cap} \subseteq_{G_1} S^\prime \cap S^{\prime\prime}$ such that $Y \ind_{G_1} \mathbf{X}_{S_\cap^c}| \mathbf{X}_{S_\cap}$. Because adding an edge does not change the vertex sets, we have $S_\cap \subseteq_{G_0} S^\prime \cap S^{\prime\prime}$. Because $S^\prime \not \sim_{G_0} S^{\prime\prime}$, we have $Y \not \ind_{G_0} \mathbf{X}_{S_\cap^c} | \mathbf{X}_{S_\cap}$. In other words, there is a path $p$ in $G_0$ between $Y$ and $\mathbf{X}_{S_\cap^c}$ such that $p$ can not be blocked by $\mathbf{X}_{S_\cap}$. 

In the following, we show that in $G_1$ the path $p$ can not be blocked by $\mathbf{X}_{S_\cap}$, neither; which contradicts with $Y \ind_{G_1} \mathbf{X}_{S_\cap^c} | \mathbf{X}_{S_\cap}$. Specifically, $p$ can not be blocked by $\mathbf{X}_{S_\cap}$ in $G_0$ means $\mathbf{X}_{S_{\cap}}$ does not contain any non-collider on $p$, and $\mathbf{X}_{S_\cap}$ contains every collider (or its descendants) on $p$ in $G_0$. Because in $G_1$, $p$ is still a path between $Y$ and $\mathbf{X}_{S_\cap^c}$, and any collider $X_i$ on $p$ in $G_0$ is still a collider on $p$ in $G_1$. Any vertex in $\mathbf{De}(X_i)$, where $X_i$ is a collider on $p$ in $G_0$, is still a descendant of the collider on $p$ in $G_1$. Any non-collider on $p$ in $G_0$ is still a non-collider on $p$ in $G_1$. We have the path $p$ can not be blocked by $\mathbf{X}_{S_\cap}$ in $G_1$, neither.
\end{proof}

\begin{lemma}[Melting property]
For a causal graph $G$ over $\mathbf{X} \cup Y$. Consider three three disjoint non-empty vertex sets ${\mathbf{C}}$, ${\mathbf{L}}$, and ${\mathbf{O}}:=\mathbf{X} \backslash ({\mathbf{L}} \cup {\mathbf{C}})$. Let ${M_G}$ be the MAG constructed\footnote{We say that the $M_G$ is constructed from $G$ via ``melting'' vertices in $\mathbf{C}$ and $\mathbf{L}$.} over $\mathbf{O}$, with $\mathbf{C}$ as the selection set, $\mathbf{L}$ as the latent set. Then, we have $N_{G}>N_{M_G}$.
\label{lemma: melting property}
\end{lemma}

\begin{proof}
Recall that Alg.~\ref{alg: recover g-equivalence} traverses over every $S^\prime\subseteq \mathrm{Neig}(Y)$ and constructs $2^{\mathrm{deg}(Y)}$ MAGs, and $N_G$ is the summation of the number of equivalence classes in the $2^{\mathrm{deg}(Y)}$ MAGs.

Now, modified Alg.~\ref{alg: recover g-equivalence} in the following way. For element $S^\prime\subseteq \mathrm{Neig}(Y)$, if $S^\prime$ matches $<{\mathbf{C}},{\mathbf{L}}>$\footnote{$S^\prime$ matches $<{\mathbf{C}},{\mathbf{L}}>$ means $\mathbf{Neig}(Y)\cap {\mathbf{C}} \subseteq \mathbf{X}_{S^\prime}$ and $(\mathbf{Neig}(Y) \cap {\mathbf{L}} )\subseteq (\mathbf{Neig}(Y) \backslash \mathbf{X}_{S^\prime}$.}, construct a MAG and recover the equivalence classes in it; Otherwise, ignore $S^\prime$ and continue. In this regard, the modified algorithm recovers the equivalence classes in $M_G$. Since parts of the $2^{\mathrm{deg}(Y)}$ MAGs are ignored, we have $N_{G}>N_{M_G}$.
\end{proof}

\begin{claim}[Complexity of tree]
For any causal graph $G$ whose skeleton is a tree with $d_L$ leaves, $N_G=\omega(c^{d_L})$ for some $1<c<2$.
\label{example: tree}
\end{claim}

\begin{proof}
We first prove the following claim. Suppose the skeleton of $G$ is a tree with $d_L$ leaves, every internal vertex of the tree is a non-chain vertex, then $N_G=\omega(c^{d_L})$ for some $1<c<2$.

Recall that an inducing path $p$ with respect to $<{\mathbf{C}},{\mathbf{L}}>$ between $V_1$ and $V_2$ is a path where every non-endpoint vertex on $p$ is either in ${\mathbf{L}}$ or a collider, and every collider on $p$ is an ancestor\footnote{$V_1$ is called an ancestor of $V_2$ if $V_1=V_2$ or there is a directed path from $V_1$ to $V_2$.} of either $V_1, V_2$, or a member of ${\mathbf{C}}$. Two vertices in the MAG are adjacent if there is an inducing path between them with respect to $<{\mathbf{C}},{\mathbf{L}}>$.

\textbf{1.} To show $N_G=\omega(c^{d_L})$, we can use Lemma \ref{lemma: melting property} and show $\exists\,$ sets ${\mathbf{O}}, {\mathbf{C}}, {\mathbf{L}}$ such that:

\begin{enumerate}[label=(\alph*),leftmargin=*,align=left]
    \item there is an inducing path w.r.t. $<{\mathbf{C}},{\mathbf{L}}>$ between $Y$ and every vertex in ${\mathbf{O}}$, and 
    \item $|{\mathbf{O}}|=\Theta(d_L)$.
\end{enumerate}

\textbf{2.} Put vertices in $G$ into different layers according to their distances from $Y$. To prove \textbf{1.}, we can construct $G$ layer by layer and show that every time the number of leaves increases by $r$, $\exists\,$ rules to adjust the sets ${\mathbf{O}}, {\mathbf{C}}, {\mathbf{L}}$ such that:

\begin{enumerate}[label=(\alph*),leftmargin=*,align=left]
    \item there is an inducing path w.r.t. $<{\mathbf{C}},{\mathbf{L}}>$ between $Y$ and every vertex in ${\mathbf{O}}$, and 
    \item $|{\mathbf{O}}|$ increases by at least $\left \lfloor \frac{r}{2} \right \rfloor$.
\end{enumerate}

\textbf{3.} Note that any newly added vertex is connected to an existing leaf vertex, otherwise, we should have added it in the previous layer. Any newly added vertex is only connected to one existing leaf vertex, otherwise, the graph is not a tree. As a result, to prove \textbf{2.}, we can consider adding vertices $X_{i_1},...,X_{i_r} (r\geq 2)$ to an existing leaf vertex $V_L$ and provide rules satisfied properties (a) and (b) in \textbf{2.}. 

We first introduce some notations that will be used in the rules. Denote the edge between $V_L$ and its parent $Pa_L$ in the tree as $E_L$, edges between $V_L$ and its children in the tree, \emph{i.e.}, $X_{i_1},...,X_{i_r}$, as $E_{i_1},...,E_{i_r}$, respectively. Call $V_L$ as a complete (non-)collider if it is a (non-)collider on any path $p_j:=<Pa_L,V_L,X_{i_j}>$ for $j=1,2,...,r$.

The rules are given in Alg. \ref{alg: rules to melt a tree}, and their validity is explained as follows:

\begin{algorithm}[ht!]
\caption{Rules to adjust the sets.}
\begin{algorithmic}[1]
    \IF{$V_L$ is $Y$} \alglinelabel{rule.1.start}
        \STATE ${\mathbf{O}}.\mathrm{add}(X_{i_1},...,X_{i_r})$. \alglinelabel{rule.1.end}
    \ELSIF{$V_L \in {\mathbf{O}}$} \alglinelabel{rule.2.start}
        \IF{$V_L$ is a complete collider} \alglinelabel{rule.2.1.start}
            \STATE ${\mathbf{O}}.\mathrm{remove}(V_L)$,  ${\mathbf{C}}.\mathrm{add}(V_L)$. 
            \STATE ${\mathbf{O}}.\mathrm{add}(X_{i_1},...,X_{i_r})$. \alglinelabel{rule.2.1.end}
        \ELSIF{$V_L$ is a complete non-collider} \alglinelabel{rule.2.2.start}
            \STATE ${\mathbf{O}}.\mathrm{remove}(V_L)$, ${\mathbf{L}}.\mathrm{add}(V_L)$.
            \STATE ${\mathbf{O}}.\mathrm{add}(X_{i_1},...,X_{i_r})$. \alglinelabel{rule.2.2.end}
        \ELSE
            \IF{$r=2$} \alglinelabel{rule.2.3.start}
                \STATE suppose $E_{i_1}$ has a tail on $V_L$, $E_{i_2}$ has an arrowhead on $V_L$.
                \STATE ${\mathbf{O}}.\mathrm{add}(X_{i_2})$.
                \IF{$E_{i_1}$ has a tail on $X_{i_1}$}
                    \STATE ${\mathbf{O}}.\mathrm{remove}(V_L)$, ${\mathbf{L}}.\mathrm{add}(V_L)$.
                    \STATE ${\mathbf{O}}.\mathrm{add}(X_{i_1})$.
                \ELSE
                    \STATE keep $V_L$ in ${\mathbf{O}}$.
                    \STATE ${\mathbf{C}}.\mathrm{add}(X_{i_1})$.
                \ENDIF \alglinelabel{rule.2.3.end}
            \ELSE \alglinelabel{rule.2.4.start}
                \STATE suppose $E_{i_1}$ has a tail on $V_L$.
                \STATE ${\mathbf{O}}.\mathrm{remove}(V_L)$, ${\mathbf{L}}.\mathrm{add}(V_L)$.
                \IF{$E_{i_1}$ has a tail on $X_{i_1}$}
                    \STATE ${\mathbf{O}}.\mathrm{add}(X_{i_1},...,X_{i_r})$.
                \ELSE
                    \STATE ${\mathbf{C}}.\mathrm{add}(X_{i_1})$.
                    \STATE ${\mathbf{O}}.\mathrm{add}(X_{i_2},...,X_{i_r})$ 
                \ENDIF
            \ENDIF \alglinelabel{rule.2.4.end}
        \ENDIF \alglinelabel{rule.2.end}
    \ELSE
        \IF{$V_L$ is a complete collider} \alglinelabel{rule.3.start}
            \STATE keep $V_L$ in ${\mathbf{C}}$.
            \STATE ${\mathbf{O}}.\mathrm{add}(X_{i_1},...,X_{i_r})$. \alglinelabel{rule.3.1.end}
        \ELSE
            \STATE suppose $E_{i_1}$ has a tail on $V_L$.  \alglinelabel{rule.3.2.start}
            \STATE ${\mathbf{C}}.\mathrm{remove}(V_L)$, ${\mathbf{L}}.\mathrm{add}(V_L)$.
            \IF{$E_{i_1}$ has a tail on $X_{i_1}$}
                \STATE ${\mathbf{O}}.\mathrm{add}(X_{i_1},...,X_{i_r})$.
            \ELSE
                \STATE ${\mathbf{C}}.\mathrm{add}(X_{i_1})$.
                \STATE ${\mathbf{O}}.\mathrm{add}(X_{i_2},...,X_{i_r})$.
            \ENDIF
        \ENDIF \alglinelabel{rule.3.end}
    \ENDIF
\end{algorithmic}
\label{alg: rules to melt a tree}
\end{algorithm}

\textbf{i)} Rule-1 (line \ref{rule.1.start} to line \ref{rule.1.end}). Because ${\mathbf{L}}={\mathbf{C}}=\emptyset$ and $X_{i_1},...,X_{i_r}$ are adjacent to $Y$, the requirement in \textbf{2.} (a) is satisfied. The number of leaves increases by $r-1$, and $|{\mathbf{O}}|$ increases by $r$, the requirement in \textbf{2.} (b) is also satisfied.

\textbf{ii)} Rule-2 (line \ref{rule.2.start} to line \ref{rule.2.end}), where $V_L$ is a covariate vertex. $V_L \in {\mathbf{O}}$ indicates there is an inducing path between $Y$ and $V_L$. When $V_L$ is a complete collider (line \ref{rule.2.1.start} to line \ref{rule.2.1.end}) and put into ${\mathbf{C}}$, the inducing path is extended to each vertex in $X_{i_1},...,X_{i_r}$ and \textbf{2.} (a) holds. The number of leaves increase by $r-1$, and $|{\mathbf{O}}|$ increases by $r-1$, which means \textbf{2.} (b) also holds. When $V_L$ is a complete non-collider (line \ref{rule.2.2.start} to line \ref{rule.2.2.end}), the proof is similar. 

Otherwise, \emph{i.e.}, $V_L$ is a collider on some paths, and a non-collider on the other paths, which indicates $E_L$ has an arrowhead on $V_L$. We first look at the case where $r=2$ (line \ref{rule.2.3.start} to line \ref{rule.2.3.end}), then look at the cases where $r\geq 3$ (line \ref{rule.2.4.start} to line \ref{rule.2.4.end}). 

When $r=2$ (line \ref{rule.2.3.start} to line \ref{rule.2.3.end}), without loss of generality, suppose $E_{i_1}$ has a tail on $V_L$, $E_{i_2}$ has an arrowhead on $V_L$. If $E_{i_1}$ also has a tail on $X_{i_1}$ and the edge is $V_L-X_{i_1}$, which means $V_L$ is ancestor of a member of selection set. Besides, $V_L$ is a non-collider on path $<Pa_L,V_L,X_{i_1}>$ and a collider on the path $<PA_L,V_L,X_{i_2}>$. As a result, when $V_L$ is put into ${\mathbf{L}}$ and $X_{i_1}, X_{i_2}$ are put into ${\mathbf{O}}$, there are inducing paths between $Y$ and $X_{i_1}, X_{i_2}$ (\textbf{2.} (a) holds). Similarly, when $E_{i_1}$ is $V_L \to X_{i_1}$, putting $X_{i_1}$ into ${\mathbf{C}}$ makes sure $V_L$ is a collider on the path $<PA_L,V_L,X_{i_2}>$ with a descendant in ${\mathbf{C}}$, thus \textbf{2.} (a) holds. In both situations, the number of leaves increases by $1$, and $|{\mathbf{O}}|$ also increases by $1$, which indicates \textbf{2.} (b) holds.

When $r \geq 3$ (line \ref{rule.2.4.start} to line \ref{rule.2.4.end}), because $V_L$ is neither a complete collider nor a complete non-collider, one edge of $E_{i_1},...,E_{i_r}$ has a tail on $V_L$. Without loss of generality, we suppose this edge is $E_{i_1}$. Then, similarly to the scenario where $r=2$, we put $V_L$ into ${\mathbf{L}}$, $X_{i_1}$ into ${\mathbf{O}}$ if $E_{i_1}$ is $V_L-X_{i_1}$ and into ${\mathbf{C}}$ if $V_L \to X_{i_1}$. This makes sure the existence of inducing paths between $Y$ and vertices newly added to ${\mathbf{O}}$ (\textbf{2.} (a) holds). In addition, the number of leaves increases by $r-1$, and $|{\mathbf{O}}|$ increases by at least $r-2$, which indicates \textbf{2.} (b) holds.

\textbf{iii)} Rule-3 (line \ref{rule.3.start} to line \ref{rule.3.end}). Firstly note that if $V_L$ is not in ${\mathbf{O}}$, then $V_L \in {\mathbf{C}}$ because a leaf is never put into ${\mathbf{L}}$. Besides, note that we only put a vertex into ${\mathbf{C}}$ when its parent vertex in the tree has an arrowhead on it. These analyses mean $V_L \in {\mathbf{C}}$ and the edge $E_L$ has an arrowhead on $V_L$.

When $V_L$ is a complete collider (line \ref{rule.3.start} to line \ref{rule.3.1.end}), keeping $V_L \in {\mathbf{C}}$ ensures that existing inducing paths are not damaged (any ancestor of $V_L$ still has a descendant in ${\mathbf{C}}$). It also ensures that there are inducing paths between $Y$ and vertices newly added to ${\mathbf{O}}$, \emph{i.e.}, $X_{i_1},...,X_{i_r}$. These indicate \textbf{2.} (a) holds. In addition, the number of leaves increases by $r-1$, and $|{\mathbf{O}}|$ increases by $r$, which indicates \textbf{2.} (b) holds.

When $V_L$ is not a complete collider (line \ref{rule.3.2.start} to line \ref{rule.3.end}), there is an edge in $E_{i_1},...,E_{i_r}$ such that it has a tail on $V_L$. Without loss of generality, suppose the edge is $E_{i_1}$. When $E_{i_1}$ also has a tail on $X_{i_1}$, $V_L$ is an ancestor of a member of the selection set. As a result, putting $V_L$ into ${\mathbf{L}}$ ensures that existing inducing paths are not damaged and there are inducing paths between $Y$ and vertices newly added to ${\mathbf{O}}$, \emph{i.e.}, $X_{i_1},...,X_{i_r}$. When $E_{i_1}$ has an arrowhead on $X_{i_1}$, putting $V_L$ into ${\mathbf{C}}$, $X_{i_1}$ into ${\mathbf{C}}$ ensures that existed inducing paths are not damaged (any ancestor of $V_L$ still has a descendant $X_{i_1}$ in ${\mathbf{C}}$) and there are inducing paths between $Y$ and vertices newly added to ${\mathbf{O}}$, \emph{i.e.}, $X_{i_2},...,X_{i_r}$. Hence, in both scenarios, \textbf{2.} (a) holds.  In addition, the number of leaves increases by $r-1$, and $|{\mathbf{O}}|$ increases by at least $r-1$, which indicates \textbf{2.} (b) holds.

To conclude, we have proved the claim that any causal graph $G$ whose skeleton is a tree with $d_L$ leaves, and every internal vertex of the tree is a non-chain vertex, has $N_G=\omega(c^{d_L})$ for some $1<c<2$. Next, we prove Claim~\ref{example: tree}, \emph{i.e.}, any causal graph $G$ whose skeleton is a tree with $d_L$ leaves, has $N_G=\omega(c^{d_L})$ for some $1<c<2$.

\textbf{1.} Any interval vertex in $G$ is either a trunk vertex or a non-chain vertex. Use Lemma \ref{lemma: melting property} to melt all chunk vertices in $G$ (put a chunk vertex into ${\mathbf{L}}$ if it is a non-collider, into ${\mathbf{C}}$ if otherwise) and call the resulted graph as $\underline{G}$. Because $G$'s skeleton is a tree and there is no cycle in it, $\underline{G}$'s skeleton is also a tree where all interval vertices are non-chain vertices and there are still $d_L$ leaves.

\textbf{2.} By Lemma \ref{lemma: melting property}, we have $N_G > N_{\underline{G}}=\omega(c^{d_L})$ and thus $N_G=\omega(c^{d_L})$ for some $1<c<2$.
\end{proof}

\begin{lemma}[Maximum leaf spanning tree]
In a connected undirected graph $G$ with $d$ vertices, if every vertex is either a non-chain vertex or a vertex of deg$=1$, then $G$ has a spanning tree with $\Theta(d)$ leaves.
\label{lemma: max leaf spaninng tree}
\end{lemma}

\begin{proof}
The proof is similar to Lemma 1 in \cite{neal2022maxleaf}. We first discuss a property that any $G$'s spanning tree satisfies, then focus on the number of leaves in $G$'s maximum leaf spanning tree.

Suppose $T$ is a spanning tree of $G$, denote the number of leaves, chunk vertices, and non-chain vertices in $T$ as $d_1, d_2, d_{>2}$, respectively. Note that a leaf in $T$ is not necessarily a vertex of deg$=1$ in $G$, however, a chunk/non-chain vertex in $T$ must be a non-chain vertex in $G$.

Let $G^\prime$ be a subgraph of $T$ consisting of: \textbf{i)} all the $d_2$ chunk vertices in $T$, and \textbf{ii)} edges among these chunk vertices in $T$. As a result, $G^\prime$ is a subgraph (maybe not a connected one) of $T$ with maximum degree $2$.

We claim the number of edges in $G^\prime$ is at least $d-4d_1-1$. The proof is as follows. \textbf{i)} Construct a tree $T^\prime$ from $T$ by slicing out all chunk vertices in $T$. Specifically, for each maximal path $p_i:=<X_{i_1},X_{i_2},...,X_{i_{(l-1)}},X_{i_l}>$ in $T$ such that all the intermediate vertices $X_{i_2},...,X_{i_{(l-1)}}$ are chunk vertices in $T$, remove the edge $X_{i_j}-X_{i_{(j+1)}}$ for $j=2,3,...,l-2$ and the intermediate vertices. Then, replace them with an edge $X_{i_1}-X_{i_l}$ (which is not necessarily also in $G$). \textbf{ii)} In $T^\prime$, every internal vertex is a non-chain vertex and there are still $d_1$ leaves. As a result, in $T^\prime$, the number of edges is at most $2d_1$. \textbf{iii)} Now compare edges in $T^\prime$ and $G^\prime$. (a) If an edge in $T^\prime$ is also in $T$, then this edge is not in $G^\prime$ because it is incident with a non-chain vertex in $T$ and $G^\prime$ does not contain any of such edge. (b) If an edge in $T^\prime$ is constructed by slicing out chunk vertices in $T$, then there is two edges missing in $G^\prime$ compared with $T$.  \textbf{iv)} As a result, for each edge in $T^\prime$, $G^\prime$ is missing at most two edges compared with $T$. Hence, $G^\prime$ is missing at most $2 \cdot 2d_1 = 4d_1$ edges compared with $T$. Because there are $d-1$ edges in $T$, there are at least $d-4d_1-1$ edges in $G^\prime$.

Because there is at most $d$ vertices in $G^\prime$, we know in $G^\prime$, the number of vertices minuses the number of edges $\leq 4d_1$. Hence, $G^\prime$ contains at most $4d_1$ paths\footnote{counting each isolated vertex in $G^\prime$ as a path, a path with $l$ vertices has $l-1$ edges.}. A path with a single vertex indicates there is a vertex in $G^\prime$ without two neighbors in $G^\prime$, a path with $\geq 2$ vertices indicates there are two vertices in $G^\prime$ without two neighbors in $G^\prime$. Because $G^\prime$ contains at most $4d_1$ paths, we know there are at most $8d_1$ vertices in $G^\prime$ without two neighbors in $G^\prime$.

Because $G^\prime$ contains all the $d_2$ chunk vertices in $T$, there are at least $d_2-8d_1$ vertices in $G^\prime$ having two neighbors in $G^\prime$. Back to $T$, we have at least $d_2-8d_1$ vertices having the following properties: \textbf{i)} They are chunk vertices in $T$, and \textbf{ii)} both of their two neighbors are chunk vertices in $T$. Call such vertices pipe vertices.

Now we discuss the number of leaves in $G$'s maximum leaf spanning tree.

Suppose $T$ is a maximum leaf spanning tree. Then, there are at least $d_2-8d_1$ pipe vertices in $T$. Let $U$ be a pipe vertex in $T$ and consider any edge $U*\!-\!*V$ from $U$ that is not in $T$. We show $V$ must be a leaf in $T$. Prove by contradiction. Suppose $V$ is not a leaf in $T$, then add the edge $U*\!-\!*V$ into $T$ and delete one of the other edges incident to $U$ to break the cycle (so the result is still a spanning tree). This makes one of $U$'s neighbors in $T$ a leaf, which indicates the number of leaves increases. Because $T$ is a maximum spanning tree, this is a contradiction and $V$ is a leaf in $T$. To conclude, for every pipe vertex $U$, every edge incident to $U$ except the two in $T$ goes to a leaf in $T$.

Each pipe vertex in $T$ is a non-chain vertex in $G$, so we have at least $d_2-8d_1$ pipe vertices in $T$ having $\mathrm{deg}\geq 3$ in $G$. As a result, there are at least $d_2-8d_1$ edges incident to pipe vertices to the $d_1$ leaves in $T$.

Next, we prove $d_2-8d_1 \leq d_1$, which indicates $d_1 \geq \frac{1}{9}d_2$ and together with Lemma \ref{lemma: property of tree} indicates $d_1 = \Theta(d)$. Prove by contradiction. Suppose $d_2-8d_1 > d_1$. Then at least one leaf in $T$, say $V$, is connected to two pipe vertices in $T$, say $U_1,U_2$. Then, add the edges $V*\!-\!*U_1$, $V*\!-\!*U_2$ to $T$ and delete one of the incident edges of $U_1, U_2$ each, we lose one leaf vertex $V$ in $T$, however, obtain two more (one of $U_1$'s neighbors and one of $U_2$'s neighbors), which contradict with $T$ being a maximum leaf spanning tree.
\end{proof}

\begin{lemma}
For a connected causal graph $G$, if $d_{>2}=\omega(\mathrm{log}(d))$, then $N_G=\mathrm{NP}(d)$.
\label{lamma: large n >= 3}
\end{lemma}

\begin{proof}
We prove the lemma by showing that any $G$ with $d_{>2}=\omega(\log(d))$ has a maximum leaf spanning tree with $\omega(d_{>2})$ leaves. In the regard, by Lemma~\ref{example: tree}, the maximum leaf spanning tree has at least $\omega(2^{\log(d)})=\mathrm{NP}(d)$ equivalence classes. Then, we can delete the edges in $G$ until $G$ becomes its maximum leaf spanning tree and use Lemma~\ref{lemma: add delete edges} to show $N_G=\mathrm{NP}(d)$.

We first construct a lower bound graph $\underline{G}$ of $G$ such that $N_{\underline{G}}\leq N_G$, by keeping all vertices of deg$=1$, deg$>2$, and removing all the chunk vertices which have deg$=2$. Specifically, use Lemma \ref{lemma: melting property} and iteratively melt the chunk vertex $X_i$, whose two neighbors denoted as $A_i,B_i$, with the following rules: \textbf{i)} If $\mathrm{deg}(A_i)=1$, or $\mathrm{deg}(B_i)=1$, or $\mathrm{deg}(A_i)=\mathrm{deg}(B_i)=2$, or $\mathrm{deg}(A_i)=\mathrm{deg}(B_i)=3$, put $X_i$ into ${\mathbf{L}}$ if it is a non-collider, ${\mathbf{C}}$ if otherwise. \textbf{ii)} Otherwise, one of $\mathrm{deg}(A_i), \mathrm{deg}(B_i)$ is $2$, the other one is $3$. Without loss of generality, suppose $\mathrm{deg}(A_i)=2, \mathrm{deg}(B_i)=3$. If $A_i$ is adjacent to $B_i$ and $A_i, X_i, B_i$ form a cycle, then delete the edge $A_i*\!-\!*X_i$. Otherwise, melt the vertex $X_i$ (put it into ${\mathbf{L}}$ if it is a non-collider, ${\mathbf{C}}$ if otherwise). 

Next, we show that the $\underline{G}$ is a connected graph with at least $d_{>2}$ vertices, such that every vertex in $\underline{G}$ is either of deg$=1$ or a non-chain vertex. This is because all vertices of deg$=1$ in $G$ are still of deg$=1$ in $\underline{G}$, all non-chain vertices in $G$ are still of deg$\geq 3$ in $\underline{G}$. 

Hence, $\underline{G}$ has a maximum leaf spanning tree $\underline{T}$ with at least $\Theta(d_{>2})$ leaves, according to Lemma \ref{lemma: max leaf spaninng tree}. By Claim \ref{example: tree}, we have $N_{\underline{T}}=\omega(c^{d_{>2}})$ for some $1<c<2$. By Lemma \ref{lemma: melting property}, we have $N_G > N_{\underline{G}} > N_{\underline{T}}$, which together with $d_{>2}=\omega(\mathrm{log}(d))$ and Lemma \ref{lemma: property of small omega} indicates $N_G=\mathrm{NP}(d)$.
\end{proof}

\begin{lemma}
\label{lemma: property 3 of d-separation}
Consider two vertex sets $\mathbf{V}^i, \mathbf{V}^j$, and a path $p$ between two vertices $V_1,V_l$. If $p$ can be blocked by $\mathbf{V}^i$, but can not be blocked by $\mathbf{V}^i \cup \mathbf{V}^j$, then, we have $V_1 \not \ind_G \mathbf{V}^j |\mathbf{V}^i$. 
\end{lemma}

\begin{proof}
To prove the lemma, we construct a path $p_1$ between $V_1$ and a vertex in $\mathbf{V}^j$ such that $p_1$ can not be blocked by $\mathbf{V}^i$.

We first prove the following properties \textbf{i)}-\textbf{iv)}:

\textbf{i)} $p$ must contain a collider. Prove by contradiction. Suppose all vertices on $p$ are non-colliders. Then, since $\mathbf{V}^i$ can block $p$, we have $\mathbf{V}^j$ contains at least one non-collider on $p$. Hence, the union set $\mathbf{V}^i \cup \mathbf{V}^j$ also contains a non-collider on $p$. Therefore, $p$ can be blocked by $\mathbf{V}^i\cup \mathbf{V}^j$, which is a contradiction.

\textbf{ii)} In a similar way, we can prove that $\mathbf{V}^i$ and $\mathbf{V}^i \cup \mathbf{V}^j$ do not contain any non-collider on $p$.

\textbf{iii)} Since $p$ can be blocked by $\mathbf{V}^i$ and \textbf{ii)}, we have: $\exists\,$ a collider on $p$ such that the collider and its descendants are all in $\mathbf{V}^i$.

\textbf{iv)} For any collider $V_c$ on $p$, if $V_c$ and any vertex in $\mathbf{De}(V_c)$ are all not in $\mathbf{V}^i$, then, either $V_c$ or a vertex in $\mathbf{De}(V_c)$ is in $\mathbf{V}^j$. This is because if otherwise, $V_i$ and all vertices in $\mathbf{De}(V_i)$ are not in $\mathbf{V}^i \cup \mathbf{V}^j$. Therefore, the path $p$ can be blocked by $\mathbf{V}^i \cup \mathbf{V}^j$, which is a contradiction.

We then construct the path $p_1$ in the following way:

Denote those colliders on $p$ such that themselves and their descendants are all not in $\mathbf{V}^i$ as (in the order of their distance to $V_1$) as $\{V_{c_1},V_{c_2},...,V_{c_l}\}$. 

Now, consider the subpath $p^\prime$ of $p$ with $p^\prime:=<V_1,V_2,...,V_{c_1-1},V_{c_1}>$. We have the following analyses: \textbf{i)} By the definition of $V_{c_1}$, among $V_1,V_2,...,V_{c_1-1}$, all colliders and their descendants are in $\mathbf{V}^i$. \textbf{ii)} Among $V_1,V_2,...,V_{c_1-1}$, all non-colliders are not in $\mathbf{V}^i$. \textbf{iii)} Either $V_{c_1}$ or a vertex in $\mathbf{De}(V_{c_1})$ is in $\mathbf{V}^j$. 

If it is $V_{c_1} \in \mathbf{V}^j$, then we have the path $p_1=<V_1,V_2,...,V_{c_1}>$ between $V_1$ and a vertex in $\mathbf{V}^j$ satisfying that $p_1$ can not be blocked by $\mathbf{V}^i$, because of \textbf{i)} and \textbf{ii)}; If it is a vertex in $\mathbf{De}(\mathbf{V}^j)$ that is in $\mathbf{V}^j$, then, we have the path $p_1=<V_1,V_2,...,V_{c_1}\to \cdots \to V_j)$ for $V_j \in \mathbf{De}(V_{c_1}>$, between $V_1$ and a vertex in $\mathbf{V}^j$ satisfying that $p_1$ can not be blocked by $\mathbf{V}^i$, because of \textbf{i)}, \textbf{ii)}, and the definition of $V_{c_1}$.
\end{proof}

\begin{lemma}[Merging property]
For any causal graph $G$ where $Y$ is adjacent to a vertex $X_0$, and vertices in $\mathbf{V}\backslash\{Y,X_0\}$ are adjacent to at most one vertex in $\{Y,X_0\}$, merge\footnote{The merging operation means contradicting the edge $(Y,X_0)$ and merging $Y, X_0$ into a new vertex $\tilde{Y}$. Edges incident to $\tilde{Y}$ in $\tilde{G}$ are edges incident to either $Y$ or $X_0$ in $G$, their orientations on the $\tilde{Y}$ side can be randomly assigned, and do not influence $N_{\tilde{G}}$, while orientations on the other side keep the same as in $G$.} $Y, X_0$ into a new vertex $\tilde{Y}$ and denote the resulted graph as $\tilde{G}$. Then, we have $N_{\tilde{G}}+1 \leq N \leq 2N_{\tilde{G}}$.
\label{lemma: merging property}
\end{lemma}

\begin{proof}
During the proof, we omit the subscript and denote $N_G$ as $N$, $N_{\tilde{G}}$ as $\tilde{N}$, $\mathbf{X}_{S_i}$ as $\mathbf{X}^i$ for brevity.

Proof of the right side. We show for any $\mathbf{T}\subseteq \{X_0\}, \mathbf{X}^{i}, \mathbf{X}^{j} \subseteq \mathbf{V}\backslash \{Y, X_0\}$, if $\mathbf{X}^i \sim_{\tilde{G}} \mathbf{X}^j$, then $\mathbf{X}^{i} \cup \mathbf{T} \sim_G \mathbf{X}^j \cup \mathbf{T}$. In this regard, for any subset $\mathbf{T}$, there are at most $\tilde{N}$ equivalent classes in $G$, thus $N\leq 2^{|\mathbf{T}|} \tilde{N}=2\tilde{N}$.

\textbf{1.} By $\mathbf{X}^i \sim_{\tilde{G}} \mathbf{X}^j$, we have $\exists\, \mathbf{X}^{ij} \subseteq \mathbf{X}^i \cap \mathbf{X}^j$ such that $\tilde{Y} \ind_{\tilde{G}}  
 \mathbf{X}^{({ij})^c} | \mathbf{X}^{ij}$. In other word, we have $\{Y,X_0\} \ind_G \mathbf{X}^{({ij})^c} | \mathbf{X}^{ij}$.

\textbf{2.} When $\mathbf{T}=\emptyset$, by \textbf{1.}, we have $Y \ind_G  \mathbf{X}^{({ij})^c}  | \mathbf{X}^{ij}$ and thus $\mathbf{X}^i \sim_G \mathbf{X}^j$ holds.

\textbf{3.} When $\mathbf{T}=\{X_0 \}$, by \textbf{1.}, we have: \textbf{i)} any path in $G$ between $\mathbf{X}^{({ij})^c}$ and $Y$ can be blocked by $\mathbf{X}^{ij}$, and \textbf{ii)} any path in $G$ between $\mathbf{X}^{({ij})^c}$ and $\mathbf{T}$ can be blocked $\mathbf{X}^{ij}$. Next, we show any path in $G$ between $\mathbf{X}^{({ij})^c}$ and $Y$ can also be blocked by $\mathbf{X}^{ij}\cup \mathbf{T}$, which indicates $\mathbf{X}^i \cup \mathbf{T} \sim_G \mathbf{X}^j \cup \mathbf{T}$.

Prove by contradiction. Suppose there is a path between $\mathbf{X}^{({ij})^c}$ and $Y$ that can be blocked by $\mathbf{X}^{ij}$ and can not be blocked by $\mathbf{X}^{ij} \cup \mathbf{T}$. By Lemma. \ref{lemma: property 3 of d-separation}, we can construct a path between $\mathbf{X}^{({ij})^c}$ and $\mathbf{T}$ such that it can not be blocked by $\mathbf{X}^{ij}$, which contradicts with \textbf{3. ii)}.

Proof of the left side. We first prove the in-equation under the case when $X_0$ is a complete collider\footnote{$X_0$ is called a complete (non-)collider if it is a (non-)collider on any path $p:=<X_i,X_0,Y>$ with $X_i \in \mathbf{Neig}_G(X_0)\backslash \{Y_0\}$.}. With $\{X_0\}$ as the selection set, the induced MAG is $\tilde{G}$, since there is at least one equivalent class when not conditioning on $X_0$, we have $\tilde{N}+1 \leq N$ by Alg.~\ref{alg: recover g-equivalence} and Lemma~\ref{lemma: melting property}. 

For the cases when $X_0$ is a complete non-collider, or $X_0$ is a partial collider and $\exists\, X_i \in \mathbf{De}(X_0)$ such that $X_i$ is incident to a tail-tail\footnote{A tail-tail edge is an edge $*\!-\!*$ with orientations at both sides being tails, \emph{i.e.}, $-$. According to the definition of a tail-tail edge \cite{zhang2008completeness}, a vertex is incident to a tail-tail edge means it is an ancestor of a member of the selection set.} edge, we can prove $\tilde{N}+1 \leq N$ in a similar way.

Next, we discuss the case where $X_0$ is a partial collider and $\forall X_i \in \mathbf{De}(X_0)$, $X_i$ is not incident to a tail-tail edge. We first show the following properties \textbf{1.} and \textbf{2.}.

\textbf{1.} In $G$, for two vertex sets $\mathbf{X}^i, \mathbf{X}^j$, if $X_0 \in \mathbf{X}^i$ and $X_0 \not \in \mathbf{X}^j$, then $\mathbf{X}^i \not \sim_G \mathbf{X}^j$. This is because $Y$ is adjacent to $X_0$ in $G$.

Further, we show for two vertex sets $\mathbf{X}^i, \mathbf{X}^j$, if $\mathbf{X}^i \cap \mathbf{De}(X_0) \neq \emptyset$ and $\mathbf{X}^j \cap \mathbf{De}(X_0) = \emptyset$, then $\mathbf{X}^i \not \sim_G \mathbf{X}^j$. This is proved as follows. For $X_i \in \mathbf{X}^i \cap \mathbf{De}(X_0)$, there is a path $p:=<Y*\!\to X_0 \to \cdots \to X_i>$ from $Y$ to $X_i$. Since $\mathbf{X}^j \cap \mathbf{De}(X_0)=\emptyset$, $\forall\, \mathbf{X}^{ij} \subseteq \mathbf{X}^i \cap \mathbf{X}^j$, we have $\mathbf{X}^{ij} \cap \mathbf{De}(X_0) = \emptyset$ and $X_i \in \mathbf{X}^{({ij})^c}$. As a result, $\forall\, \mathbf{X}^{ij} \subseteq \mathbf{X}^i \cap \mathbf{X}^j$, there is a path $p$ between $Y$ and $\mathbf{X}^{({ij})^c}$ such that $p$ can not be blocked by $\mathbf{X}^{ij}$, \emph{i.e.}, $\mathbf{X}^i \not \sim_G \mathbf{X}^j$.

In $\tilde{G}$, similarly, for two vertex sets $\mathbf{X}^i, \mathbf{X}^j$, if $\mathbf{X}^i \cap \mathbf{De}(X_0) \neq \emptyset$ and $\mathbf{V}_j \cap \mathbf{De}(X_0) = \emptyset$, then $\mathrm{X}_i \not \sim_{\tilde{G}} \mathbf{X}^j$.

\textbf{2.} In $G$, by \textbf{1.}, divide subsets of $\mathbf{X}$ into those that contain $X_0$ and those that do not contain $X_0$. Denote the number of equivalent classes in them as $N_1, N_2$, respectively, we have $N=N_1+N_2$. Further, divide those subsets that do not contain $X_0$ into those that have an intersection with $\mathbf{De}(X_0)$ and those that have no intersection with $\mathbf{De}(X_0)$. Denote the number of equivalent classes in them as $N_3, N_4$, respectively. We have $N_2=N_3+N_4$ and thus $N=N_1+N_3+N_4$.

Similarly, in $\tilde{G}$, divide subsets of $\mathbf{X}\backslash \{X_0\}$ into those that have an intersection with $\mathbf{De}(X_0)$ and those that have no intersection with $\mathbf{De}(X_0)$. Denote the number of equivalent classes in them as $\tilde{N}_1, \tilde{N}_2$, respectively. We have $\tilde{N}=\tilde{N}_1+\tilde{N}_2$.

It is straightforward to have $N_4 \geq 1$. In the following, we will show $N_1 \geq \tilde{N}_2, N_3=\tilde{N}_1$ and thus $N\geq \tilde{N}+1$.

\textbf{Claim. 1.} For two subsets of vertex $\mathbf{X}^i, \mathbf{X}^j$ such that $\mathbf{X}^i \cap (X_0 \cup \mathbf{De}(X_0))=\emptyset$ and $\mathbf{X}^j \cap (X_0 \cup \mathbf{De}(X_0))=\emptyset$, then $\mathbf{X}^i \sim_{\tilde{G}} \mathbf{X}^j \Leftrightarrow  \mathbf{X}^i \cup X_0 \sim_G \mathbf{X}^j \cup X_0$, which indicates $N_1 \geq \tilde{N}_2$.

Proof of \textbf{Claim. 1}. $\Rightarrow$ can be proved similarly as the proof of the right side.

$\Leftarrow$ Suppose $\mathbf{X}^i \cup X_0 \sim_G \mathbf{X}^j \cup X_0$, we will show $\mathbf{X}^i \sim_{\tilde{G}} \mathbf{X}^j$, given the fact that $\mathbf{X}^i, \mathbf{X}^j$ do not contain $X_0$ nor its descendants, and any member of $\{X_0\}\cup \mathbf{De}(X_0)$ is not incident to a tail-tail edge.

\textbf{1.} For $X_i \in \mathbf{De}(X_0)$ and $X_j \not\in \mathbf{De}(X_0)$, since $X_i - X_j$ and $X_i \to X_j$ is not allowed, it must be $X_i \leftarrow\!* X_j$. Similarly, we have $X_0 \to X_i$ and $X_0 \leftarrow\!* X_j$.

\textbf{2.} By $\mathbf{X}^i \cup X_0 \sim_G \mathbf{X}^j \cup X_0$, we have $\exists\, \mathbf{X}^{ij} \subseteq (\mathbf{X}^i \cup X_0) \cap (\mathbf{X}^j \cup X_0)$ such that $Y \ind_G  (\mathbf{X}^i \cup \mathbf{X}^j \cup X_0) \backslash \mathbf{X}^{ij}  | \mathbf{X}^{ij}$. Since $Y$ is adjacent to $X_0$, $\mathbf{X}^{ij}$ must contain $X_0$. That is, $\exists\, \mathbf{X}^{ij} \subseteq \mathbf{X}^i \cap \mathbf{X}^j$ such that $ Y \ind_G (\mathbf{X}^i \cup \mathbf{X}^j \cup X_0) \backslash (\mathbf{X}^{ij} \cup X_0)| \mathbf{X}^{ij} \cup X_0$, which is equivalent to $Y \ind_G  \mathbf{X}^{({ij})^c}  | \mathbf{X}^{ij} \cup X_0$.

\textbf{3.} We first show $Y \ind_G  \mathbf{X}^{({ij})^c}  | \mathbf{X}^{ij}$, which is equivalent to showing any path between $Y$ and $\mathbf{X}^{({ij})^c}$ can be blocked by $\mathbf{X}^{ij}$. Prove by contradiction. Suppose there is a path $p_0:=<X_{k_1},X_{k_2},...,Y>$ between $Y$ and $\mathbf{X}^{({ij})^c}$ that can be blocked by $\mathbf{X}^{ij} \cup X_0$, and can not be blocked by $\mathbf{X}^{ij}$. By Lemma. \ref{lemma: union d/m-separation}, the set $\{X_0\}$ contains a non-collider on $p_0$, which means $X_0$ is a non-collider on $p_0$.

Then, $X_0$ is incident to at least one tail on $p_0$, since $X_0 \leftarrow\!* X_j$ for $X_j \not\in \mathbf{De}(X_0)$, $p_0$ must contain a member of $\mathbf{De}(X_0)$ and $p_0=<X_{k_1},X_{k_2},...,X_0 \to X_i, ..., Y>$ for $X_i \in \mathbf{De}(X_0)$. Since $Y$ is not a member of $\mathbf{De}(X_0)$, there is $X_0 \to \cdots \to X_i \leftarrow\!* X_j$ for $X_i \in \mathbf{De}(X_0)$ and $X_j \not\in \mathbf{De}(X_0)$ on $p_0$. As a result, $p_0$ contains a collider that itself nor its descendants are in $\mathbf{X}^{ij}$. This means $p_0$ can be blocked by $\mathbf{X}^{ij}$ and thus a contradiction.

\textbf{4.} We then show $X_0 \ind_G \mathbf{X}^{({ij})^c} | \mathbf{X}^{ij}$, which together with \textbf{3.} means $\{Y,X_0\} \ind_G  \mathbf{X}^{({ij})^c}| \mathbf{X}^{ij}$ and thus $\mathbf{X}^i \sim_{\tilde{G}} \mathbf{X}^j$. We show this by proving any path between $\mathbf{X}^{({ij})^c}$ and $X_0$ can be blocked by $\mathbf{X}^{ij}$. Prove by contradiction. Suppose there is a path $p_1:=<X_{k_1},X_{k_2},...,X_{k_l},X_0>$ that can not be blocked by $\mathbf{X}^{ij}$.

Then, if $X_{k_l} \not\in \mathbf{De}(X_0)$, we have a path $p_2:=<X_{k_1},...,X_{k_l} *\!\to X_0 \leftarrow * Y>$ such that $p_2$ can not be blocked by $\mathbf{X}^{ij} \cup X_0$, which contradicts with $Y \ind_G  \mathbf{X}^{({ij})^c}  | \mathbf{X}^{ij} \cup X_0$. Otherwise $X_{k_l} \in \mathbf{De}(X_0)$, then we have a path $p_2:=<X_{k_1},...,X_{k_l} \leftarrow X_0 \leftarrow\!* Y>$. Since $X_{k_1} \in \mathbf{X}^{({ij})^c}$ and thus $X_{k_1} \not\in \mathbf{De}(X_0)$, there is $X_j \to X_i \leftarrow \cdots \leftarrow X_0$, with $X_j \not\in \mathbf{De}(X_0)$, $X_i \in \mathbf{De}(X_0)$, between $X_{k_1}$ and $X_0$. Hence, we have $X_i$ is a collider on $p_1$, itself nor its descendants are in $\mathbf{X}^{ij}$, which means $p_1$ can be blocked by $\mathbf{X}^{ij}$ and thus a contradiction.

To conclude, \textbf{3.} and \textbf{4.} mean $\Leftarrow$ is true.

\textbf{Claim. 2.} Two subsets of vertices $\mathbf{X}^i, \mathbf{X}^j$ such that $\mathbf{X}^i, \mathbf{X}^j$ do not contain $X_0$, $\mathbf{X}^i \cap \mathbf{De}(X_0) \neq \emptyset$, and $\mathbf{X}^j \cap \mathbf{De}(X_0) \neq \emptyset$, then, $\mathbf{X}^i \sim_{\tilde{G}} \mathbf{X}^j \Leftrightarrow \mathbf{X}^i \sim_G \mathbf{X}^j$, which indicates $N_3=\tilde{N}_1$.

Proof of \textbf{Claim. 2}.  $\Rightarrow$ can be proved similarly as the proof of the right side.

$\Leftarrow$ \textbf{1.} By $\mathbf{X}^i \sim_G \mathbf{X}^j$, we have $\exists\, \mathbf{X}^{ij} \subseteq \mathbf{X}^i \cap \mathbf{X}^j$ such that $Y \ind_G  \mathbf{X}^{({ij})^c}  | \mathbf{X}^{ij}$.  

We show that $\mathbf{X}^{ij}$ must contain a member of $\mathbf{De}(X_0)$. Prove by contradiction. Suppose $\mathbf{X}^{ij} \cap \mathbf{De}(X_0) = \emptyset$, which means $\mathbf{X}^{ij}$ does not contain $X_0$ nor its descendants. Since $\mathbf{X}^i \cap \mathbf{De}(X_0) \neq \emptyset$ and $\mathbf{X}^j \cap \mathbf{De}(X_0) \neq \emptyset$, $(\mathbf{X}^i \cup \mathbf{X}^j) \backslash (\mathbf{X}^{ij} \cap \mathbf{De}(X_0)) \neq \emptyset$. As a result, there is a path $p_0:=<Y*\!\to X_0 \to \cdots \to X_i>$, for $X_i \in \mathbf{De}(X_0)$, between $Y$ and $\mathbf{X}^{({ij})^c}$ that can not be blocked by $\mathbf{X}^{ij}$, which contradicts with $Y \ind_G  \mathbf{X}^{({ij})^c}  | \mathbf{X}^{ij}$. As a result, $\mathbf{X}^{ij}$ must contain a member of $\mathbf{De}(X_0)$.

\textbf{2.} Next, we show $X_0 \ind_G \mathbf{X}^{({ij})^c} | \mathbf{X}^{ij}$, which together with \textbf{1.} indicates $\{Y,X_0\} \ind_G  \mathbf{X}^{({ij})^c} | \mathbf{X}^{ij}$ and thus $\mathbf{X}^i \sim_{\tilde{G}} \mathbf{X}^j$. We prove this by showing any path between $\mathbf{X}^{({ij})^c}$ and $X_0$ can be blocked by $\mathbf{X}^{ij}$.

Prove by contradiction. Suppose there is a path $p_1:=<X_{k_1},X_{k_2},...,X_{k_l},X_0>$ that can not be blocked $\mathbf{X}^{ij}$. Consider the path $p_2:=<X_{k_1},X_{k_2},...,X_{k_l},X_0\leftarrow\!* Y>$ constructed from $p_1$. If $X_0$ is a non-collider, since $\mathbf{X}^{ij}$ does not contain $X_0$, $p_2$ can not be blocked by $\mathbf{X}^{ij}$. Otherwise $X_0$ is a collider, since $\mathbf{X}^{ij}$ contains a member of $\mathbf{De}(X_0)$, $p_2$ can not be blocked by $\mathbf{X}^{ij}$ neither. These results contradict with $Y \ind_G  \mathbf{X}^{({ij})^c}  | \mathbf{X}^{ij}$. Hence, we have $X_0 \ind_G \mathbf{X}^{({ij})^c}| \mathbf{X}^{ij}$.

To conclude, \textbf{1.} and \textbf{2.} mean $\Leftarrow$ is true. \textbf{Claim.1} indicates $N_1 \geq \tilde{N}_2$, \textbf{Claim.2} indicates $N_3 = \tilde{N}_1$, and thus $N\geq \tilde{N}+1$.
\end{proof}

\begin{corollary}[Merging property for multiple vertices]
For any causal graph $G$ where $Y$ is adjacent to a connected vertex sets $\mathbf{X}_0$, and vertices in $\mathbf{V}\backslash (\mathbf{X}_0 \cup Y)$ is adjacent to at most one vertex in $\mathbf{X}_0 \cup Y$. Merge $Y, \mathbf{X}_0$ into a new vertex $\tilde{Y}$ and call the resulted graph $\tilde{G}$. Then, $\tilde{N}+|\mathbf{X}_0| \leq N \leq 2^{|\mathbf{X}_0|} \tilde{N}$.
\end{corollary}

\begin{proof}
Proof of the right side is the same as Lemma. \ref{lemma: merging property}. 

Proof of the left side. Since $\mathbf{X}_0$ is a connected set and $Y$ is adjacent to $\mathbf{X}_0$, we can delete edges among $\mathbf{X}_0 \cup Y$ until the subgraph over $\mathbf{X}_0 \cup Y$ becomes a spanning tree over $\mathbf{X}_0 \cup Y$ with $Y$ as the root vertex. Then, iteratively merge vertices and use Lemma. \ref{lemma: merging property}, Lemma~\ref{lemma: add delete edges}, we have $N \geq \tilde{N}+|\mathbf{X}_0|$.
\end{proof}

\subsection{Details of Prop.~\ref{prop:complexity(informal)}: Complexity}

In this section, we discuss the complexity of searching $N_G$ equivalence classes. We show that compared to the exponential cost $O(2^{d_S})$ of exhaustive search, our search strategy enjoys a polynomial cost $\mathrm{P}(d_S)$ when $G_S$ is mainly composited of chain vertices. Our analysis mainly uses the results in Lemma~\ref{lamma: large n >= 3} and Lemma~\ref{lemma: merging property}. The idea is briefed as follows:

Lemma~\ref{lamma: large n >= 3} shows that any $G_S$ with $d_{>2}=\omega(\log(d))$ has $N_G=\mathrm{NP}(d_S)$. Hence, we need to look at cases when $d_{>2}=O(\log(d))$. For these cases, Lemma~\ref{lemma: merging property} shows that the non-chain vertices in $G_S$ do not influence the rank of $N_G$. This is because we can iteratively merge the non-chain vertices into $Y$ and have $N_G$ being squeezed within $N_{\tilde{G}}\sim 2^{d_{>2}}N_{\tilde{G}}$. Since $2^{d_{>2}}=\mathrm{P}(d_S)$, we have $N_G=\mathrm{P}(d_S)$ if and only if $N_{\tilde{G}}=\mathrm{P}(d_S)$. Therefore, the rank of $N_G$ when $d_{>2}=O(\log(d_S))$ is decided by $N_{\tilde{G}}$, in other words, by the chain vertices with deg$\leq 2$ in $G_S$.

Intuitively, when the chain vertices compose different chains that do not intersect each other, by Claim~\ref{example: chain}, the $N_{\tilde{G}}$ is the product of the chains' lengths. Hence, the more ``intensive'' the chain vertices distribute, the smaller $N_{\tilde{G}}$ and thus $N_G$ will be. In the following, we will provide a formal metric $F_G$ to measure the intensity of chain vertices.

We first define the following structures on $G_S$. For brevity, we omit the subscript and denote $G_S$ as $G$, $d_S$ as $d$, respectively.
\begin{enumerate}
    \item A path is a sequence of distinct vertices $<V_1,V_2,...,V_l>$ where $V_i, V_{i+1} (i=1,2,...,l-1)$ are adjacent in $G$. The length of the path is $l$. Define the distance between a vertex and $Y$ and the length of the shortest path between them.
    \item A chain is a path where every vertex on it has $\mathrm{deg} \leq 2$ in $G$. The head of a chain is the vertex in it that is closest to $Y$. A maximal chain is a chain that can not be made longer by adding new vertices. 
    \item For a maximal chain $c$, let $c \in \mathbf{Ch}^\prime(Y)$ if there is no other maximal chain in the shortest path between the head of $c$ and $Y$. For two maximal chains $c_1, c_2$, let $c_2 \in \mathbf{De}^\prime(c_1)$ if there is a path between a vertex in $c_2$ and $Y$ that contains $c_1$.
    \item For a maximal chain $c \in \mathbf{Ch}^\prime(Y)$, define a set of operations $\mathrm{opt}_c(G)$ on $G$. Specifically, if $\mathbf{De}^\prime(c) = \emptyset$, define $\mathrm{opt}_{c}(G):=\{\text{remove} \, c\}$; Otherwise, define $\mathrm{opt}_{c}(G):=\{\text{remove} \, \mathbf{X}_{c}^{1:i} | i=1,2,...,l\} \cup \{ \text{replace} \, c \, \text{with an edge}\}$, with $l$ the number of vertices on $c$, $X^i_{c}$ the $i$-th one (in the order of the distance to $Y$), and $\mathbf{X}^{1:i}_c:=\{X_c^1,X_c^2,...,X_c^i\}$.
    \item For a maximal chain with $l$ vertices, define $\mathrm{cost}(c):=l+1$ if $c$ has one head vertex, $(l^2+l+2)/2$ if $c$ has two head vertices (that is both sides of $c$ have equal distance to $Y$).
\end{enumerate}

\noindent\textbf{Proposition~\ref{prop:complexity(informal)}} (Complexity)\textbf{.} Let $F_G$ be an recursive metric defined over maximal chains in $G$:
\begin{equation*}
    F_G:=\prod_{\substack{c\in \mathbf{Ch}^\prime_G(Y) \\ \mathbf{De}^\prime(c) = \emptyset}} \mathrm{cost}(c) \sum_{\mathrm{opt}\in \prod_{c\in \mathbf{Ch}^\prime_G(Y)}\mathrm{opt}_{c}} F_{\mathrm{opt}(G)}.
\end{equation*}
Then, $N_G = \mathrm{P}(d)$ if and only if $d_{>2} = O(\mathrm{log}(d))$ and $F_G = \mathrm{P}(d)$.

\begin{proof}
To prove the proposition, we will show \textbf{i)} if $d_{>2}=\omega(\mathrm{log}(d))$, then $N_G=\mathrm{NP}(d)$; \textbf{ii)} if $d_{>2}=O(\mathrm{log}(d))$, then $N_G=\mathrm{P}(d) \Leftrightarrow F_G = \mathrm{P}(d)$. 

Specifically, \textbf{ii)} means if $d_{>2} = O(\mathrm{log}(d))$ and $F_G = \mathrm{P}(d)$, then $N_G=\mathrm{P}(d)$, which shows $\Leftarrow$ of the proposition. \textbf{i)} means if $N_G = \mathrm{P}(d)$, then $d_{>2} = O(\mathrm{log}(d))$, together with \textbf{ii)}, it means if $N_G=\mathrm{P}(d)$, then $d_{>2}=O(\mathrm{log}(d))$ and $F_G=\mathrm{P}(d)$, which shows $\Rightarrow$ of the proposition.

The proof of \textbf{i)} is at Lemma. \ref{lamma: large n >= 3}. The proof of \textbf{ii)} is as follows:

\textbf{Claim. 1.} $F_G \leq N_G \leq 2^{d_{>2}} F_G$.

Proof of \textbf{Claim. 1}. Modify Alg.~\ref{alg: recover g-equivalence} in the following way:

\textbf{i)} Add before line-2: if $\mathbf{Neig}(Y)$ contains non-chain vertices, then merge them into $Y$ until $\mathbf{Neig}(Y)$ only contains non-chain vertices. Call the resulting graph as $\tilde{G}$. \textbf{ii)} Replace the $G$ with $\tilde{G}$, $G^\prime$ with $\tilde{G}^\prime$, and $N_G$ with $N_{\tilde{G}}$, in lines 4-12. Call the modified algorithm as $N_{\tilde{G}}=\mathrm{count}^\prime(G)$.

Next, we first show $F_G=N_{\tilde{G}}$, then prove \textbf{Claim. 1} via the $\mathrm{count}^\prime(G)$ algorithm.

After merging non-chain vertices around $Y$, in $\tilde{G}$, $Y$'s neighbors are the head vertices of the maximal chains in $\mathbf{Ch}^\prime(Y)$. Recursively conduct the $\mathrm{count}^\prime$ algorithm on $\tilde{G}$ and its induced MAGs $\tilde{M}_G$, until all vertices in all maximal chains in $\mathbf{Ch}^\prime(Y)$ have been traversed.

For maximal chains without descendants, since they are disjoint with the other maximal chains, we have $N_{\tilde{G}}=(\prod_{c\in \mathbf{Ch}^\prime_{\tilde{G}}(Y), \mathbf{De}^\prime_{\tilde{G}}(c)=\emptyset} \mathrm{cost}(c))N_{\mathrm{opt}_1(\tilde{G})}$, where $\mathrm{cost}(c)=l+1$ if $c$ has one head vertex (see Claim~\ref{example: chain}) and $(l^2+l+2)/2$ if $c$ has two head vertices (see Claim~\ref{example: circle}), and $\mathrm{opt}_1:=\prod_{c\in \mathbf{Ch}^\prime_{\tilde{G}}(Y), \mathbf{De}^\prime_{\tilde{G}}(c)=\emptyset} \mathrm{cost}(c) \mathrm{opt}_{c}$, and $\mathrm{opt}_1(\tilde{G})$ is the causal graph after removing all maximal chains without descendants.

In $\mathrm{opt}_1(\tilde{G})$, denote the remained maximal chains in $\mathbf{Ch}^\prime(Y)$ as $\{c_i\}_{i=1:r}$ and the vertices on them as $\{X^i_1, ..., X^i_{l_i}\}_{i=1:r}$. To obtain $N_{\mathrm{opt}_1(\tilde{G})}$, for each $c_i$, similarly as the analysis of the Claim \ref{example: chain}, we need to consider the following $l_i+1$ situations: $X^i_1$ is blocked\footnote{A vertex of deg=2 is blocked if it is a non-collider and is put in the selection set, or it is a collider and is put in the latent set. A vertex is open if it is not blocked}; $X^i_1$ is open, $X^i_2$ is blocked; ...; $X^i_1, ...,X^i_{l_i-1}$ are blocked, $X^i_{l_i}$ is open; and $X^i_1,...,X^i_{t_{i}}$ are open. Because vertices in the $r$ maximal chains are disjoint, we in total need to consider $\prod_{i=1}^r l_i+1$ situations, and $N_{\mathrm{opt}_1(\tilde{G})}=\sum_{j=1}^{\prod_{i=1}^r l_i+1} N_{\mathrm{opt}_1 (\tilde{G})^\prime_j}$, where $\mathrm{opt}_1 (\tilde{G})^\prime_j$ denotes the induced subgraph from $\mathrm{opt}_1(\tilde{G})$ in the $j$-th situation.

Note that each subgraph $\mathrm{opt}_1 (\tilde{G})^\prime_j$ corresponds to an operation in $\prod_{c\in \mathbf{Ch}^\prime_{\tilde{G}}(Y)} \mathrm{opt}_{c}$ on $\tilde{G}$, we have $N_{\mathrm{opt}_1(\tilde{G})}=\sum_{\mathrm{opt}\in \prod_{c\in \mathbf{Ch}^\prime_{\tilde{G}}(Y)}\mathrm{opt}_{c}}N_{\mathrm{opt}(\tilde{G})}$. Hence, $N_{\tilde{G}}=\prod_{c\in \mathbf{Ch}^\prime_{\tilde{G}}(Y),\mathbf{De}^\prime_{\tilde{G}}(c)=\emptyset} \mathrm{cost}(c) \sum_{\mathrm{opt} \in \prod_{c\in \mathbf{Ch}^\prime_{\tilde{G}}(Y)}\mathrm{opt}_{c}} N_{\mathrm{opt}(\tilde{G})}$ and $N_{\tilde{G}}=F_G$.

During the recursive execution of the $\mathrm{count}^\prime(G)$ to obtain $N_{\tilde{G}}$, there are at most $d_{>2}$ non-chain vertices merged into $Y$. As a result, by Lemma. \ref{lemma: merging property}, we have $N_G \leq 2^{d_{>2}} N_{\tilde{G}}$. The number of non-chain vertices merged into $Y$ is at least $0$, so we also have $N_{\tilde{G}} \leq N_G$.

To conclude, we have $F_G \leq N_G \leq 2^{d_{>2}}F_G$, which means \textbf{Claim. 1.} and hence the proposition is true.
\end{proof}

\begin{remark}
If the skeleton of $G$ is a tree, for two maximal chains $c_1, c_2$, define $c_2 \in \mathbf{Ch}^\prime(c_1)$ if $c_1$ contains the first non-chain vertex in the path from the head of $c_2$ to $Y$. The $F_G$ degenerates to:

\begin{equation*}
    F_G = \prod_{\tilde{c}\in \mathbf{Ch}^\prime(Y)} f(\tilde{c}),
\end{equation*}
with $f(c):=\mathrm{cost}(c)+\prod_{\tilde{c}\in \mathbf{Ch}^\prime(c)}f(\tilde{c})$, $\mathrm{cost}(c)=\mathrm{len}(c)+ \mathbbm{1}(\mathbf{Ch}^\prime(c) = \emptyset)$.
\end{remark}

\newpage
\section{Experiment}
\label{sec. appendix-exp}
\subsection{Implementation details}
\label{sec. appendix implementation}
All codes are implemented with PyTorch 1.10 and run on an Intel Xeon E5-2699A v4@2.40GHz CPU.

\noindent\textbf{Baselines.}
    \begin{enumerate}
    \item Vanilla. $\mathbb{E}[Y|\boldsymbol{x}]$ is implemented by the same neural network as $f_{S^\prime}$.
    \item ICP (\url{https://github.com/juangamella/icp}). The level of the test procedure is set to $0.05$. The estimator is implemented by the same neural network as $f_{S^\prime}$.
    \item IC (\url{https://github.com/mrojascarulla/causal_transfer_learning}). The level of the test procedure is set to $0.05$. Levene test is used. The estimator is implemented by the same neural network as $f_{S^\prime}$.
    \item DRO (\url{https://github.com/duchi-lab/certifiable-distributional-robustness}). The $\gamma$ is set to $2$. The estimator is implemented by the same neural network as $f_{S^\prime}$.
    \item Surgery estimator. Since there is no official implementation available, we implement it based on our method. Specifically, we pick $2\sim 3$ validation environments from $\mathcal{E}_{\text{tr}}$ and use the validation loss to select $S^*$.
    \item IRM (\url{https://github.com/facebookresearch/InvariantRiskMinimization}). The best $\phi$ is chosen by comparing the validation loss of $\mathrm{reg}=0,10^{-5},10^{-4},10^{-3},10^{-2},10^{-1}$.
    \item HRM (\url{https://github.com/LhSthu/HRM}). The cluster number is set to the number of deployment environments. $\sigma$ and $\lambda$ are both set to $0.1$. The overall threshold for subset selection is set to $0.25$.
    \item IB-IRM (\url{https://github.com/ahujak/IB-IRM}). The $\lambda_{\text{ib}}$ is set to $0.1$, the $\lambda_{\text{irm}}$ is set to $0.75$. The estimator is implemented by the same neural network as $f_{S^\prime}$.
    \item Anchor regression (\url{https://github.com/rothenhaeusler/anchor-regression}). The $\gamma$ is set to $1.5$.
\end{enumerate}

\noindent\textbf{Synthetic study.} The neural networks to implement $f_{S^\prime}$ and $h_\theta$ are two-layers MLPs. We use a sigmoid activation function in the hidden layer to add non-linearity. We use the Adam optimizer. The learning rate is set to $0.02$, and epochs are set to $10000$ with an early stop.

\noindent\textbf{Alzheimer's disease diagnosis.} The neural networks to implement $f_{S^\prime}$ and $h_\theta$ are the same as the synthetic study. We use the SGD optimizer. For the estimation of $f_{S^\prime}$, the epochs are set to $5000$, the learning rate is set to $0.25$ in the first $4000$ epochs, and decreased to $0.1$ in the last $1000$ epochs. For the estimation of $\mathcal{L}$, the epochs are set to $12000$ with an early stop, the learning rate is set to $0.4$.

\noindent\textbf{Gene function prediction.} The neural networks to implement $f_{S^\prime}$ and $h_\theta$ are the same as the synthetic study. We use the SGD optimizer. The epochs are set to $10000$. For the estimation of $f_{S^\prime}$, the learning rate is set to $0.01$. For the estimation of $\mathcal{L}$, the learning rate is set to $0.05$.

\begin{figure}[ht!]
    \centering
    \includegraphics[width=\textwidth]{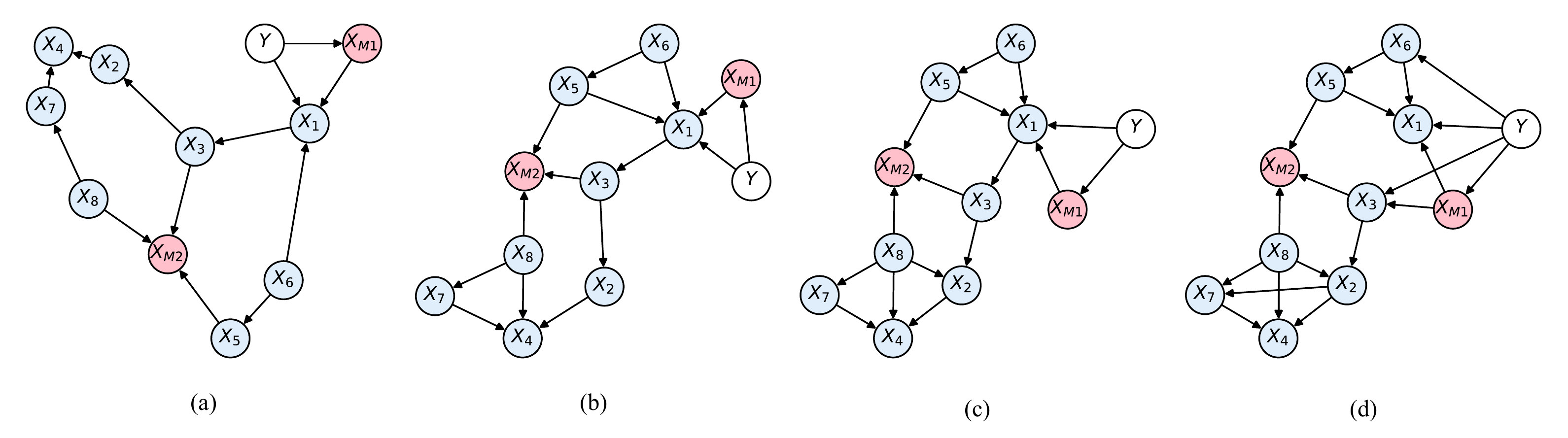}
    \caption{The synthetic causal graphs for complexity analysis. Stable and mutable variables are respectively marked blue and red. We have $d_{>2}=1,2,5,6$ in (a), (b), (c), (d), respectively. The sparse graphs (a), (b) are generated by deleting edges from Fig.~\ref{fig:simg3_ad_causal_graphs} (a). The dense graphs (c), (d) are generated by adding edges to Fig.~\ref{fig:simg3_ad_causal_graphs} (a).}
    \label{fig:simu1245_causal_graph}
\end{figure}

\begin{table}[h!]
  \caption{Indices for brain region partition.}
  \label{tab: aal indices}
  \centering
  \begin{tabular}{llc}
     \toprule
     Abbreviation & Brain region & AAL index \cite{tzourio2002automated} \\
     \midrule
     FSL & Frontal superior lobe & 2101,2102,2111,2112,2601,2602 \\
     FML & Frontal middle lobe & 2201,2202,2211,2212,2611,2612 \\
     FIL & Frontal inferior lobe & 2301,2302,2311,2312,2321,2322 \\
     \hline
     TSL & Temporal superior lobe & 8111,8112 \\
     TML & Temporal medial lobe & 8201,8202 \\
     TIL & Temporal inferior lobe & 8301,8302 \\
     TP & Temporal pole & 8121,8122,8211,8212 \\
     \hline 
     PSL & Parietal superior lobe & 6101,6102 \\
     PIL & Parietal inferior lobe & 6201,6202 \\
     \hline
     OSL & Occipital superior lobe & 5101,5102 \\
     OML & Occipital middle lobe & 5201,5202 \\
     OIL & Occipital inferior lobe & 5301,5302 \\
     \hline
     CA & Cingulum anterior & 4001,4002 \\
     CM & Cingulum middle & 4011,4012 \\
     CP & Cingulum posterior & 4021,4022 \\
     \hline
     INS & Insula & 3001,3002 \\
     AMY & Amygdala & 4201,4202 \\
     CAU & Caudate & 7001,7002 \\
     HP & Hippocampus & 4101,4102 \\
     PAL & Pallidum & 7021,7022 \\
     PUT & Putamen & 7011,7012 \\
     THA & Thalamus & 7101,7102 \\
     \bottomrule
  \end{tabular}
\end{table}

\clearpage
\subsection{Extra results}
\label{sec. appendix extra results}
\begin{table}[h]
\caption{Max. MSE evaluation on synthetic and IMPC datasets. The first column notes the methods we compare. The second column represents the max. MSE over deployment environments. Data for Syn-a,b,c,d are respectively generated by the causal graphs (a) ,(b), (c), and (d) shown in Fig.~\ref{fig:simu1245_causal_graph}. The best results are \textbf{boldfaced}.}
    \label{tab: maxmse, numerical value, synthetic}
    \centering
    \scalebox{1}{
    \begin{tabular}{c|c|c|c|c|c}
    \hline
    \multirow{2}{*}{Method} &  \multicolumn{5}{c}{max. MSE ($\downarrow$)}\\
    \cline{2-6} 
     &  Syn-a & Syn-b & Syn-c & Syn-d & IMPC\\
    \hline
    Vanilla & $15.946_{\pm 2.7}$ & $3.033_{\pm 2.7}$ & $5.613_{\pm 3.5}$ & $1.814_{\pm 0.4}$ & $1.227_{\pm 0.1}$ \\ 
    ICP \cite{peters2016causal} & $1.777_{\pm 0.6}$ & $1.629_{\pm 0.6}$ & $1.631_{\pm 0.6}$ & $1.097_{\pm 0.1}$ & $1.291_{\pm 0.3}$\\
    IC \cite{rojas2018invariant} & $5.580_{\pm 0.3}$ & $1.631_{\pm 0.4}$ & $2.322_{\pm 0.7}$ & $1.665_{\pm 0.3}$ & $1.253_{\pm 0.2}$\\
    DRO \cite{sinha2018certifiable}& $4.511_{\pm 1.8}$ & $1.628_{\pm 0.4}$ & $2.311_{\pm 0.7}$ & $1.827_{\pm 0.4}$ & $1.196_{\pm 0.1}$\\
    Surgery \cite{subbaswamy2019preventing} & $1.325_{\pm 0.0}$ & $1.086_{\pm 0.0}$ & $1.005_{\pm 0.1}$ & $1.190_{\pm 0.2}$ & $1.071_{\pm 0.1}$\\
    \hline
    IRM \cite{arjovsky2019invariant}& $6.328_{\pm 2.3}$ & $1.439_{\pm 0.2}$ & $3.067_{\pm 0.9}$ & $1.601_{\pm 0.5}$ & $1.296_{\pm 0.1}$\\
    HRM \cite{liu2021heterogeneous}& $4.511_{\pm 1.8}$ & $1.537_{\pm 0.4}$ & $1.019_{\pm 0.0}$ & $1.427_{\pm 0.3}$ & $1.205_{\pm 0.1}$\\
    IB-IRM \cite{ahuja2021invariance}& $1.194_{\pm 0.0}$ & $1.177_{\pm 0.0}$ & $1.111_{\pm 0.1}$ & $1.108_{\pm 0.0}$ & $1.288_{\pm 0.1}$\\
    AncReg \cite{rothenhausler2021anchor}& $1.482_{\pm 0.4}$ & $\mathbf{1.032_{\pm 0.0}}$ & $1.117_{\pm 0.7}$ & $1.631_{\pm 0.5}$ & $1.127_{\pm 0.1}$\\
    \hline
    Ours (Alg.~\ref{alg:identify-f-star}) & $\mathbf{1.037_{\pm 0.0}}$ & $1.046_{\pm 0.0}$ & $\mathbf{0.689_{\pm 0.0}}$ & $\mathbf{1.067_{\pm 0.0}}$ & $\mathbf{0.952_{\pm 0.0}}$\\
    \hline
    \end{tabular}}
\end{table}

\begin{table}[h]
\caption{Std. MSE evaluation on synthetic and IMPC datasets. The first column notes the methods we compare. The second column represents the max. MSE over deployment environments. Data for Syn-a,b,c,d are respectively generated by the causal graphs (a) ,(b), (c), and (d)  shown in Fig.~\ref{fig:simu1245_causal_graph}. The best results are \textbf{boldfaced}.}
    \label{tab: stdmse, numerical value, synthetic}
    \centering
    \scalebox{1}{
    \begin{tabular}{c|c|c|c|c|c}
    \hline
    \multirow{2}{*}{Method} &  \multicolumn{5}{c}{std. MSE ($\downarrow$)}\\
    \cline{2-6} 
     &  Syn-a & Syn-b & Syn-c & Syn-d & IMPC\\
    \hline
    Vanilla & $3.184_{\pm 1.2}$ & $0.552_{\pm 0.5}$ & $1.543_{\pm 1.0}$ & $0.463_{\pm 0.3}$ & $0.257_{\pm 0.0}$\\ 
    ICP \cite{peters2016causal} & $0.132_{\pm 0.1}$ & $0.145_{\pm 0.0}$ & $0.219_{\pm 0.1}$ & $0.055_{\pm 0.0}$& $0.289_{\pm 0.1}$ \\ 
    IC \cite{rojas2018invariant} & $0.421_{\pm 0.0}$ & $0.324_{\pm 0.2}$ & $0.635_{\pm 0.4}$ & $0.376_{\pm 0.1}$ & $0.302_{\pm 0.1}$\\  
    DRO \cite{sinha2018certifiable}& $0.329_{\pm 0.1}$ & $0.321_{\pm 0.2}$ & $0.633_{\pm 0.4}$ & $0.474_{\pm 0.3}$ & $0.266_{\pm 0.0}$\\ 
    Surgery \cite{subbaswamy2019preventing} & $0.367_{\pm 0.0}$ & $0.071_{\pm 0.0}$ & $0.147_{\pm 0.0}$ & $0.101_{\pm 0.1}$ & $0.237_{\pm 0.0}$\\  
    \hline
    IRM \cite{arjovsky2019invariant}& $1.560_{\pm 0.7}$ & $0.205_{\pm 0.1}$ & $0.803_{\pm 0.4}$ & $0.369_{\pm 0.3}$ & $0.319_{\pm 0.1}$\\ 
    HRM \cite{liu2021heterogeneous}& $0.328_{\pm 0.1}$ & $0.280_{\pm 0.2}$ & $\mathbf{0.011_{\pm 0.0}}$ & $0.195_{\pm 0.1}$& $0.278_{\pm 0.0}$ \\ 
    IB-IRM \cite{ahuja2021invariance}& $0.104_{\pm 0.0}$ & $0.063_{\pm 0.0}$ & $0.082_{\pm 0.0}$ & $\mathbf{0.055_{\pm 0.0}}$ & $0.259_{\pm 0.0}$\\ 
    AncReg \cite{rothenhausler2021anchor}& $0.210_{\pm 0.1}$ & $\mathbf{0.030_{\pm 0.0}}$ & $0.285_{\pm 0.3}$ & $0.410_{\pm 0.3}$ & $0.275_{\pm 0.0}$\\ 
    \hline
    Ours (Alg.~\ref{alg:identify-f-star}) & $\mathbf{0.034_{\pm 0.0}}$ & $0.043_{\pm 0.0}$ &  $0.073_{\pm 0.0}$ & $0.062_{\pm 0.0}$ & $\mathbf{0.017_{\pm 0.0}}$\\ 
    \hline
    \end{tabular}}
\end{table}

\begin{table}[ht!]
\caption{Std. over equivalent subsets on synthetic and IMPC datasets. Data for Syn-a,b,c,d are respectively generated by the causal graphs (a) ,(b), (c), and (d)  shown in Fig.~\ref{fig:simu1245_causal_graph}.}
    \label{tab: appendix verify equivalence}
    \centering
    \scalebox{1}{
    \begin{tabular}{c c c c c c}
    \hline
    \multirow{2}{*}{Metric} &  \multicolumn{5}{c}{Dataset}\\
    \cline{2-6} 
     &  Syn-a & Syn-b & Syn-c & Syn-d & IMPC\\
    \hline
     Inter-class std. & 0.118 & 1.178 & 0.697 & 1.668 & 0.288\\
     Intra-class std. & 0.005 & 0.017 & 0.001 & 0.015 & 0.023 \\
    \hline
    \end{tabular}}
\end{table}

\clearpage
\begin{figure}[h]
    \centering
    \includegraphics[width=0.6\textwidth]{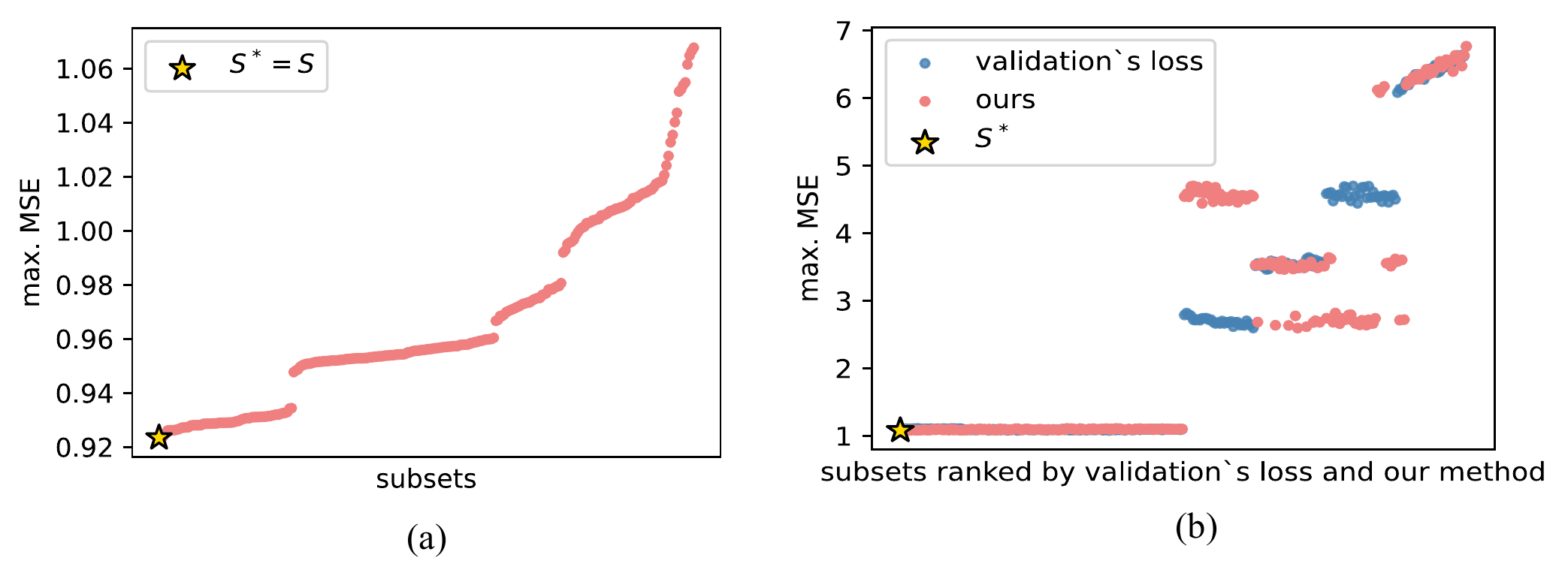}
    \caption{Detailed results for Fig.~\ref{fig:simg3_scatters_search_cost}. (a) and (b) respectively show the max. MSE of all the 256 subsets for Fig.~\ref{fig:simg3_scatters_search_cost} (a) and (b).}
    \label{fig:simg1245_minimax_validation}
\end{figure}

\begin{figure}[h]
    \centering
    \includegraphics[width=\textwidth]{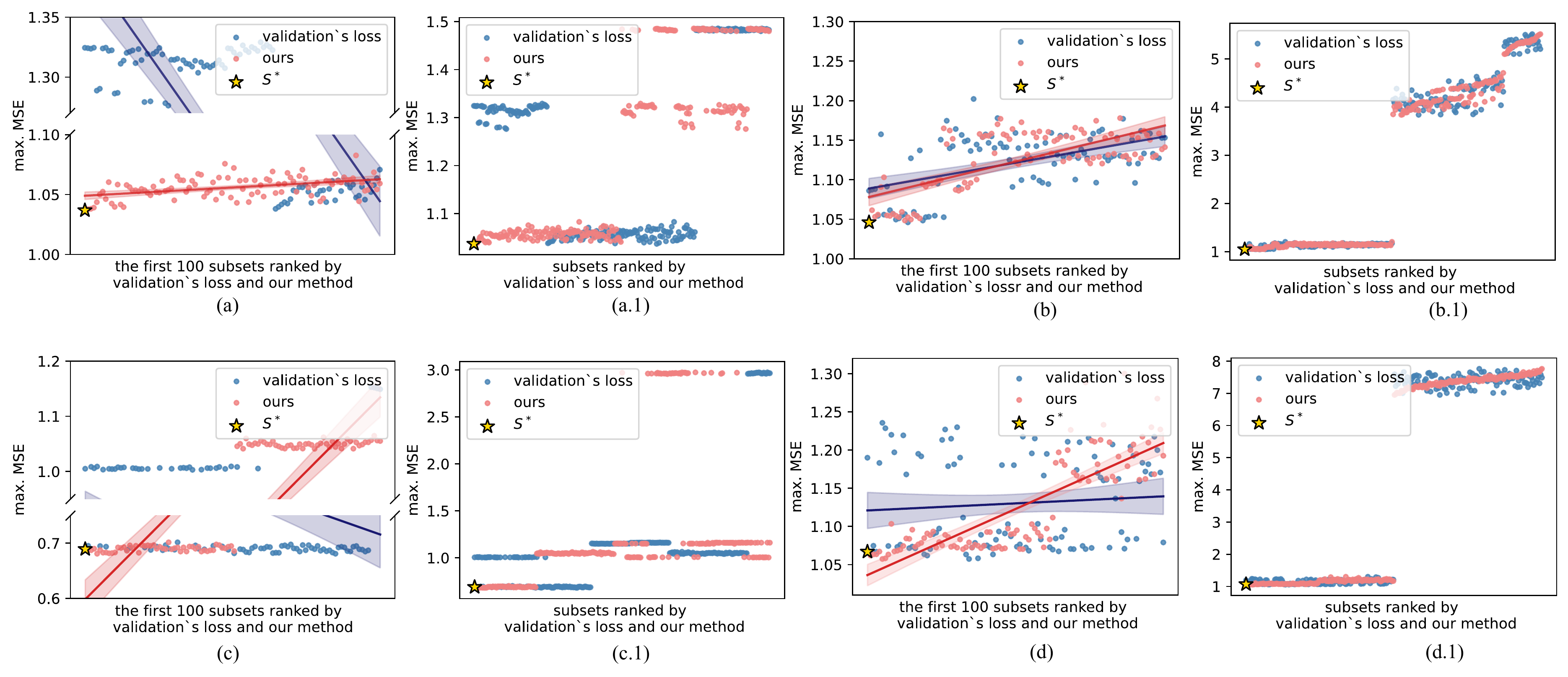}
    \caption{Results on synthetic data. (a), (b), (c), and (d) respectively show the max. MSE of the first 100 subsets ranked respectively according to our method and the validation`s loss, with data generated by the causal graphs in Fig.~\ref{fig:simu1245_causal_graph} (a), (b), (c), and (d). Detailed results of all the 256 subsets are shown in (a.1), (b.1), (c.1), and (d.1), respectively.}
    \label{fig:appdx_synthetic_reults}
\end{figure}

\clearpage
\subsubsection{Gene function prediction}
\label{sec.gene}
In this section, we evaluate our method on gene function prediction, which can potentially help better understand the human-disease progress \citep{munoz2018international}.

\begin{figure}[ht!]
    \centering
    \includegraphics[width=\textwidth]{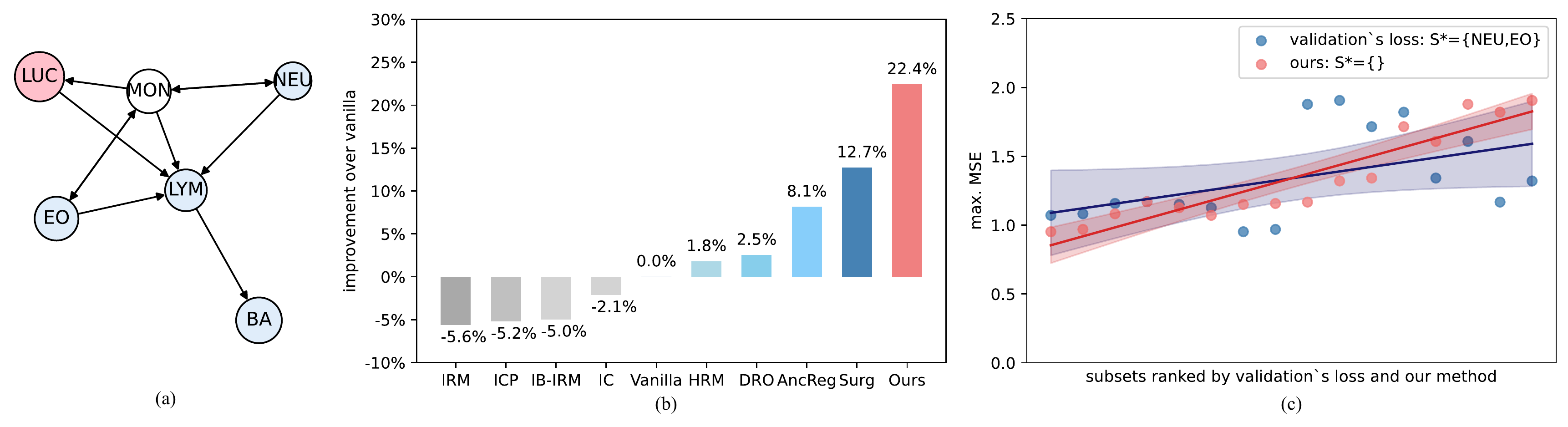}
    \caption{Gene function prediction. (a) The learned causal graph, where $\leftrightarrow$ denotes undirected edges. (b) Comparison of max. MSE over deployed environments. (c) Max. MSE of subsets that are ranked respectively according to the estimated worst-case risk of our method and the validation's loss.}
    \label{fig:gene_results}
\end{figure}

\noindent\textbf{Dataset.} We consider the International Mouse Phenotyping Consortium (IMPC) dataset\footnote{\url{http://www.crm.umontreal.ca/2016/Genetics16/competition_e.php}} that was originally published in a causal inference challenge and later used as a benchmark for domain generalization \cite{magliacane2018domain}. The IMPC contains the hematology phenotypes of both wild-type mice and mutant mice with $13$ kinds of single-gene knockout. To predict the gene function, we knock out this gene and assess the cell counts of monocyte (MON), with cell counts of neutrophil (NEU), lymphocyte (LYM), eosinophil (EO), basophil (BA), and large unstained cell (LUC) as covariates. We use the kind of knocked-out gene to divide environments. The training environments contain wild-type mice and five randomly picked gene knockouts. The deployed environments contain the rest nine kinds of gene knockouts. This random train-test split is repeated $45$ times.

\noindent\textbf{Causal discovery and $\mathrm{Pow}(S)/\!\sim_G$ recovery.} The learned causal graph is shown in Fig.~\ref{fig:gene_results} (a). As we can see,  we have $\text{LUC} \to \text{LYM}$ and $\text{LYM} \to \text{BA}$, which respectively echo the existing studies that monocyte can activate the lymphocyte \citep{carr1994monocyte} and increase the number of LUC \citep{lee2021clinical}, and that the lymphocyte participates in activating basophil \citep{goetzl1984basophil}. Since MON is mutable and $\text{LYM} \in \mathbf{De}(\text{MON})$ is pointed by LUC, the condition in Thm.~\ref{thm:graph degenerate} is violated. We need to compare the equivalence classes to find the optimal subset. Applying Alg.~\ref{alg: recover g-equivalence}, we find that there are $12$ equivalence classes out the $2^4$ subsets ($d_S|=4$ as shown in Fig.~\ref{fig:gene_results} (a)).

\noindent\textbf{Results.} Fig.~\ref{fig:gene_results} (b) and Tab.~\ref{tab: maxmse, numerical value, synthetic}, \ref{tab: stdmse, numerical value, synthetic} respectively report the max. MSE and std. MSE of our method and baselines. As we can see, our method can outperform the others by a significant margin. Besides, Fig.~\ref{fig:gene_results} (c) shows that our $\mathcal{L}$ can well reflect the worst-case risk.

\noindent\textbf{Analysis of $\sim_G$ equivalence.} We 
compute the intra-class std. versus the inter-class std. The results are shown in Tab.~\ref{tab: appendix verify equivalence} (IMPC), which show that equivalence subsets have comparable max. MSE.

%%%%%%%%%%%%%%%%%%%%%%%%%%%%%%%%%%%%%%%%%%%%%%%%%%%%%%%%%%%%%%%%%%%%%%%%%%%%%%%
%%%%%%%%%%%%%%%%%%%%%%%%%%%%%%%%%%%%%%%%%%%%%%%%%%%%%%%%%%%%%%%%%%%%%%%%%%%%%%%

\end{document}